\documentclass[manuscript, screen]{jair}

\setcopyright{cc}
\copyrightyear{2025}
\acmDOI{10.1613/jair.1.xxxxx}

\JAIRAE{Insert JAIR AE Name}
\JAIRTrack{} 
\acmVolume{4}
\acmArticle{6}
\acmMonth{8}
\acmYear{2025}

\usepackage{amsmath}
\usepackage{comment}
\usepackage{enumitem}

\usepackage[noend,ruled,noline]{algorithm2e}
\usepackage{latexsym,nicefrac,eufrak,xspace}
\usepackage{subcaption}

\usepackage{adjustbox,soul}
\usepackage{multirow,siunitx}
\usepackage[graphicx]{realboxes} 
\usepackage{booktabs,tabularx}

\usepackage{pifont}

\usepackage[capitalise]{cleveref}

\newcommand{\pjs}[1]{\textcolor{blue}{\textsc{Pjs:} #1}}

\usepackage{tikz}
\usetikzlibrary{arrows,calc,decorations,fit,positioning,shapes}

\usepackage{stmaryrd}
\definecolor{tyellow1}{HTML}{FCE94F}
\definecolor{tyellow2}{HTML}{EDD400}
\definecolor{tyellow3}{HTML}{C4A000}
\definecolor{torange1}{HTML}{FCAF3E}
\definecolor{torange2}{HTML}{F57900}
\definecolor{torange3}{HTML}{C35C00}
\definecolor{tbrown1}{HTML}{E9B96E}
\definecolor{tbrown2}{HTML}{C17D11}
\definecolor{tbrown3}{HTML}{8F5902}
\definecolor{tgreen1}{HTML}{8AE234}
\definecolor{tgreen2}{HTML}{73D216}
\definecolor{tgreen3}{HTML}{4E9A06}
\definecolor{tblue1}{HTML}{729FCF}
\definecolor{tblue2}{HTML}{3465A4}
\definecolor{tblue3}{HTML}{204A87}
\definecolor{tpurple1}{HTML}{AD7FA8}
\definecolor{tpurple2}{HTML}{75507B}
\definecolor{tpurple3}{HTML}{5C3566}
\definecolor{tred1}{HTML}{EF2929}
\definecolor{tred2}{HTML}{CC0000}
\definecolor{tred3}{HTML}{A40000}
\definecolor{tlgray1}{HTML}{EEEEEC}
\definecolor{tlgray2}{HTML}{D3D7CF}
\definecolor{tlgray3}{HTML}{BABDB6}
\definecolor{tdgray1}{HTML}{888A85}
\definecolor{tdgray2}{HTML}{555753}
\definecolor{tdgray3}{HTML}{2E3436}

\definecolor{gray}{rgb}{.4,.4,.4}
\definecolor{midgrey}{rgb}{0.5,0.5,0.5}
\definecolor{middarkgrey}{rgb}{0.35,0.35,0.35}
\definecolor{darkgrey}{rgb}{0.3,0.3,0.3}
\definecolor{darkred}{rgb}{0.7,0.1,0.1}
\definecolor{midblue}{rgb}{0.2,0.2,0.7}
\definecolor{darkblue}{rgb}{0.1,0.1,0.5}
\definecolor{defseagreen}{cmyk}{0.69,0,0.50,0}

\definecolor{tyellow1}{HTML}{FCE94F}
\definecolor{tyellow2}{HTML}{EDD400}
\definecolor{tyellow3}{HTML}{C4A000}
\definecolor{torange1}{HTML}{FCAF3E}
\definecolor{torange2}{HTML}{F57900}
\definecolor{torange3}{HTML}{C35C00}
\definecolor{tbrown1}{HTML}{E9B96E}
\definecolor{tbrown2}{HTML}{C17D11}
\definecolor{tbrown3}{HTML}{8F5902}
\definecolor{tgreen1}{HTML}{8AE234}
\definecolor{tgreen2}{HTML}{73D216}
\definecolor{tgreen3}{HTML}{4E9A06}
\definecolor{tblue1}{HTML}{729FCF}
\definecolor{tblue2}{HTML}{3465A4}
\definecolor{tblue3}{HTML}{204A87}
\definecolor{tpurple1}{HTML}{AD7FA8}
\definecolor{tpurple2}{HTML}{75507B}
\definecolor{tpurple3}{HTML}{5C3566}
\definecolor{tred1}{HTML}{EF2929}
\definecolor{tred2}{HTML}{CC0000}
\definecolor{tred3}{HTML}{A40000}
\definecolor{tlgray1}{HTML}{EEEEEC}
\definecolor{tlgray2}{HTML}{D3D7CF}
\definecolor{tlgray3}{HTML}{BABDB6}
\definecolor{tdgray1}{HTML}{888A85}
\definecolor{tdgray2}{HTML}{555753}
\definecolor{tdgray3}{HTML}{2E3436}

\newtheorem{remark}{Remark}

\newcommand{\fml}[1]{{\mathcal{#1}}}
\newcommand{\mfrk}[1]{{\mathfrak{#1}}}
\newcommand{\set}[1]{\{ #1 \}}

\newcommand{\tn}[1]{\textnormal{#1}}

\newcommand{\mbf}[1]{\ensuremath\mathbf{#1}}
\newcommand{\mbb}[1]{\ensuremath\mathbb{#1}}
\newcommand{\msf}[1]{\ensuremath\mathsf{#1}}

\newcommand{\ddp}{\tn{D}^{\tn{P}}}

\newcommand{\bigland}{\ensuremath\bigwedge}

\newcommand{\cond}{d} 
\newcommand{\sep}{s}  

\newcommand{\bb}[1]{\ensuremath{[\![#1]\!]}}

\def\squareforqed{\hbox{\rlap{$\sqcap$}$\sqcup$}}
\def\qed{\ifmmode\squareforqed\else{\unskip\nobreak\hfil
\penalty50\hskip1em\null\nobreak\hfil\squareforqed
\parfillskip=0pt\finalhyphendemerits=0\endgraf}\fi}

\newcommand{\nocxp}[1]{}
\newcommand{\ignore}[1]{}

\DeclareMathOperator*{\nentails}{\nvDash}
\DeclareMathOperator*{\entails}{\vDash}
\DeclareMathOperator*{\lequiv}{\leftrightarrow}
\DeclareMathOperator*{\limply}{\rightarrow}

\AtBeginDocument{

}

\newcommand{\jnoteF}[1]{}

\newcommand{\znote}[1]{\medskip\noindent$\llbracket$\textcolor{darkred}{Yacine}: \emph{\textcolor{middarkgrey}{#1}}$\rrbracket$\medskip}

\SetAlgoNoEnd
\SetAlgoNoLine
\SetFillComment
\SetKwBlock{Let}{let}{}
\SetKwBlock{FBlock}{}{}
\SetKw{KwNot}{not\xspace}
\SetKw{KwAnd}{and\xspace}
\SetKw{KwOr}{or\xspace}
\SetKw{Break}{break\xspace}
\SetKwData{false}{{\small false}}
\SetKwData{true}{{\small true}}
\SetKwData{st}{\small{\sl st}}
\SetKwFunction{SAT}{SAT}
\SetKwFunction{Entails}{Entails}
\SetKwFunction{FindMHS}{FindMHS}
\SetKwFunction{Clauses}{Clausify}
\SetKwFunction{ExtractAXp}{ExtractAXp}
\SetKwFunction{ReportAXp}{ReportAXp}
\SetKwFunction{ExtractCXp}{ExtractCXp}
\SetKwFunction{ReportCXp}{ReportCXp}
\SetKwFunction{NegLits}{NegateLiteralsOf}
\SetKwFunction{PosLits}{UseLiteralsOf}
\SetKwBlock{Let}{let}{end}
\SetKwBlock{FBlock}{}{end}
\SetKwInput{KwVars}{Variables}
\SetKwInOut{Global}{Global}
\SetKwHangingKw{Algorithm}{Algorithm}
\SetKwHangingKw{Function}{function}
\SetKw{Func}{Function}
\SetKwBlock{Begin}{}{}

\graphicspath{{plots/}}
\DeclareGraphicsExtensions{.pdf,.png}

\DeclareMathAlphabet{\mathcal}{OMS}{cmsy}{m}{n}

\begin{document}

\title[Explanations for Tree Ensembles]{Rigorous Explanations for Tree Ensembles}

\author{Yacine Izza}
\authornote{Corresponding Author.}
\email{izza@comp.nus.edu.sg}
\orcid{TbD}
\affiliation{
  \institution{National University of Singapore}
  \country{Singapore}
}

\author{Alexey Ignatiev}
\email{alexey.ignatiev@monash.edu}
\orcid{TbD}
\affiliation{
  \institution{Monash University}
  \city{Melbourne}
  \country{Australia}
}

\author{Xuanxiang Huang}
\email{xuanxiang.huang@ntu.edu.sg}
\orcid{TbD}
\affiliation{
  \institution{Nanyang Technological University}
  \country{Singapore}
}

\author{Peter J.\ Stuckey}
\email{peter.stuckey@monash.edu}
\orcid{TbD}
\affiliation{
  \institution{Monash University}
  \city{Melbourne}
  \country{Australia}
}

\author{Joao Marques-Silva}
\email{joao.marques-silva@icrea.cat}
\orcid{TbD}
\affiliation{
  \institution{ICREA \& Univ. Lleida}
  \city{Lleida}
  \country{Spain}
}

\renewcommand{\shortauthors}{Izza et al.}


\begin{abstract}
  Tree ensembles (TEs) find a multitude of practical applications.
  They represent 
  one of the most general and accurate classes of machine learning
  methods. While they are typically quite concise in representation,
  their operation remains inscrutable to human decision makers.
  One solution to build trust in the operation of TEs is to
  automatically identify explanations for the predictions made.
  Evidently, we can only achieve trust using explanations, if those
  explanations are rigorous, that is truly reflect properties of the
  underlying predictor they explain
  %
  This paper investigates the computation of rigorously-defined,
  logically-sound explanations for the concrete case of two well-known
  examples of tree ensembles, namely random forests and boosted
  trees.
  (...) 
  TbD
\end{abstract}


\maketitle

\clearpage
\tableofcontents
\clearpage

\section{Introduction} \label{sec:intro}


%
%
%
%


The recent successes of Machine Learning (ML), coupled with the continuing expansion of ML-enabled applications, have further heightened the need for understanding and validating the decisions produced by modern AI systems. Despite unprecedented progress in predictive accuracy across countless domains, many of the most effective ML models operate as black boxes, offering no direct insight into the reasons underlying their predictions. This opacity constitutes a major obstacle to the widespread and responsible deployment of AI/ML technologies. Moreover, the increasing use of ML in high-risk and safety-critical settings---including healthcare, finance, autonomous systems, and public-sector decision-making---has motivated calls for rigorous, verifiable forms of explainability. Such concerns underpin recent expert recommendations for trustworthy AI, as well as emerging regulatory frameworks and proposals aimed at ensuring transparency and accountability in algorithmic decision-making (e.g., EU AI Act; OECD principles; national AI safety initiatives).

The rapid growth of research in Explainable AI (XAI) reflects this urgency. However, the dominant fraction of existing XAI work relies on heuristic or non-formal approaches, such as perturbation-based model-agnostic methods or saliency-based techniques~\cite{guestrin-kdd16b,lundberg-nips17,guestrin-aaai18,Zisserman-iclr24,Grad-CAM-ijcv20}. Although popular in practice, these approaches offer no formal guarantees with respect to the ML model being explained. Their explanations are often unstable, sensitive to noise, or inconsistent across equivalent models, and they may fail to capture the true semantics of the classifier. Consequently, their use in high-stakes contexts remains problematic, and recent studies have highlighted fundamental shortcomings of informal XAI methods~\cite{inms-corr19,ignatiev-ijcai20,msh-cacm24}.

In contrast, the emerging field of formal or logic-based XAI seeks to provide model-precise and semantically faithful explanations, grounded in rigorous logic-based definitions and computed through practically exact automated reasoners \cite{darwiche-ijcai18,inms-aaai19,ms-isola24}. Formal approaches enable guarantees of sufficiency, minimality, and correctness, thus ensuring that an explanation is consistent with the exact behaviour of the underlying ML model. Nevertheless, these benefits introduce nontrivial computational challenges: formal reasoning over complex classifiers may be computationally demanding, and careful encoding choices are essential to obtain scalable solutions.

Among the ML models where explainability is especially relevant, \emph{tree ensembles} (TEs)---including Random Forests (RFs), Gradient Boosted Trees (BTs), and widely used frameworks such as XGBoost---occupy a central role. TEs combine competitive accuracy with robustness, often outperforming deep neural networks in tabular data tasks. Yet despite the transparency of individual decision trees, the aggregation of many trees through majority or weighted voting yields predictive functions whose structure can be intricate and difficult to analyze. Recent work in formal XAI has begun addressing the explainability of tree ensembles, including abductive and constrastive explanations, feature necessity/relevancy in explanations, and inflated abductive/constrastive explanations~\cite{ignatiev-ijcai20,ims-ijcai21,iism-aaai22,hims-aaai23,hcmpms-tacas23,iisms-aaai24,iirmss-ijcai25}. 

\medskip
\noindent\textbf{Contributions.} 
This paper advances the state of the art in formal explainability of tree ensembles by introducing a unified logical framework that supports the explanation and verification of a broad class of ensemble architectures. Concretely, the main contributions are as follows:
\begin{enumerate}
    \item We propose a \emph{unified propositional encoding} capable of representing RFs with majority vote, RFs with weighted vote, and BT models such as XGBoost. This provides, to our knowledge, the first uniform logic-based framework that captures these different aggregation schemes within a common formalism.

    %
  \item Building on earlier MaxSAT-based work on boosted trees, we extend the MaxSAT approach~\cite{iism-aaai22} to Random Forests with weighted vote.\footnote{%
 The MaxSAT-based encoding for RFs with weighted voting is also documented in our recent CoRR report~\cite{iirmss-corr25}, which provides 
 an explicit timestamp for this contribution.} 
  Our encodings handle the linear aggregation constraints induced by non-uniform tree weights. Moreover, our approach supports the computation of smallest explanations w.r.t. cardinality size. 

    \item Building on the propositional encoding~\cite{ims-ijcai21}, we introduce formal encodings for checking \emph{fairness} and \emph{robustness} properties of RFs. These enable verification of sensitive-attribute fairness and adversarial vulnerability.

    \item We conduct an extensive experimental evaluation on a range of
      realistic tree-ensemble models. The results demonstrate that the
      proposed propositional reasoning based  approaches are scalable and efficient, frequently outperforming existing SMT-based methods, 
    thereby pushing the scalability frontier of logic-based XAI for tree ensembles. 
     
\end{enumerate}

Overall, this paper advances the state of formal XAI for tree ensembles along three axes: (i) it provides a single, coherent encoding framework covering several widely used ensemble semantics; (ii) it generalizes existing MaxSAT-based explanation techniques to more expressive aggregation schemes; and (iii) it shows that (minimum-size) explanation and verification tasks, including fairness and robustness, can be addressed within the same logic-based pipeline.

\section{Preliminaries} \label{sec:prelim}

\subsection{Logic Foundations}

\paragraph{SAT and MaxSAT}

We assume standard definitions for propositional satisfiability (SAT)
and maximum satisfiability (MaxSAT) solving~\cite{sat-handbook}.
A propositional formula is said to be in \emph{conjunctive}
\emph{normal form} (CNF) if it is a conjunction of clauses.
%
A \emph{clause} is a disjunction of literals.
%
A \emph{literal} is either a Boolean variable or its negation.
Whenever convenient, clauses are treated as sets of
literals.
%
%
A truth assignment maps each variable to $\{0,1\}$.
Given a truth assignment, a clause is satisfied if at least one of its
literals is assigned value 1; otherwise, it is falsified.
A formula is satisfied if all of its clauses are satisfied; otherwise,
it is falsified.
If there exists no assignment that satisfies a CNF formula, then the
formula is \emph{unsatisfiable}.

\paragraph{First Order Logic (FOL) and SMT}
When necessary, the paper will consider the restriction of FOL to
Satisfiability Modulo Theories (SMT). These represent restricted (and
often decidable) fragments of FOL~\cite{BarrettSST09}. All the
definitions above apply to SMT.
A SMT-solver reports whether a formula is satisfiable, and if so, may provide 
a model of this satisfaction. 
(Other possible features include dynamic addition and retraction of
constraints, production of proofs, and optimization.)

\paragraph{MaxSAT}
In the context of optimization over Boolean formulas,, Maximum Satisfiability 
(MaxSAT) is commonly used to model
problems where a set of constraints must be satisfied while 
optimizing a linear objective function.
A number of variants of MaxSAT
exist~\cite[Chapters~23~and~24]{sat-handbook}.
In this work, we adopt the perspective of \emph{Partial Weighted MaxSAT} 
formulated as a constrained optimization problem.
Formally, a formula $\phi$ consists of a set of \emph{hard}
constraints (clauses, in clausal form) $\fml{H}$, and 
a set of weighted literals $\fml{S}$ 
(also referred to as \emph{soft} clauses). 
%
%
The Partial Weighted MaxSAT problem consists in finding an assignment
that satisfies all the hard clauses $\fml{H}$ and maximizes the linear objective 
function ($\sum w_i \, l_i$) represented by $\fml{S}$. i.e.\ maximizes the total weight 
of satisfied literals $l_i$ of $\fml{S}$.
Note that in case $\phi$ is expressed in FOL, then optimizing $\fml{S}$ subject 
to constraints $\fml{H}$ is formulated as a \emph{MaxSMT} problem.

\paragraph{MUS, MCS \& MHS}
In the context of unsatisfiability formulas, an \emph{MUS} is the minimal 
subset of clauses of an unsatisfiable formula that causes inconsistency.\footnote{%
Informally, an MUS provides the minimal information that needs to be added 
to the background knowledge $\fml{H}$ to obtain an inconsistency}  
Let $\fml{H} \cup \fml{R} \entails \bot$ an inconsistent set of clauses.
A set $\fml{X} \subseteq \fml{R}$ is an MUS iff
$\fml{H} \cup \fml{X} \entails \bot$ and for every $\fml{X}' \subset \fml{X}$, we have
$\fml{H} \cup \fml{X}' \nentails \bot$.
Conversely, an \emph{MCS} is a minimal set of clauses to remove from an unsatisfiable 
formula to restore consistency. Formally, 
a set $\fml{Y} \subseteq \fml{R}$ is an MCS iff
$\fml{H} \cup (\fml{R} \setminus \fml{Y}) \nentails \bot$
and for every $\fml{Y}' \subset \fml{Y}$, we have
$\fml{H} \cup (\fml{R} \setminus \fml{Y}') \entails \bot$.
It is well-known that MCSes are minimal hitting sets (MHSes)
of MUSes and vice-versa~\cite{reiter-aij87,lozinskii-jetai03}.  
A set $\fml{X}$ is a MHS of a set 
of sets $\mbb{Y}$ if $\fml{X} \cap \fml{Y}_i \neq \emptyset, \forall \fml{Y}_i \in \mbb{Y}$ 
and no subset of $\fml{X}$ also satisfies this condition.

\subsection{Machine Learning Models}

\paragraph{Classification models.}
We consider a classification problem, characterized by a set of
features $\fml{F}=\{1,\ldots,m\}$, and by a set of classes
$\fml{K}=\{c_1,\ldots,c_K\}$.
Each feature $j\in\fml{F}$ is characterized by a domain $D_j$.  As a
result, feature space is defined as
$\mbb{F}={D_1}\times{D_2}\times\ldots\times{D_m}$.
A specific point in feature space represented by
$\mbf{v}=(v_1,\ldots,v_m)$
%
denotes an \emph{instance} (or an
\emph{example}).
Also, we use $\mbf{x}=(x_1,\ldots,x_m)$ to denote an arbitrary
point in feature space.
In general, when referring to the value of a feature $j\in\fml{F}$, we
will use a variable $x_j$, with $x_j$ taking values from $D_j$.
A classifier implements a \emph{total classification function}
$\kappa:\mbb{F}\to\fml{K}$.
For technical reasons, we also require $\kappa$ not to be a constant
function, i.e.\ there exists at least two points in feature space with
differing predictions.
An \emph{instance} denotes a pair $(\mbf{v}, c)$, where
$\mbf{v}\in\mbb{F}$ and $c\in\fml{K}$.
We represent a classifier by a tuple
$\fml{M}=(\fml{F},\mbb{F},\fml{K},\kappa)$.
Given the above, an \emph{explanation problem} is a tuple
$\fml{E} = (\fml{M},(\mbf{v},c))$, where $\kappa(\mbf{v}) = c$.

\paragraph{Regression models.}
For regression models, a model similar to the classification case is assumed. Instead 
of classes we now assume a set of values $\fml{V}$. Moreover, the regression 
function $\rho$ maps feature space to the set of values, $\rho:\mbb{F}\to\fml{V}$. 
Although the paper focuses solely on classification models, we could easily 
adopt a unified ML model with a prediction function, 
without explicitly distinguishing classification from regression~\cite{ms-isola24}.

\subsection{Formal Explainability} \label{ssec:fxai}

\paragraph{Interpretability \& explanations.}
Interpretability is generally accepted to be a subjective concept,
without a formal definition~\cite{lipton-cacm18}. In this paper we
measure interpretability in terms of the overall succinctness of the
information provided by an ML model to justify a given prediction.
%
%
%
Moreover, and building on earlier work, we equate explanations with
the so-called \emph{abductive explanations}
(AXps)~\cite{darwiche-ijcai18,inms-aaai19,inms-nips19,darwiche-ecai20,iims-corr20,inams-aiia20,barcelo-nips20,msgcin-nips20,msgcin-icml21,hiims-kr21,msi-aaai22,hiicams-aaai22},
i.e.\ subset-minimal sets of feature-value pairs that are sufficient
for the prediction.
More formally, given an instance $\mbf{v}\in\mbb{F}$, with prediction
$c\in\fml{K}$, i.e.\ $\kappa(\mbf{v})=c$, an AXp is a minimal subset
$\fml{X}\subseteq\fml{F}$ such that,
\begin{equation} \label{eq:axp}
\forall(\mbf{x}\in\mbb{F}).\bigwedge\nolimits_{i\in\fml{X}}(x_i=v_i)\limply(\kappa(\mbf{x})=c)
\end{equation}
Another name for an AXp is a \emph{prime implicant} (PI)
explanation~\cite{darwiche-ijcai18,darwiche-ecai20}.
For simplicity,
we will use abductive explanation and the acronym \emph{AXp}
interchangeably.

In a similar vein, we consider \emph{contrastive explanations}
(CXps)~\cite{miller-aij19,inams-aiia20}.
Contrastive explanation can be defined as a (subset-)minimal set of
feature-value pairs ($\fml{Y}\subseteq\fml{F}$) that suffice to
change the prediction if they are allowed to take \emph{some
arbitrary} value from their domain.
Formally and as suggested in~\cite{inams-aiia20}, a CXp is defined as
a minimal subset $\fml{Y}\subseteq\fml{F}$ such that,
\begin{equation} \label{eq:cxp}
\exists(\mbf{x}\in\mbb{F}).\bigwedge\nolimits_{i\not\in\fml{Y}}(x_i=v_i)\land(\kappa(\mbf{x})\not=c)
\end{equation}
(It is possible and simple to adapt the definition to target a
specific class $c'\not=c$).
Moreover, building on the seminal work of R.~Reiter
in~\cite{reiter-aij87}, recent work demonstrated a minimal hitting set
relationship between AXps and CXps~\cite{inams-aiia20}, namely each
AXp is a minimal hitting set (MHS) of the set of CXps and vice-versa.
%
%
%

Illustrative examples of abductive and contrastive explanations 
for tree ensemble classifiers are shown in the next section.

\section{Tree Ensembles}
Decision
trees~\cite{rivest-ipl76,breiman-bk84,quinlan-ml86,quinlan-book93} are
one of the oldest forms of ML model, and are very widely used, but
they are known to suffer from bias if the tree is small, and tend to
overfit the data if the tree is large.
They are commonly deemed to have a significant advantage over most ML
models in that their decisions appear easy to explain to a human,
although~\citet{iims-corr20} show that this does not hold in
general.

In order to overcome the weaknesses of decision trees, \emph{tree
ensembles} were introduced, which generate many decision trees and
rely on aggregating their decisions to come to an overall decision.
Two popular ensemble methods are \emph{random forests}
(RFs)~\cite{Breiman01} and \emph{gradient boosted trees}
(BTs)~\cite{friedman-tas01}.

Therefore, we proposes a unified representation for tree ensembles
(TEs). The unified representation will be instrumental when relating
concrete instantiations of TEs.

\subsection{General TE model}
We propose a general model for selecting a class in a tree ensemble
(TE). As shown below, the proposed model serves to represent boosted
trees~\cite{friedman-tas01}, random forests with majority
voting~\cite{Breiman01}, and random forests with weighted
voting (Scikit-learn~\cite{scikitlearn}).

A decision tree $T$ on features $\fml{F}$ for classes $\fml{K}$
is a binary rooted-tree whose internal nodes are labeled 
by \emph{split conditions}   
which are predicates of the form $x_i < d$ for some feature 
$i \in \fml{F}$ and $d \in \mbb{D}_i$, and leaf nodes labelled by a class
$c \in \fml{K}$. 
A path $R_\mfrk{n}$ to a leaf node $\mfrk{n}$ can be considered
as the set of split conditions 
of nodes 
visited on the path from the root to $\mfrk{n}$.

A tree ensemble $\mfrk{T}$ has $n$ decision trees, $T_1,\ldots,T_n$. 
For each tree $T_t$, let $\mathcal{L}(T_t)$ denote the set of its leaf nodes.
Each leaf $l \in \mathcal{L}(T_t)$ induces a root-to-leaf path $R_l$ 
and is associated with a class-label mapping $cl : \mathcal{L}(T_t) \rightarrow \mathcal{K}$, where 
$cl(l)$ is the class predicted by $l$,
and $w_l \in \mathbb{R}$ denotes its associated weight.
Given a point $\mbf{v}$ in feature space, and for each decision tree
$T_t$, with $1\le{t}\le{n}$, there is a single path $R_{k_{\mbf{v},t}}$
that is consistent with $\mbf{v}$ in tree $T_t$.
The index $k_{\mbf{v},t}$
denotes which path of $T_t$ is consistent with $\mbf{v}$.

For decision tree $T_t$, and given $\mbf{v}$, the picked class is 
denoted $cl(k_{\mbf{v},t})$  and the picked weight $w_{k_{\mbf{v},t}}$.
For each class $c \in\fml{K}$, the class weight will be computed by,
\begin{equation} \label{eq:wght}
W(c,\mbf{v})=\sum\nolimits_{t=1}^{n} \textbf{if}~ c=cl(k_{{\mbf{v}},t})
\textbf{~then~} w_{k_{{\mbf{v}},t}} \textbf{~else~} 0
\end{equation}
Finally, the predicted class by $\mfrk{T}$ is given by,
$$
\kappa(\mbf{v}) = \arg\max\nolimits_{c\in\fml{K}}\{W(c,\mbf{v})\} 
$$

\begin{figure}[ht]
  \centering
    \resizebox{\textwidth}{!}{
\tikzstyle{box} = [draw=black!90, thick, rectangle, rounded corners,
                     inner sep=10pt, inner ysep=20pt, dotted
                  ]
\tikzstyle{title} = [draw=black!90, fill=black!5, semithick, top color=white,
                     bottom color = black!5, text=black!90, rectangle,
                     font=\small, inner sep=2pt, minimum height=1.3em,
                     top color=tyellow2!27, bottom color=tyellow2!27
                    ]
\tikzstyle{feature} = [rectangle,font=\scriptsize,rounded corners=1mm,thick,%
                       draw=black!80, top color=tblue2!20,bottom color=tblue2!25,%
                       draw, minimum height=1.1em, text centered,%
                       inner sep=2pt%
                      ]
\tikzstyle{pscore} = [rectangle,font=\scriptsize,rounded corners=1mm,thick,%
                     draw=black!80, top color=tgreen3!20,bottom color=tgreen3!27,%
                     draw, minimum height=1.1em, text centered,%
                     inner sep=2pt%
                    ]
\tikzstyle{nscore} = [rectangle,font=\scriptsize,rounded corners=1mm,thick,%
                     draw=black!80, top color=tred2!20,bottom color=tred2!25,%
                     draw, minimum height=1.1em, text centered,%
                     inner sep=2pt%
                    ]
\tikzstyle{score} = [rectangle,font=\scriptsize,rounded corners=1mm,thick,%
                     draw=black!80, top color=tgreen3!10,bottom color=tgreen3!15,%
                     draw, minimum height=1.1em, text centered,%
                     inner sep=2pt%
                    ]  
\begin{adjustbox}{center}
\setlength{\tabcolsep}{12pt}
\def\arraystretch{3}
\begin{tabular}{ccc}
    \begin{tikzpicture}[node distance = 4.0em, auto]
        \node [box] (box) {%
        \begin{minipage}[t!]{0.154\textwidth}
            \vspace{0.9cm}\hspace{1.5cm}
        \end{minipage}
        };
        \node[title] at (box.north) {$\text{T}_\text{1}$ (setosa)};

        \node [feature] (feat) at (0, 0.25) {petal.length $<$ 2.45?};

        \node [score, below left = 0.9em and -1.4em of feat] (pos)  { 0.42762};
        \node [score, below right = 0.9em and -1.6em of feat] (neg) {-0.21853};

        \draw [->,thick,black!80] (feat) to[] node[above, pos=1.2, font=\scriptsize] {yes} (pos.north);
        \draw [->,thick,black!80] (feat) to[] node[above, pos=1.1, font=\scriptsize] { no} (neg.north);
    \end{tikzpicture}
    &
    \begin{tikzpicture}[node distance = 4.0em, auto]
        \node [box] (box) {%
        \begin{minipage}[t!]{0.264\textwidth}
            \vspace{0.9cm}\hspace{1.5cm}
        \end{minipage}
        };
        \node[title] at (box.north) {$\text{T}_\text{2}$ (versicolor)};

        \node [feature] (feat1) at (0.04, 0.57) {sepal.width $<$ 2.95?};
        \node [feature, below left  = 0.8em and -2.em of feat1] (feat2) {petal.width $<$ 1.7?};
        \node [feature, below right = 0.8em and -2.em of feat1] (feat3) {petal.length $<$ 3?};

        \node [score, below left  = 0.8em and -2.3em of feat2] (pos1) {0.36131};
        \node [score, below right  = 0.8em and -2.5em of feat2] (neg1) {-0.18947};
        \node [score, below left = 0.8em and -2.4em of feat3] (pos2) {-0.21356};
        \node [score, below right = 0.8em and -2.4em of feat3] (neg2) {-0.03830};

        \draw [->,thick,black!80] (feat1) to[] node[above, pos=1.2, font=\scriptsize] {yes} (feat2.north);
        \draw [->,thick,black!80] (feat1) to[] node[above, pos=1.1, font=\scriptsize] { no} (feat3.north);
        \draw [->,thick,black!80] (feat2) to[] node[above, pos=1.2, font=\scriptsize] {yes} (pos1.north);
        \draw [->,thick,black!80] (feat2) to[] node[above, pos=1.1, font=\scriptsize] { no} (neg1.north);
        \draw [->,thick,black!80] (feat3) to[] node[above, pos=1.2, font=\scriptsize] {yes} (pos2.north);
        \draw [->,thick,black!80] (feat3) to[] node[above, pos=1.1, font=\scriptsize] { no} (neg2.north);
    \end{tikzpicture}
    &
    \begin{tikzpicture}[node distance = 4.0em, auto]
        \node [box] (box) {%
        \begin{minipage}[t!]{0.215\textwidth}
            \vspace{0.9cm}\hspace{1.5cm}
        \end{minipage}
        };
        \node[title] at (box.north) {$\text{T}_\text{3}$ (virginica)};

        \node [feature] (feat1) at (-0.375, 0.57) {petal.length $<$ 4.75?};
        \node [feature, below right = 0.8em and -2.2em of feat1] (feat2) {petal.width $<$ 1.7?};

        \node [score, below left  = 0.8em and -1.2em of feat1] (pos1) {-0.21869};
        \node [score, below left = 0.8em and -2.4em of feat2] (pos2) {0.08182};
        \node [score, below right = 0.8em and -2.4em of feat2] (neg) {0.42282};

        \draw [->,thick,black!80] (feat1) to[] node[above, pos=1.2, font=\scriptsize] {yes} (pos1.north);
        \draw [->,thick,black!80] (feat1) to[] node[above, pos=1.1, font=\scriptsize] { no} (feat2.north);
        \draw [->,thick,black!80] (feat2) to[] node[above, pos=1.2, font=\scriptsize] {yes} (pos2.north);
        \draw [->,thick,black!80] (feat2) to[] node[above, pos=1.1, font=\scriptsize] { no} (neg.north);
    \end{tikzpicture}
    \\
    \begin{tikzpicture}[node distance = 4.0em, auto]
        \node [box] (box) {%
        \begin{minipage}[t!]{0.154\textwidth}
            \vspace{0.9cm}\hspace{1.5cm}
        \end{minipage}
        };
        \node[title] at (box.north) {$\text{T}_\text{4}$ (setosa)};

        \node [feature] (feat) at (0, 0.25) {petal.length $<$ 2.45?};

        \node [score, below left = 0.9em and -1.4em of feat] (pos) {0.29522};
        \node [score, below right = 0.9em and -1.6em of feat] (neg) {-0.19674};

        \draw [->,thick,black!80] (feat) to[] node[above, pos=1.2, font=\scriptsize] {yes} (pos.north);
        \draw [->,thick,black!80] (feat) to[] node[above, pos=1.1, font=\scriptsize] { no} (neg.north);
    \end{tikzpicture}
    &
    \begin{tikzpicture}[node distance = 4.0em, auto]
        \node [box] (box) {%
        \begin{minipage}[t!]{0.275\textwidth}
            \vspace{0.9cm}\hspace{1.5cm}
        \end{minipage}
        };
        \node[title] at (box.north) {$\text{T}_\text{5}$ (versicolor)};

        \node [feature] (feat1) at (0.125, 0.57) {sepal.width $<$ 2.95?};
        \node [feature, below left  = 0.8em and -2.em of feat1] (feat2) {petal.length $<$ 4.85?};
        \node [feature, below right = 0.8em and -2.em of feat1] (feat3) {petal.length $<$ 3?};

        \node [score, below left  = 0.8em and -2.3em of feat2] (pos1) {0.27994};
        \node [score, below right  = 0.8em and -2.5em of feat2] (neg1) {-0.11330};
        \node [score, below left = 0.8em and -2.4em of feat3] (pos2) {-0.18999};
        \node [score, below right = 0.8em and -2.4em of feat3] (neg2) {-0.02829};

        \draw [->,thick,black!80] (feat1) to[] node[above, pos=1.2, font=\scriptsize] {yes} (feat2.north);
        \draw [->,thick,black!80] (feat1) to[] node[above, pos=1.1, font=\scriptsize] { no} (feat3.north);
        \draw [->,thick,black!80] (feat2) to[] node[above, pos=1.2, font=\scriptsize] {yes} (pos1.north);
        \draw [->,thick,black!80] (feat2) to[] node[above, pos=1.1, font=\scriptsize] { no} (neg1.north);
        \draw [->,thick,black!80] (feat3) to[] node[above, pos=1.2, font=\scriptsize] {yes} (pos2.north);
        \draw [->,thick,black!80] (feat3) to[] node[above, pos=1.1, font=\scriptsize] { no} (neg2.north);
    \end{tikzpicture}
    &
    \begin{tikzpicture}[node distance = 4.0em, auto]
        \node [box] (box) {%
        \begin{minipage}[t!]{0.215\textwidth}
            \vspace{0.9cm}\hspace{1.5cm}
        \end{minipage}
        };
        \node[title] at (box.north) {$\text{T}_\text{6}$ (virginica)};

        \node [feature] (feat1) at (-0.375, 0.57) {petal.length $<$ 4.75?};
        \node [feature, below right = 0.8em and -2.2em of feat1] (feat2) {petal.width $<$ 1.7?};

        \node [score, below left  = 0.8em and -1.2em of feat1] (pos1) {-0.19776};
        \node [score, below left = 0.8em and -2.4em of feat2] (pos2) {0.08067};
        \node [score, below right = 0.8em and -2.4em of feat2] (neg) {0.30170};

        \draw [->,thick,black!80] (feat1) to[] node[above, pos=1.2, font=\scriptsize] {yes} (pos1.north);
        \draw [->,thick,black!80] (feat1) to[] node[above, pos=1.1, font=\scriptsize] { no} (feat2.north);
        \draw [->,thick,black!80] (feat2) to[] node[above, pos=1.2, font=\scriptsize] {yes} (pos2.north);
        \draw [->,thick,black!80] (feat2) to[] node[above, pos=1.1, font=\scriptsize] { no} (neg.north);
    \end{tikzpicture}
\end{tabular}
\end{adjustbox}

}
    \caption{\small Example of a simple boosted tree
      model~\protect\cite{guestrin-kdd16a} generated by XGBoost on the
     well-known Iris classification dataset. Here, the tree
      ensemble has 2 trees for each of the 3 classes with the depth of
      each tree being at most 2.}
    \label{fig:BT}
\end{figure}
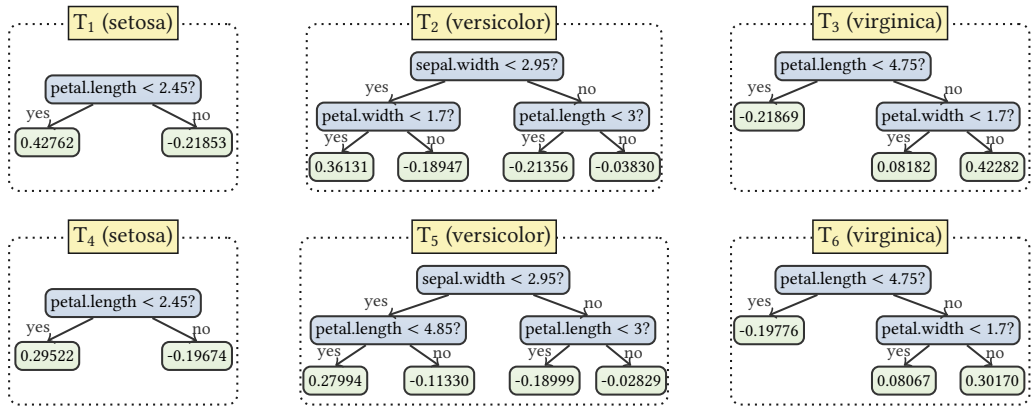

\subsection{Boosted trees}
Here, we focus on BT models trained with the use of the XGBoost
algorithm~\cite{guestrin-kdd16a}.
In general, BTs are a tree ensemble that relies on building a sequence
of small decision trees.
In each stage, new trees are built to compensate for the inaccuracies
of the already constructed tree ensemble.
Therefore, each tree is pre-assigned a class $c_i$ and 
contributes to the total class weight $W(c_i,\mbf{x})$; the weights of each class 
are summed up over all trees, and the class with the largest weight is
picked.

\begin{example} \label{ex:bt}
  An example of a BT trained by XGBoost for the Iris dataset is shown
  in~\cref{fig:BT}.
  Each class $i\in[3]$ is represented in the BT by 2 trees $T_{3j+i}$,
  $j\in\{0, 1\}$.
  This way, given an input instance $(\text{\em
  sepal.length}=\text{\em 5.1}) \land (\text{\em
  sepal.width}=\text{\em 3.5}) \land (\text{\em petal.length}=\text{\em
  1.4}) \land (\text{\em petal.width}=\text{\em 0.2})$, the scores of
  classes 1--3 are $w_1=T_1+T_4=0.72284$,
  $w_2=T_2+T_5=-0.40355$, and
  $w_3=T_3+T_6=-0.41645$.
  Hence, the model predicts class~1 (``Setosa'') as it has the highest
  score.
  The path taken by this instance in tree $T_2$ is defined as $\{
  \text{\em sepal.width} \geq \text{\em 2.95}, \text{\em petal.length} <
  \text{\em 3} \}$.
\end{example}

\begin{example} \label{ex:axp}
  Consider the BT from Example~\ref{ex:bt} and the same data instance.
  By examining the model, one can observe that it suffices to specify
  $\text{\em petal.length}=\text{\em 1.4}$
  to guarantee that the prediction is ``Setosa''.
  The weight for ``Setosa'' will be 0.72284 as before, and the maximal
  weight for ``Versicolor'' will be 0.36131 + 0.27994 = 0.64125, and
  the weight for ``Virginica'' will be  -0.41645 as before.
  Therefore, an abductive explanation for this prediction includes a
  single feature.
  Furthermore, it is the only AXp for the given instance.
  \qed
\end{example}

To embed a boosted tree model into our general tree ensemble 
framework, we proceed as follows.
Any tree in the BT will be mapped to a tree in the TE model.
Each tree path weight in the BT will be mapped to the the same tree
path weight in the TE model. Finally, for each decision tree in the
group of trees assigned to class $c$ in the BT, any path in the
corresponding tree in the TE model will output class $c$.

Formally, let a BT model be given by a set of trees 
\[
\mathfrak{B} \;=\; \bigcup_{i=1}^{K} \{\, T_{Kj+i}  \,\}_{j = 0}^{\,n-1}
\]
where each tree $T_{Kj+i}$ is associated with a class $c_i\in\mathcal{K}$, and 
each leaf $l \in \mathcal{L}(T_{Kj+i})$ carries a real-valued 
score $w_l \in \mathbb{R}$.

We define an equivalent TE $\mathfrak{T}$ as follows. 
For each $T_{Kj+i} \in \mathfrak{B}$, we introduce a corresponding 
tree $T'_{Kj+i} \in \mathfrak{T}$ with identical structure and leaves. For every 
leaf $l\in \mathcal{L}(T_{Kj+i})$ we assign: 
\[
cl(l) \;:=\; c_i
\qquad\text{and}\qquad
w_l^{\mathfrak{T}} \;:=\; w_l^{\mathfrak{B}}
\]

For any input instance $\mathbf{v}$ and class $c_i \in \mathcal{K}$, the aggregated score
induced by $\mathfrak{T}$ is given by
\[
    W(c_i,\mbf{v})=\sum\nolimits_{j=0}^{n-1}  w_{k_{{\mbf{v}}, Kj+i}} 
\]

This construction preserves the prediction semantics of the original BT model,
i.e.,
\[
\arg\max\nolimits_{c_i \in \mathcal{K}}   W_{\mathfrak{T}}(c_i, \mathbf{v})
\;=\;
\arg\max\nolimits_{c_i \in \mathcal{K}}  W_{\mathfrak{B}}(c_i, \mathbf{v})
\]

\subsection{Random forests with majority voting}
%
In a random forest with majority voting (RFmv)~\cite{Breiman01}, 
and given a point in feature space, each decision tree picks a class;
the final prediction corresponds to the class with the maximal vote
 count among the ensemble of trees.
%

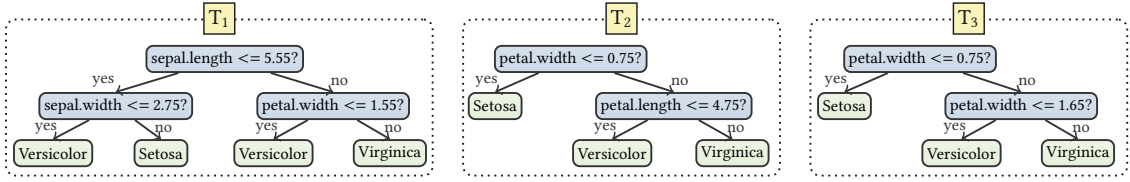
\begin{figure}[ht]
  \centering
    \resizebox{\textwidth}{!}{
\tikzstyle{box} = [draw=black!90, thick, rectangle, rounded corners,
                     inner sep=10pt, inner ysep=20pt, dotted
                  ]
\tikzstyle{title} = [draw=black!90, fill=black!5, semithick, top color=white,
                     bottom color = black!5, text=black!90, rectangle,
                     font=\small, inner sep=2pt, minimum height=1.3em,
                     top color=tyellow2!27, bottom color=tyellow2!27
                    ]
\tikzstyle{feature} = [rectangle,font=\scriptsize,rounded corners=1mm,thick,%
                       draw=black!80, top color=tblue2!20,bottom color=tblue2!25,%
                       draw, minimum height=1.1em, text centered,%
                       inner sep=2pt%
                      ]
\tikzstyle{score} = [rectangle,font=\scriptsize,rounded corners=1mm,thick,%
                     draw=black!80, top color=tgreen3!10,bottom color=tgreen3!15,%
                     draw, minimum height=1.1em, text centered,%
                     inner sep=2pt%
                    ]
\begin{adjustbox}{max width=\textwidth}
\setlength{\tabcolsep}{6pt}
\def\arraystretch{3}
\begin{tabular}{ccc}
    \begin{tikzpicture}[node distance = 4.0em, auto]
        \node [box] (box) {%
        \begin{minipage}[t!]{0.352\textwidth}
            \vspace{0.9cm}\hspace{1.5cm}
        \end{minipage}
        };
        \node[title] at (box.north) {$\text{T}_\text{1}$};

        \node [feature] (feat1) at (0.08, 0.57) {sepal.length $<=$ 5.55?};
        \node [feature, below left  = 0.8em and -2.em of feat1] (feat2) {sepal.width $<=$ 2.75?};
        \node [feature, below right = 0.8em and -2.em of feat1] (feat3) {petal.width $<=$ 1.55?};

        \node [score, below left  = 0.8em and -2.3em of feat2] (pos1) {Versicolor};
        \node [score, below right  = 0.8em and -2.5em of feat2] (neg1) {Setosa};
        \node [score, below left = 0.8em and -2.4em of feat3] (pos2) {Versicolor};
        \node [score, below right = 0.8em and -2.4em of feat3] (neg2) {Virginica};

        \draw [->,thick,black!80] (feat1) to[] node[above, pos=1.2, font=\scriptsize] {yes} (feat2.north);
        \draw [->,thick,black!80] (feat1) to[] node[above, pos=1.1, font=\scriptsize] { no} (feat3.north);
        \draw [->,thick,black!80] (feat2) to[] node[above, pos=1.2, font=\scriptsize] {yes} (pos1.north);
        \draw [->,thick,black!80] (feat2) to[] node[above, pos=1.1, font=\scriptsize] { no} (neg1.north);
        \draw [->,thick,black!80] (feat3) to[] node[above, pos=1.2, font=\scriptsize] {yes} (pos2.north);
        \draw [->,thick,black!80] (feat3) to[] node[above, pos=1.1, font=\scriptsize] { no} (neg2.north);
    \end{tikzpicture}
    &
    \begin{tikzpicture}[node distance = 4.0em, auto]
        \node [box] (box) {%
        \begin{minipage}[t!]{0.25\textwidth}
            \vspace{0.9cm}\hspace{1.5cm}
        \end{minipage}
        };
        \node[title] at (box.north) {$\text{T}_\text{2}$};

        \node [feature] (feat1) at (-0.75, 0.57) {petal.width $<=$ 0.75?};
        \node [feature, below right = 0.8em and -2.2em of feat1] (feat2) {petal.length $<=$ 4.75?};

        \node [score, below left  = 0.8em and -1.2em of feat1] (pos1) {Setosa};
        \node [score, below left = 0.8em and -2.4em of feat2] (pos2) {Versicolor};
        \node [score, below right = 0.8em and -2.4em of feat2] (neg) {Virginica};

        \draw [->,thick,black!80] (feat1) to[] node[above, pos=1.2, font=\scriptsize] {yes} (pos1.north);
        \draw [->,thick,black!80] (feat1) to[] node[above, pos=1.1, font=\scriptsize] { no} (feat2.north);
        \draw [->,thick,black!80] (feat2) to[] node[above, pos=1.2, font=\scriptsize] {yes} (pos2.north);
        \draw [->,thick,black!80] (feat2) to[] node[above, pos=1.1, font=\scriptsize] { no} (neg.north);
    \end{tikzpicture}
    &
    \begin{tikzpicture}[node distance = 4.0em, auto]
        \node [box] (box) {%
        \begin{minipage}[t!]{0.25\textwidth}
            \vspace{0.9cm}\hspace{1.5cm}
        \end{minipage}
        };
        \node[title] at (box.north) {$\text{T}_\text{3}$};

        \node [feature] (feat1) at (-0.72, 0.57) {petal.width $<=$ 0.75?};
        \node [feature, below right = 0.8em and -2.2em of feat1] (feat2) {petal.width $<=$ 1.65?};

        \node [score, below left  = 0.8em and -1.2em of feat1] (pos1) {Setosa};
        \node [score, below left = 0.8em and -2.4em of feat2] (pos2) {Versicolor};
        \node [score, below right = 0.8em and -2.4em of feat2] (neg) {Virginica};

        \draw [->,thick,black!80] (feat1) to[] node[above, pos=1.2, font=\scriptsize] {yes} (pos1.north);
        \draw [->,thick,black!80] (feat1) to[] node[above, pos=1.1, font=\scriptsize] { no} (feat2.north);
        \draw [->,thick,black!80] (feat2) to[] node[above, pos=1.2, font=\scriptsize] {yes} (pos2.north);
        \draw [->,thick,black!80] (feat2) to[] node[above, pos=1.1, font=\scriptsize] { no} (neg.north);
    \end{tikzpicture} %
  \end{tabular}
\end{adjustbox}}
    \caption{
        Example of a simple RFmv trained using Scikit-learn with 3 trees, on {\it Iris} dataset. 
        Here, the tree ensemble has 3 trees  with the depth of each tree being at most 2.  
        In the running examples, classes ``Setosa'', ``Versicolor'' and  ``Virginica'' are 
        denoted, resp.\  class $c_1$, $c_2$ and $c_3$, and features {\it sepal.length},  
        {\it sepal.width}, {\it petal.length} and  {\it petal.width} are ordered, resp., as 
        features 1--4.    
    }
    \label{fig:RF(mv)}
\end{figure}

\begin{example} \label{ex:RF(mv)}
  \cref{fig:RF(mv)} shows an example of a RFmv of 3 trees, trained 
  with Scikit-learn on {\it Iris} dataset. 
  Consider a sample $\mbf{v} = (6.0, 3.5, 1.4, 0.2)$ 
  then the counts of votes for classes 1--3 are, resp.,
  $W(c_1,\mbf{v}) = 2$ (votes of T$_2$ and T$_3$), $W(c_2,\mbf{v}) =1$ 
  (vote of T$_1$),   and  $W(c_3,\mbf{v}) = 0$. 
  Hence, the predicted class is $c_1$ (``Setosa'') as it has the highest
  score.
\end{example}

To represent an RF with majority voting in our general TE framework, we
proceed as follows. The weight of every path in each tree is $1$.
Moreover, the class output by each tree path in the TE is the class
predicted by the corresponding tree path in the RFmv.

Formally, let an RFmv model be given by a set of trees 
$ \mathfrak{F} \;=\;  \{\, T_t  \}_{t=1}^{\,n} $  
where each leaf $l \in \mathcal{L}(T_t)$ of a  
 tree $T_t$ is associated with a class $cl_{\mathfrak{F}}(l) = c_i, \, c_i\in\mathcal{K}$.
 
We define an equivalent TE $\mathfrak{T}$ as follows. 
For each $T_t \in \mathfrak{F}$, we introduce a corresponding 
tree $T'_t \in \mathfrak{T}$ with identical structure and leaf set. 
For every leaf $l \in \mathcal{L}(T_t)$ we assign:
\[
cl(l) \;:=\; cl_{\mathfrak{F}}(l)
\qquad\text{and}\qquad
w_l \;:=\; 1
\]

The aggregated class score is computed in the same way as in~\cref{eq:wght}.

\subsection{Random forests with weighted voting}
%
In an RF with weighted voting (RFwv)~\cite{scikitlearn},\footnote{%
The presented RFwv definition follows the random forest implementation 
provided by Scikit-learn} and given a point in feature space,
each tree picks a weight for a chosen selected class; the final picked
class is the class with the largest weight.

As a result, for a given input sample $\mbf{v}$, each tree $T_t$ produces 
a class probability estimate $P_{T_t}(c \mid \mbf{v})$ for class $c\in\fml{K}$.
The class weight for each $c \in \fml{K}$ is obtained by averaging
the class probability estimates across all trees:

\[
W(c, \mbf{v}) = \frac{1}{n} \sum_{t=1}^{n} P_{T_t}(c \mid \mbf{v}), \quad c \in \fml{K}
\]


\begin{figure}[ht]
   \centering
    \resizebox{\textwidth}{!}{
\tikzstyle{box} = [draw=black!90, thick, rectangle, rounded corners,
                     inner sep=10pt, inner ysep=20pt, dotted
                  ]
\tikzstyle{title} = [draw=black!90, fill=black!5, semithick, top color=white,
                     bottom color = black!5, text=black!90, rectangle,
                     font=\small, inner sep=2pt, minimum height=1.3em,
                     top color=tyellow2!27, bottom color=tyellow2!27
                    ]
\tikzstyle{feature} = [rectangle,font=\scriptsize,rounded corners=1mm,thick,%
                       draw=black!80, top color=tblue2!20,bottom color=tblue2!25,%
                       draw, minimum height=1.1em, text centered,%
                       inner sep=2pt%
                      ]
\tikzstyle{pscore} = [rectangle,font=\scriptsize,rounded corners=1mm,thick,%
                     draw=black!80, top color=tgreen3!20,bottom color=tgreen3!27,%
                     draw, minimum height=1.1em, text centered,%
                     inner sep=2pt%
                    ]
\tikzstyle{scores} = [rectangle,font=\scriptsize,rounded corners=1mm,thick,%
                     draw=black!80, top color=tgreen3!10,bottom color=tgreen3!15,%
                     draw, minimum height=1.1em, text centered,%
                     inner sep=2pt%
                    ]
\tikzstyle{nscore} = [rectangle,font=\scriptsize,rounded corners=1mm,thick,%
                     draw=black!80, top color=tred2!20,bottom color=tred2!25,%
                     draw, minimum height=1.1em, text centered,%
                     inner sep=2pt%
                    ]

\begin{adjustbox}{center}
\setlength{\tabcolsep}{12pt}
\def\arraystretch{3}
\begin{tabular}{cc}
   \multicolumn{2}{c}{
    \begin{tikzpicture}[node distance = 4.0em, auto]
        \node [box] (box) {%
        \begin{minipage}[t!]{0.542\textwidth}
            \vspace{1.4cm}\hspace{1.5cm}
        \end{minipage}
        };
        \node[title] at (box.north) {$\text{T}_\text{1} $};

        \node [feature] (feat1) at (0.04, 0.9) {sepal.length $<=$ 5.55?};
        \node [feature, below left  = 0.8em and -0.2em of feat1] (feat2) {sepal.width $<=$ 2.75?};
        \node [feature, below right = 0.8em and -0.1em of feat1] (feat3) {petal.width $<=$ 1.55?};

        \node [scores, below left  = 0.8em and -2.3em of feat2, align=left] (pos1)  { Setosa: 0.307 \\  {\bf Versicolor}: {\bf 0.615} \\ Virginica: 0.076};
        \node [scores, below right  = 0.8em and -2.5em of feat2, align=left] (neg1) { {\bf Setosa}: {\bf 1.0} \\  Versicolor: 0.0 \\  Virginica: 0.0};
        \node [scores, below left = 0.8em and -2.4em of feat3, align=left] (pos2) { Setosa: 0.093 \\  {\bf Versicolor}: {\bf 0.843} \\  Virginica: 0.062};
        \node [scores, below right = 0.8em and -2.4em of feat3, align=left] (neg2) { Setosa: 0.0 \\  Versicolor: 0.047 \\   {\bf Virginica}: {\bf 0.952}};

        \draw [->,thick,black!80] (feat1) to[] node[above, pos=1.2, font=\scriptsize] {yes} (feat2.north);
        \draw [->,thick,black!80] (feat1) to[] node[above, pos=1.1, font=\scriptsize] { no} (feat3.north);
        \draw [->,thick,black!80] (feat2) to[] node[above, pos=1.2, font=\scriptsize] {yes} (pos1.north);
        \draw [->,thick,black!80] (feat2) to[] node[above, pos=1.1, font=\scriptsize] { no} (neg1.north);
        \draw [->,thick,black!80] (feat3) to[] node[above, pos=1.2, font=\scriptsize] {yes} (pos2.north);
        \draw [->,thick,black!80] (feat3) to[] node[above, pos=1.1, font=\scriptsize] { no} (neg2.north);
    \end{tikzpicture}
    } %
     \\
    \begin{tikzpicture}[node distance = 4.0em, auto]
        \node [box] (box) {%
        \begin{minipage}[t!]{0.353\textwidth}
            \vspace{1.4cm}\hspace{1.5cm}
        \end{minipage}
        };
        \node[title] at (box.north) {$\text{T}_\text{2}$ };

        \node [feature] (feat1) at (-0.73, 0.9) {petal.width $<=$ 0.75?};
        \node [feature, below right = 0.8em and -2.2em of feat1] (feat2) {petal.length $<=$ 4.75?};

        \node [scores, below left  = 0.8em and -1.2em of feat1, align=left] (pos1) { {\bf Setosa}: {\bf 1.0} \\  Versicolor: 0.0  \\   Virginica: 0.0 };
        \node [scores, below left = 0.8em and -2.4em of feat2, align=left] (pos2) { Setosa: 0.0 \\   {\bf Versicolor}: {\bf 1.0}  \\  Virginica: 0.0 };
        \node [scores, below right = 0.8em and -2.4em of feat2, align=left] (neg) { Setosa: 0.0 \\   Versicolor: 0.166  \\  {\bf Virginica}: {\bf 0.833} };

        \draw [->,thick,black!80] (feat1) to[] node[above, pos=1.2, font=\scriptsize] {yes} (pos1.north);
        \draw [->,thick,black!80] (feat1) to[] node[above, pos=1.1, font=\scriptsize] { no} (feat2.north);
        \draw [->,thick,black!80] (feat2) to[] node[above, pos=1.2, font=\scriptsize] {yes} (pos2.north);
        \draw [->,thick,black!80] (feat2) to[] node[above, pos=1.1, font=\scriptsize] { no} (neg.north);
    \end{tikzpicture}
    &
    \begin{tikzpicture}[node distance = 4.0em, auto]
        \node [box] (box) {%
        \begin{minipage}[t!]{0.353\textwidth}
            \vspace{1.4cm}\hspace{1.5cm}
        \end{minipage}
        };
        \node[title] at (box.north) {$\text{T}_\text{3}$};

        \node [feature] (feat1) at (-0.53, 0.9) {petal.width $<=$ 0.75?};
        \node [feature, below right = 0.8em and -2.2em of feat1] (feat2) {petal.width $<=$ 1.65?};

        \node [scores, below left  = 0.8em and -0.5em of feat1, align=left] (pos1) { {\bf Setosa}: {\bf 1.0} \\ Versicolor: 0.0 \\ Virginica: 0.0};
        \node [scores, below left = 0.8em and -2.4em of feat2, align=left] (pos2) {Setosa: 0.0 \\  {\bf Versicolor}: {\bf 0.904} \\  Virginica: 0.095};
        \node [scores, below right = 0.8em and -2.4em of feat2, align=left] (neg) {Setosa: 0.0 \\  Versicolor: 0.05 \\ {\bf Virginica}: {\bf 0.95}};

        \draw [->,thick,black!80] (feat1) to[] node[above, pos=1.2, font=\scriptsize] {yes} (pos1.north);
        \draw [->,thick,black!80] (feat1) to[] node[above, pos=1.1, font=\scriptsize] { no} (feat2.north);
        \draw [->,thick,black!80] (feat2) to[] node[above, pos=1.2, font=\scriptsize] {yes} (pos2.north);
        \draw [->,thick,black!80] (feat2) to[] node[above, pos=1.1, font=\scriptsize] { no} (neg.north);
    \end{tikzpicture}
\end{tabular}
\end{adjustbox}

}
    \caption{
        Example of a simple RFwv trained using Scikit-learn with 3 trees, on {\it Iris} dataset. 
        Here, the tree ensemble has 3 trees  with the depth of each tree being at most 2.  
        In the running examples, classes ``Setosa'', ``Versicolor'' and  ``Virginica'' are 
        denoted, resp.\  class $c_1$, $c_2$ and $c_3$, and features {\it sepal.length},  
        {\it sepal.width}, {\it petal.length} and  {\it petal.width} are ordered, resp., as 
        features 1--4.    
    }
    \label{fig:RF(wv)}
\end{figure}
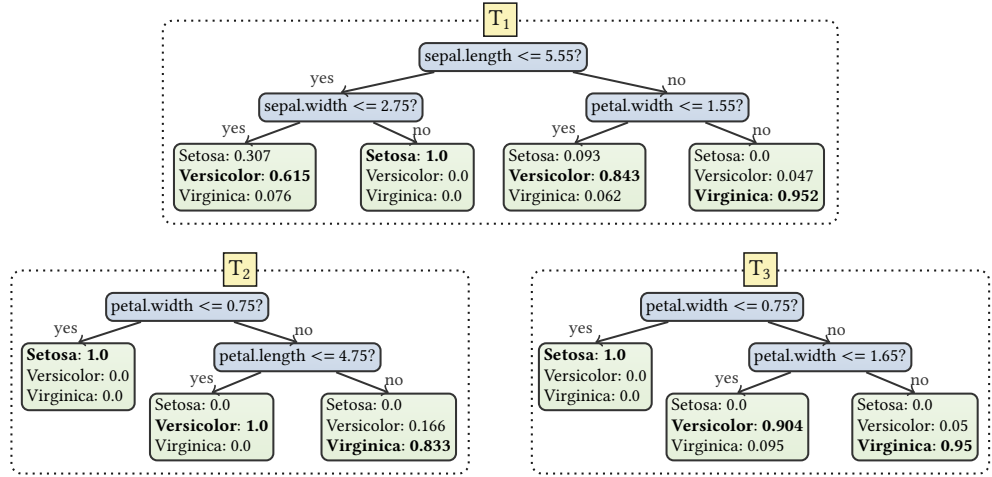

\begin{example} \label{ex:RF(wv)}
  \cref{fig:RF(wv)} shows an example of a RF with weighted voting, which 
  has 3 trees.
  %
  %
  Consider the sample $\mbf{v} = (5.1, 2.75, 1.4, 0.8)$, then we have 
  $W(c_1,\mbf{v}) = 0.307 + 0.0 + 0.0 =  0.307$, 
  $W(c_2,\mbf{v}) = 0.615 + 1.0 + 0.904 = 2.519$ and 
  $W(c_3,\mbf{v}) = 0.076 + 0.0 + 0.095 = 0.171$. 
  As a results, the predicted class is $c_2$ (``Versicolor'').
\end{example}

To represent an RF with weighted voting, in our general TE model, we
proceed as follows. The weight of each tree path in the TE model is
the one in the corresponding tree path in the RF. Moreover, the class
output by the path in the TE is the class predicted by the
corresponding tree path in the RF.

Formally, let an RFwv $\mathfrak{F}$ be given by a set of trees
$\mathfrak{F} = \{\, T_t \}_{t=1}^{\,n}$,
where for each tree $T_t$ and each leaf $l \in \mathcal{L}(T_t)$,
every class $c_i \in \mathcal{K}$ is associated with a weight via a mapping
\[
w_{\mathfrak{F}} : \mathcal{L}(T_t) \times \mathcal{K} \;\rightarrow\; \mathbb{R} 
\]

We define an equivalent TE $\mathfrak{T}$ as follows.
For each $T_t \in \mathfrak{F}$ and each class $c_i \in \mathcal{K}$,
we introduce a corresponding tree $T'_{K(t-1)+i} \in \mathfrak{T}$
with identical paths.
For every leaf $l \in \mathcal{L}(T'_{K(t-1)+i})$ we assign:
\[
cl(l) \;:=\; c_i
\qquad\text{and}\qquad
w_l \;:=\; w_{\mathfrak{F}}(l, c_i)
\]


\begin{remark}
  Given the results in this section, it is simple to reduce RFmv to
  RFs with weighted voting (RFwv), but also to BTs.
  The reduction from RFmv to RFwv is straightforward, in that the
  weights are just set to 1. 
  The reduction from RFmv to BT also seems simple. A tree $T$ in the
  RF is split into $L$ trees, one for each class predicted by $T$.
  For copy $l$, added to the group representing class $c_l$, the paths
  predicting $c_l$ are given weight 1; all other paths are given
  weight 0.
\end{remark}


Given the proposed reductions, we have the following result, in
addition to~\cref{thm:dpc}.


\begin{proposition}
  For each random forest with majority voting there exists an
  equivalent boosted tree with a size that is polynomial on the size
  of the starting random forest.
\end{proposition}


%
\begin{figure}[ht]
  \centering
  \resizebox{\textwidth}{!}{\input{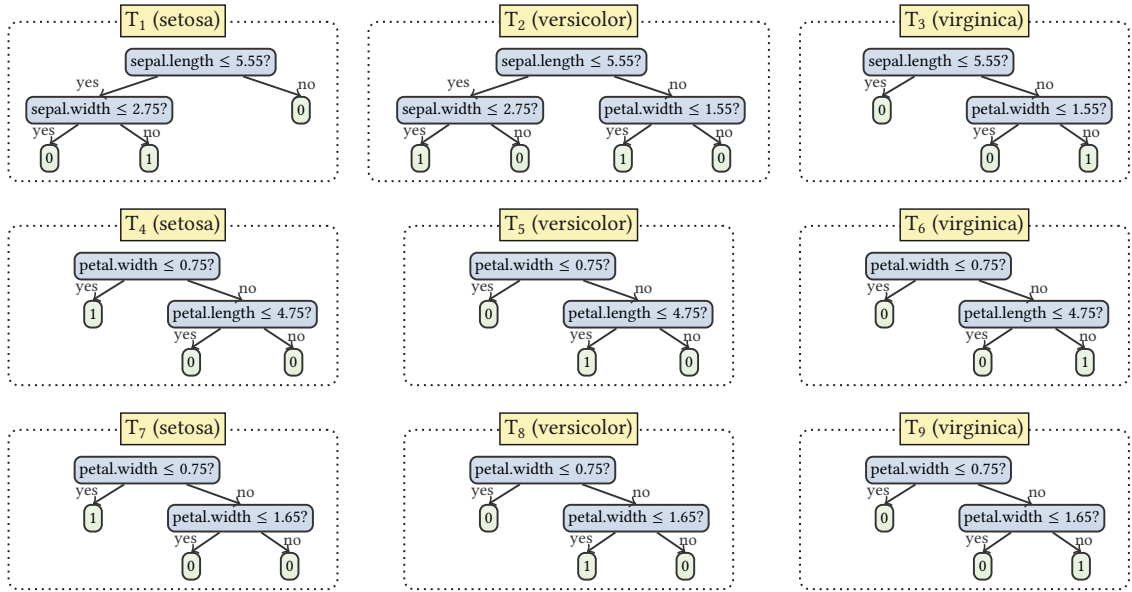}}
  \caption{Reducing example RFmv to BT. Note 
   that splitting nodes yielding to a score $0$, in its leaf nodes, are pruned 
  and replaced with labeled leaf node $0$.} 
  \label{fig:rfmv2bt}
\end{figure}

\begin{example}
  \cref{fig:rfmv2bt} presents the reduction of the RFmv
  from~\cref{fig:RF(mv)} to a BT. As results, we obtain $3\times 3$ trees, 
  such that each tree of the RFmv is duplicated $K=3$ times 
  and each duplicate predicts the class scores $0/1$ of the corresponding class 
  accordingly, i.e. ``Setosa'', ``Versicolor'' and  ``Virginica''.    
  Observe that splitting nodes yielding to a score $0$, in its leaf nodes, are pruned 
  and replaced with labeled leaf node $0$, .e.g. for T$_1$ the internal node 
  $(\text{\em petal.width} \le 1.55)$ has both children nodes predicting score $0$ 
  for class ``Setosa'', hence it is pruned and replaced by a leaf node $0$. 
\end{example}

\begin{proposition}
  For each random forest with majority voting there exists an
  equivalent random forest with weighted voting with a size that is
  linear on the size of the starting random forest.
\end{proposition}

\begin{example}
  \cref{fig:rfmv2rfwv} presents the reduction of the RFmv
  from~\cref{fig:RF(mv)} to a RFwv.
\end{example}

\begin{figure}[ht]
  \resizebox{\textwidth}{!}{
%
%
%

\tikzstyle{box} = [draw=black!90, thick, rectangle, rounded corners,
                     inner sep=10pt, inner ysep=20pt, dotted
                  ]
\tikzstyle{title} = [draw=black!90, fill=black!5, semithick, top color=white,
                     bottom color = black!5, text=black!90, rectangle,
                     font=\small, inner sep=2pt, minimum height=1.3em,
                     top color=tyellow2!27, bottom color=tyellow2!27
                    ]
\tikzstyle{feature} = [rectangle,font=\scriptsize,rounded corners=1mm,thick,%
                       draw=black!80, top color=tblue2!20,bottom color=tblue2!25,%
                       draw, minimum height=1.1em, text centered,%
                       inner sep=2pt%
                      ]
\tikzstyle{pscore} = [rectangle,font=\scriptsize,rounded corners=1mm,thick,%
                     draw=black!80, top color=tgreen3!20,bottom color=tgreen3!27,%
                     draw, minimum height=1.1em, text centered,%
                     inner sep=2pt%
                    ]
\tikzstyle{scores} = [rectangle,font=\scriptsize,rounded corners=1mm,thick,%
                     draw=black!80, top color=tgreen3!10,bottom color=tgreen3!15,%
                     draw, minimum height=1.1em, text centered,%
                     inner sep=2pt%
                    ]                    
\tikzstyle{nscore} = [rectangle,font=\scriptsize,rounded corners=1mm,thick,%
                     draw=black!80, top color=tred2!20,bottom color=tred2!25,%
                     draw, minimum height=1.1em, text centered,%
                     inner sep=2pt%
                    ]

\begin{adjustbox}{center}
\setlength{\tabcolsep}{16pt}
\def\arraystretch{3}
\begin{tabular}{cc}
   \multicolumn{2}{c}{
    \begin{tikzpicture}[node distance = 4.0em, auto]
        \node [box] (box) {%
        \begin{minipage}[t!]{0.485\textwidth}
            \vspace{1.4cm}\hspace{1.5cm}
        \end{minipage}
        };
        \node[title] at (box.north) {$\text{T}_\text{1} $};

        \node [feature] (feat1) at (0.04, 0.9) {sepal.length $<=$ 5.55?};
        \node [feature, below left  = 0.8em and -0.2em of feat1] (feat2) {sepal.width $<=$ 2.75?};
        \node [feature, below right = 0.8em and -0.1em of feat1] (feat3) {petal.width $<=$ 1.55?};

        \node [scores, below left  = 0.8em and -2.3em of feat2, align=left] (pos1)  { Setosa: 0 \\  {\bf Versicolor}: {\bf 1} \\ Virginica: 0};
        \node [scores, below right  = 0.8em and -2.5em of feat2, align=left] (neg1) { {\bf Setosa}: {\bf 1} \\  Versicolor: 0 \\  Virginica: 0};
        \node [scores, below left = 0.8em and -2.4em of feat3, align=left] (pos2) { Setosa: 0 \\  {\bf Versicolor}: {\bf 1} \\  Virginica: 0};
        \node [scores, below right = 0.8em and -2.4em of feat3, align=left] (neg2) { Setosa: 0 \\  Versicolor: 0 \\   {\bf Virginica}: {\bf 1}};

        \draw [->,thick,black!80] (feat1) to[] node[above, pos=1.2, font=\scriptsize] {yes} (feat2.north);
        \draw [->,thick,black!80] (feat1) to[] node[above, pos=1.1, font=\scriptsize] { no} (feat3.north);
        \draw [->,thick,black!80] (feat2) to[] node[above, pos=1.2, font=\scriptsize] {yes} (pos1.north);
        \draw [->,thick,black!80] (feat2) to[] node[above, pos=1.1, font=\scriptsize] { no} (neg1.north);
        \draw [->,thick,black!80] (feat3) to[] node[above, pos=1.2, font=\scriptsize] {yes} (pos2.north);
        \draw [->,thick,black!80] (feat3) to[] node[above, pos=1.1, font=\scriptsize] { no} (neg2.north);
    \end{tikzpicture}
    } %
     \\  
    \begin{tikzpicture}[node distance = 4.0em, auto]
        \node [box] (box) {%
        \begin{minipage}[t!]{0.33\textwidth}
            \vspace{1.4cm}\hspace{1.5cm}
        \end{minipage}
        };
        \node[title] at (box.north) {$\text{T}_\text{2}$ };

        \node [feature] (feat1) at (-0.375, 0.9) {petal.width $<=$ 0.75?};
        \node [feature, below right = 0.8em and -2.2em of feat1] (feat2) {petal.length $<=$ 4.75?};

        \node [scores, below left  = 0.8em and -1.2em of feat1, align=left] (pos1) { {\bf Setosa}: {\bf 1} \\  Versicolor: 0  \\   Virginica: 0 };
        \node [scores, below left = 0.8em and -2.4em of feat2, align=left] (pos2) { Setosa: 0 \\   {\bf Versicolor}: {\bf 1}  \\  Virginica: 0 };
        \node [scores, below right = 0.8em and -2.4em of feat2, align=left] (neg) { Setosa: 0 \\   Versicolor: 0  \\  {\bf Virginica}: {\bf 1} };

        \draw [->,thick,black!80] (feat1) to[] node[above, pos=1.2, font=\scriptsize] {yes} (pos1.north);
        \draw [->,thick,black!80] (feat1) to[] node[above, pos=1.1, font=\scriptsize] { no} (feat2.north);
        \draw [->,thick,black!80] (feat2) to[] node[above, pos=1.2, font=\scriptsize] {yes} (pos2.north);
        \draw [->,thick,black!80] (feat2) to[] node[above, pos=1.1, font=\scriptsize] { no} (neg.north);
    \end{tikzpicture}
    &
    \begin{tikzpicture}[node distance = 4.0em, auto]
        \node [box] (box) {%
        \begin{minipage}[t!]{0.33\textwidth}
            \vspace{1.4cm}\hspace{1.5cm}
        \end{minipage}
        };
        \node[title] at (box.north) {$\text{T}_\text{3}$};

        \node [feature] (feat1) at (-0.375, 0.9) {petal.width $<=$ 0.75?};
        \node [feature, below right = 0.8em and -2.2em of feat1] (feat2) {petal.width $<=$ 1.65?};

        \node [scores, below left  = 0.8em and -0.5em of feat1, align=left] (pos1) { {\bf Setosa}: {\bf 1} \\ Versicolor: 0 \\ Virginica: 0};
        \node [scores, below left = 0.8em and -2.4em of feat2, align=left] (pos2) {Setosa: 0 \\  {\bf Versicolor}: {\bf 1} \\  Virginica: 0};
        \node [scores, below right = 0.8em and -2.4em of feat2, align=left] (neg) {Setosa: 0 \\  Versicolor: 0 \\ {\bf Virginica}: {\bf 1}};

        \draw [->,thick,black!80] (feat1) to[] node[above, pos=1.2, font=\scriptsize] {yes} (pos1.north);
        \draw [->,thick,black!80] (feat1) to[] node[above, pos=1.1, font=\scriptsize] { no} (feat2.north);
        \draw [->,thick,black!80] (feat2) to[] node[above, pos=1.2, font=\scriptsize] {yes} (pos2.north);
        \draw [->,thick,black!80] (feat2) to[] node[above, pos=1.1, font=\scriptsize] { no} (neg.north);
    \end{tikzpicture}    
\end{tabular}
\end{adjustbox}}
  \caption{Reducing example RFmv to RFwv.} 
  \label{fig:rfmv2rfwv}
\end{figure}
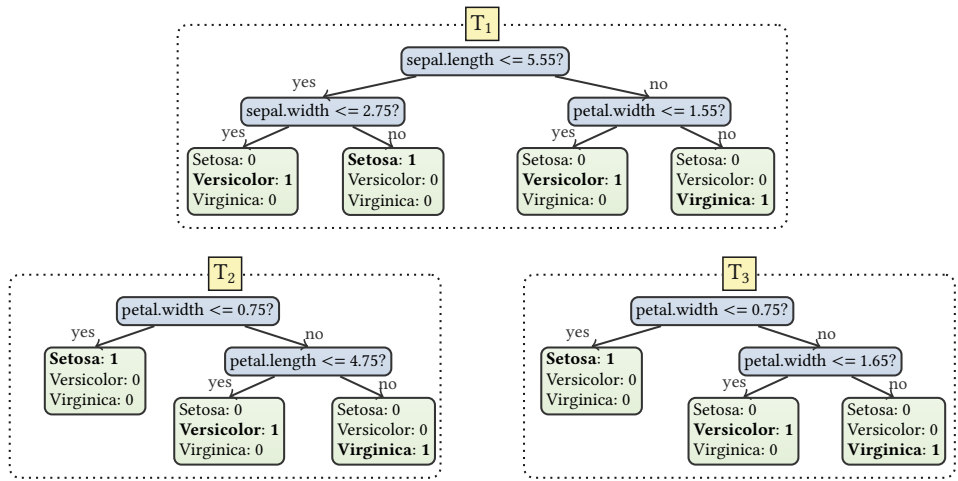


\section{Computing Explanations}

\subsection{Complexity of Explaining TEs}

\subsubsection{Complexity of AXps for RFs} \label{sec:cplxRF}
Recent work identified classes of classifiers for which one
AXp can be computed in polynomial
time~\cite{marquis-kr20,msgcin-nips20}. These classes of classifiers
include those respecting specific criteria of the knowledge
compilation map~\cite{marquis-kr20},\footnote{%
  The knowledge compilation map was first proposed in
  2002~\cite{darwiche-jair02}.}
but also Naive Bayes classifiers (NBCs) and linear classifiers (LCs)~\cite{msgcin-nips20}. 
(In the case of NBCs and LCs, enumeration of AXps was shown
to be solved with polynomial delay.)
One question is thus whether there might exist a polynomial time
algorithm for computing one AXp of an RF.
This section shows that this is unlikely to be the case, by proving
that deciding whether a set of features represents an AXp is
$\ddp$-complete.\footnote{The class $\ddp$ \cite{Papadimitriou94} is
  the set of languages defined by the intersection of two languages,
  one in NP and one in coNP.} 

Let $\mfrk{T}$ be an RFmv, with classification function $\kappa$, and
let $\mbf{v}\in\mbb{F}$, with prediction $\kappa(\mbf{v})=c\in\fml{K}$.
$\kappa$ is parameterized with $c$, to obtain the boolean function
$\kappa_c$, s.t.\ $\kappa_c(\mbf{x})=1$ iff $\kappa(\mbf{x})=c$.
A set of literals $I_{\mbf{v}}$ is associated with each $\mbf{v}$.
Let $\rho$ be a subset of the literals associated with $\mbf{v}$,
i.e.\ $\rho\subseteq{I_{\mbf{v}}}$. Hence,

\begin{theorem} \label{thm:dpc}
  For a random forest with majority vote $\mfrk{T}$, given an instance $\mbf{v}$ with
  prediction $c$, deciding whether a set of literals is an
  AXp is $\ddp$-complete.
\end{theorem}

\begin{proof}
  Given an instance $\mbf{v}$ and predicted class $c$, deciding
  whether a set of literals $\rho$ is an AXp of an RFmv
  $\mfrk{T}$ is clearly in $\ddp$. We need to prove that
  $\rho\entails\kappa_c$, which is a problem in coNP. We also need to
  prove that a set of literals $\rho'$, obtained by the removal of any
  single literal from $\rho$ (and there can be at most $m$ of these),
  is such that $\rho\nentails\kappa_c$, a problem in NP.
  To prove that the problem is hard for $\ddp$,
  we reduce the problem of computing a prime implicant of a DNF, which
  is known to be complete for $\ddp$
  , to the
  problem of computing an abductive explanation of an RFmv $\mfrk{T}$.\\
  Consider a DNF $\phi=t_1\lor\cdots\lor{t_n}$, where each term $t_i$
  is a conjunction of literals defined on a set $X=\{x_1,\ldots,x_m\}$
  of boolean variables. 
  Given $\phi$, we construct an RFmv $\mfrk{T}$, defined on a set
  $\fml{F}$ of $m$ features, where each feature $i$ is associated with
  an $x_i$ element of $X$, and where $D_i=\{0,1\}$. Moreover,
  $\mfrk{T}$ is such that  $\phi(\mbf{x})=1$ iff
  $\kappa_1(\mbf{x})=1$.
  $\mfrk{T}$ is constructed as follows.
  \begin{enumerate}[nosep,label=\roman*.]
  \item
    Associate a decision tree (DT) $T_i$ with each term
    $t_i$, such that the assignment satisfying $t_i$ yields class 1,
    and the other assignments yield class 0. Clearly, the size of the
    DT $T_i$ is linear on the size of $t_i$, since each literal
    not taking the value specified by the term will be connected to a
    terminal node with prediction 0.
    \item
    Create $(n-1)$ additional trees, each having exactly one
    terminal node and no non-terminal nodes. Moreover, associate class
    $1$ with the terminal node.
    \end{enumerate}
  Next, we prove that $\phi(\mbf{x})=1$ iff $\kappa_1(\mbf{x})=1$.
  \begin{enumerate}[nosep,label=\roman*.]
  \item
    Let $\mbf{x}$ be such that $\phi(\mbf{x})=1$. Then, there is
    at least one term $t_j$, such that $t_j(\mbf{x})=1$. As a
    result, the corresponding tree $T_j$ in the RFmv will
    predict class 1.  Hence, at least $n$ trees predict class 1, and
    at most $n-1$ trees predict class 0. As a result, the predicted
    class is 1, and so $\kappa_1(\mbf{x})=1$.
    \item
    Let $\mbf{x}$ be such that $\kappa_1(\mbf{x})=1$. This means
    that at least one of the trees associated with the terms $t_i$
    must predict value 1. Let such tree be $T_j$, associated
    with term $t_j$. For this tree to predict class 1, then
    $t_j(\mbf{x})=1$, and so $\phi(\mbf{x})=1$.
    \end{enumerate}
  Now, let $\rho$ be a conjunction of literals defined on $X$. Then,
  we must have $\rho\entails\phi$ iff $\rho\entails\kappa_1$. Every model
  of $\rho$ is also a model of $\phi$, and so it must also be a model
  of $\kappa_1$. Conversely, every model of $\rho$ is also a model of
  $\kappa_1$, and so it must also be a model of $\phi$.
\end{proof}




\begin{proposition}
  Given~\cref{thm:dpc}, deciding whether a set of features is an AXp
  is $\ddp$-complete, both for BTs and for RFs with weighted voting.
\end{proposition}

\subsubsection{Complexity of CXps for RFs}
\label{sec:cplxRF:cxp}
Similarly to AXp for RFmv, deciding whether a set of features represents
a CXp is likewise $\ddp$-complete.
%
%
Moreover, deciding whether 
a set of literals $\rho$ is a CXp responds to deciding if $\neg\rho$ 
is a \emph{prime implicate}~\cite{Audemard-ecai23} of a boolean function $\kappa_c$, i.e. 
$\kappa_c\not\entails\rho$.

\begin{proposition} \label{prop:cxp:dpc}
  For a random forest with majority vote $\mfrk{T}$, given an 
  instance $\mbf{v}$ with prediction $c$, deciding whether a set 
  of features is a CXp is $\ddp$-complete.
\end{proposition}

\begin{proof}
  Membership in $\ddp$ follows directly from the definition 
  of CXp~\eqref{eq:cxp}: it requires the existence 
  of an assignment witnessing a prediction change, a problem in 
  NP; minimality requires that no proper subset has this property, 
  a problem in coNP. Hardness follows from~\cref{thm:dpc}.
\end{proof}

\subsection{Computing AXps}
A minimal set of features $\fml{X}\in\fml{F}$ is an AXp
if~\eqref{eq:axp} holds. Clearly, this condition holds iff the
following formula is unsatisfiable,
\[
    \left[\bigwedge\nolimits_{i\in{\fml{X}}}\bb{x_i=v_i}\right]
    \land  
    \msf{Enc}(\kappa(\mbf{x})\not=c)
\]
The previous formula has two components
$\langle\fml{H}\cup\fml{S}_{c,c'},\fml{R}\rangle$. 
\begin{enumerate}[nosep,label=\roman*.]
\item
    $\fml{H}$ represents the set of (hard) clauses encoding the
    representation of the ML model and $\fml{S}_{c,c'}$ imposing an adversarial class $c' \not=c$ 
    to be picked, i.e.\ $\msf{Enc}(\kappa(\mbf{x})\not=c)$.
\item
    $\fml{R}$ consists of unit (relaxable) clauses capturing 
    literals $\bb{x_i=v_i}$.  As such, $\fml{R}$ enforces the class assignment
    $\kappa(\mbf{v}) = c$, which makes the formula $\fml{H} \land \fml{S}_{c,c'} \land \fml{R}$
    inconsistent. 
    We are interested in searching for a minimal subset  $\fml{R}\prime$ of
    $\fml{R}$, such that,
    \[
        \left[\bigwedge\nolimits_{\bb{x_i=v_i}\in\fml{R}\prime}\bb{x_i=v_i}\right]
        \land \left[ \fml{H} \land \fml{S}_{c,c'}\right]
    \]
    is unsatisfiable. Intuitively, this corresponds to find a minimal information 
    that needs to be added to the background knowledge $\fml{H} \cup \fml{S}_{c,c'}$ to have   
    inconsistency. 
    As a results, this corresponds to finding a minimal unsatisfiable subset (MUS) of
    $\langle\fml{H} \cup \fml{S}_{c,c'},\fml{R}\rangle$.
\end{enumerate}
Hence,
 any off-the-shelf MUS extraction algorithm can be used
for computing an AXp (as noted in earlier
work~\cite{inms-aaai19}).

\paragraph{Linear Search Approach for AXp.}
The overall explanation procedure is depicted in
Algorithm~\ref{alg:del} and follows the standard deletion-based AXp
extraction~\cite{inms-aaai19}.
It receives a TE model $\mfrk{T}$ computing function $\kappa(\mbf{x})$,
a concrete data instance $\mbf{v}$ and the corresponding prediction
$c=\kappa(\mbf{v})$.
First, Algorithm~\ref{alg:del} encodes the classifier (see
line~\ref{ln:enc}), as described in Section~\ref{sec:mxenc}.
Then it iterates through the features $i\in\fml{X}$ (initially set
to $\fml{F}$) and checks if $\fml{X}\setminus\{i\}$ entails the
prediction $c$ by invoking Algorithm~\ref{alg:call}.
If it does, feature $i$ is irrelevant and is removed.
Otherwise, it must be kept in $\fml{X}$.
Algorithm~\ref{alg:del} proceeds until all features are tested, and
reports the updated~$\fml{X}$.

An immediate observation to make here is that instead of a simple
deletion-based algorithm for AXp extraction, one could apply the ideas
behind the QuickXplain and Progression
algorithms~\cite{junker-aaai04,msjb-cav13}, which in some cases may
help reducing the number of oracle calls.

Also, it should be noted that besides computing a single AXp, the approach 
described above, can be set up to compute a
single contrastive explanation (CXp) but also to enumerate either a
given number of AXps and/or CXps or exhaustively enumerate all of
them, following the minimal hitting set duality between AXps and
CXps~\cite{inams-aiia20}.

\begin{algorithm}[ht]
%
\SetInd{0.4em}{0.4em}
\SetStartEndCondition{ }{}{}%
\SetKwProg{Fn}{def}{\string:}{}
\SetKwFunction{Range}{range}
\SetKw{KwTo}{in}\SetKwFor{For}{for}{\string:}{}%
\SetKwIF{If}{ElseIf}{Else}{if}{:}{elif}{else:}{}%
\SetKwFor{While}{while}{:}{}%
\SetKwFor{ForEach}{foreach}{:}{}%
\AlgoDontDisplayBlockMarkers\SetAlgoNoEnd\SetAlgoNoLine%
\LinesNumbered
\DontPrintSemicolon%
\SetKw{KwNot}{not\xspace}
\SetKw{KwAnd}{and\xspace}
\SetKw{KwOr}{or\xspace}
\SetKw{KwBreak}{break\xspace}
\SetKwData{false}{{\small false}}
\SetKwData{true}{{\small true}}
\SetKwData{st}{{\small \slshape st}}
\SetKwFunction{CNF}{CNF}
\SetKwFunction{SAT}{SAT}
\SetKwInOut{Input}{input}\SetKwInOut{Output}{output}
\newcommand\commfont[1]{\footnotesize\em{#1}}
\SetKwComment{tcpy}{\# }{}
\SetCommentSty{commfont}
\SetKwFunction{tfunc}{{\sc ExtractAXp}} %
\SetKwFunction{efunc}{{\sc EntCheck}} %
\SetKwFunction{cfunc}{{\sc Encode}} %
\Func \tfunc{$\fml{E})$} \;
\Indp
\KwIn{
  $\fml{E} = (\mfrk{T},(\mbf{v},c))$ s.t. 
  $\mfrk{T}$: TE computing $\kappa(\mbf{x})$,
  \hspace{0.4em}$\mbf{v}$: Input instance, 
  \hspace{0.4em}$c$: Prediction, i.e.\ $\kappa(\mbf{v})=c$
}
\KwOut{
  $\fml{X}$: abductive explanation
}
\setcounter{AlgoLine}{0}
\BlankLine
{
  $\langle\fml{H},\fml{S}\rangle\gets\cfunc(\fml{E})$\tcpy*[r]{SAT/MaxSAT/SMT enc. s.t.\ $\fml{S}\triangleq\cup\fml{S}_{c,c'}$}\label{ln:enc}
  $\fml{X} \gets \fml{F}$ \tcpy*[r]{$\fml{X}$ is over-approximation}
  \ForEach(\label{ln:iter}){$i\in\fml{X}$}{
    \uIf(\tcpy*[f]{need $i$?}){$\efunc(\langle\fml{H},\fml{S}\rangle,\fml{E},\fml{X}\setminus\{i\})$}{ 
        $\fml{X} \gets \fml{X} \setminus \{i\}$ \tcpy*[r]{if not, drop it} 
      }
    }
  \Return $\fml{X}$ \tcpy*[r]{$\fml{X}$ is AXp}
}
\Indm
%

  \caption{Deletion-based AXp extraction} \label{alg:del}
\end{algorithm}

\subsection{Computing CXps}
Clearly, adapting the described procedure above of computing one AXp to 
one that computes a CXp is straightforward.  That is, the minimal set $\fml{Y}$
of $\fml{F}$ to search is, such that,
\[
\left[\bigwedge\nolimits_{i \in\fml{F}\setminus\fml{Y}}\bb{x_i=v_i}\right]
\land
\msf{Enc}(\kappa(\mbf{x})\not=c)
\]
is satisfiable. Further, hitting set duality between  AXps and CXps allows  
to exploit any algorithm for computing MUSes/MCSes
  to
enumerate both kinds of explanations (AXps and CXps).  (Recent
work~\cite{inams-aiia20,msgcin-icml21,IgnatievS21} exploits the MUS/MCS
enumeration algorithm MARCO~\cite{lpmms-cj16} for enumerating
AXps/CXps.)

\paragraph{Linear Search Approach for CXp.}
\Cref{alg:cxp-del} depicts the deletion-based (or linear Search) approach 
to find one CXp.
In the case of a CXp, the features in $\fml{Y}$, if fixed, cause
that no counter-example for target class $c$ can be identified.
Features are freed as long as the loop invariant is preserved.
An oracle for checking the entailment  that the prediction 
$\kappa(\mbf{x})\not=c$, is used to decide 
whether or not the loop invariant is preserved when another
feature is freed.
It can be shown that the final set $\fml{Y}$ is a CXp.
The proposed algorithms are among the simplest that can be
devised.%
\footnote{%
The problem of finding one AXp/CXp represents an example of
computing one minimal set subject to a monotone predicate
(MSMP)~\cite{msjm-aij17}, and so existing algorithms for MSMP can be
used.}
Nevertheless, other more sophisticated algorithms 
(see e.g.~\cite{mshjpb-ijcai13,mpms-ijcai15}) for deciding the
existence of adversarial examples can be used, which require in
practice fewer calls to the oracle than the linear search algorithm 
outlined in the paper.

\begin{algorithm}[ht]
%
\SetInd{0.4em}{0.4em}
\SetStartEndCondition{ }{}{}%
\SetKwProg{Fn}{def}{\string:}{}
\SetKwFunction{Range}{range}
\SetKw{KwTo}{in}\SetKwFor{For}{for}{\string:}{}%
\SetKwIF{If}{ElseIf}{Else}{if}{:}{elif}{else:}{}%
\SetKwFor{While}{while}{:}{}%
\SetKwFor{ForEach}{foreach}{:}{}%
\AlgoDontDisplayBlockMarkers\SetAlgoNoEnd\SetAlgoNoLine%
\LinesNumbered
\DontPrintSemicolon%
\SetKw{KwNot}{not\xspace}
\SetKw{KwAnd}{and\xspace}
\SetKw{KwOr}{or\xspace}
\SetKw{KwBreak}{break\xspace}
\SetKwData{false}{{\small false}}
\SetKwData{true}{{\small true}}
\SetKwData{st}{{\small \slshape st}}
\SetKwFunction{CNF}{CNF}
\SetKwFunction{SAT}{SAT}
\SetKwInOut{Input}{input}\SetKwInOut{Output}{output}
\newcommand\commfont[1]{\footnotesize\em{#1}}
\SetKwComment{tcpy}{\# }{}
\SetCommentSty{commfont}
\SetKwFunction{tfunc}{{\sc ExtractAXp}} %
\SetKwFunction{efunc}{{\sc EntCheck}} %
\SetKwFunction{cfunc}{{\sc Encode}} %
\Func \tfunc{$\fml{E}$} \;

\Indp
\KwIn{
$\fml{E} = (\mfrk{T},(\mbf{v},c))$ s.t. 
  $\mfrk{T}$: TE computing $\kappa(\mbf{x})$,
  \hspace{0.4em}$\mbf{v}$: Input instance, 
  \hspace{0.4em}$c$: Prediction, i.e.\ $\kappa(\mbf{v})=c$
}
\KwOut{
  $\fml{Y}$: contrastive explanation
}
\setcounter{AlgoLine}{0}
\BlankLine
{
  $\langle\fml{H},\fml{S}\rangle\gets\cfunc(\fml{E})$\tcpy*[r]{SAT/MaxSAT/SMT enc. s.t.\ $\fml{S}\triangleq\cup\fml{S}_{c,c'}$}\label{ln:enc}
  $\fml{Y} \gets \fml{F}$ \tcpy*[r]{Initially, no feature in reference set is fixed}
  \ForEach(\label{ln:iter}){$i\in\fml{Y}$}{
    \uIf(\tcpy*[f]{need $i$?}){$\KwNot~\efunc(\langle\fml{H},\fml{S}\rangle, \fml{E}, (\fml{F} \setminus\fml{Y}) \cup \{i\}))$}{
        $\fml{Y} \gets \fml{Y} \setminus \{i\}$ \tcpy*[r]{if yes, drop it} 
      }
    }
  \Return $\fml{Y}$ \tcpy*[r]{$\fml{Y}$ is CXp}
}
\Indm
%

  \caption{Deletion-based CXp extraction} \label{alg:cxp-del}
\end{algorithm}

\section{Explanation Encodings for Tree Ensembles}

\subsection{SMT-based Explanations for TEs} \label{sec:smtbt}
This section\footnote {This work has been previously published in~\cite{inms-corr19,ignatiev-ijcai20}.} proposes an SMT encoding of an ensemble of decision trees
produced by XGBoost algorithm. 

\paragraph{SMT Encoding for BTs.}
At a high level, our encoding simulates the scores computation for a possible input.  We introduce three sets of variables. For a binary 
feature $i$ we introduce a Boolean variable $\bb{x_i = 1}$, $i \in [m]$.  
These Boolean variables represent the space of all possible inputs.
In case of categorical data, a common approach is to apply \emph{one hot encoding} 
to convert discrete values to binary values. This transformation is performed 
on the original data. Hence, our encoding can be applied directly 
with a small augmentation. We enforce that exactly one of binary features 
that encode a categorical feature can be true.
The case of continuous features is handled similarly. The main difference is that 
logical predicates in a node are of the form $(x_i < d)$, where 
$d$ is a constant value. Here, for each predicate that occurs 
in $T_t, T_t \in \mfrk{T}$, we introduce 
$\bb{x_i < d}$ Boolean variable  such that
$\neg \bb{x_i < d}$ represents $(x_i \ge d)$.
For the $t$th tree $T_t$ we introduce a real valued 
variable $r_t$, $t \in [m\times K]$. The $r_t$ variable encodes
the score contribution from the $t$th tree. Finally, for the $j$th class we introduce 
one variable $v_j$, $j\in[K]$ that stores the score of the $j$th class.
%
Then, we encode the score computation for a tree $t_l \in T$ by
encoding all paths in a tree $t_t$ as follows. 
We recall that if a path $R_l$ follows the right (left, resp.) 
branch in a node $(x_i < d)$ then the predicate $\bb{x_i < d}$ has to hold 
(to be violated, resp.). Then, we enforce, 

$$
    \left( \bigwedge\nolimits_{(x_i < d) \in R_l} \bb{x_i < d} \wedge
    \bigwedge\nolimits_{\neg (x_i < d) \in R_l} \neg \bb{x_i < d} \right) \rightarrow r_t = w_l
$$

To compute the score of the $j$th class we add $K$ constraints,  $j \in [K]$:
$v_j = \sum_{t=1}^q r_{qj+t}$.

\begin{example}
Consider how the encoding works on the running example shown in~\cref{fig:BT}.
Consider the first tree.
For simplicity, we denote the index attribute of the root node as `petal.length'.
 We have two paths in the tree.
For the first path, we add a constraint
$\bb{\text{\em petal.length} < 2.45} \rightarrow r_1 = 0.42762$
and for the second path we add a constraint
$\neg \bb{\text{\em petal.length} < 2.45} \rightarrow r_1 = -0.21853$.
Again consider the fourth tree. We have two paths, which give 
$\bb{\text{\em petal.length} < 2.45} \rightarrow r_4 = 0.29522$ 
and 
$\bb{\text{\em petal.length} < 2.45} \rightarrow r_4 = -0.19674$
With two trees per class, we get $v_1= r_1 + r_4$. Other trees could be
encoded similarly.
\end{example}

\paragraph{SMT Encoding for RFs.}
%
%
%
Building on the SMT encoding for BTs, the formulation for RFs 
follows naturally. Each individual decision tree (DT) in the ensemble 
is represented as a set of logical implication rules, capturing the tree’s 
structure and decision paths.
Specifically, the antecedent of each rule is a conjunction of predicates 
corresponding to the non-terminal nodes along a path, while 
the consequent is a predicate representing the class label assigned 
at the terminal node.
%
%
%
%
Next, we aggregate the prediction (votes) of the DTs and count the prediction 
score for each class. This can be expressed by an arithmetic function that calculates 
the sum of trees predicting the same class. Lastly, a linear inequality 
checks which class has the largest score. 

The implementation of the encoding above resulted in performance
results comparable to those of BTs \cite{ignatiev-ijcai20}. However, in the
case of RFmv it is possible to devise a purely propositional encoding,
as detailed next.

\subsection{SAT-based Explanations for RFs} \label{sec:satrf}
Our goal is to represent the structure of an RFmv with a propositional
formula.\footnote {This work has been published in~\cite{ims-ijcai21}.} 
This requires abstracting away what will be shown to be
information used by the RFmv that is unnecessary for finding an
explanation.
We start by detailing how to encode the nodes of the decision trees in
an RF. This is done such that the actual values of the features are
abstracted away. Then, we present the encoding of the RFmv classifier
itself, namely how the majority voting is represented.

To encode a terminal node of a DT, we proceed as follows.
Given a set of classes  $\fml{K} = \{c_1, c_2, \ldots, c_K \}$, a
terminal node of $R_l$ labeled with one class of $\fml{K}$. Then, 
we define for each $c_j \in{\fml{K}}$ a variable $l_{tj}$ and represent 
the terminal node of a path $R_l$ with its corresponding 
label class $c_j$, i.e. $cl(l) = c_j$.

Moreover, the encoding of a non-terminal node of a DT is organized as
follows.
Given a feature $i \in{\fml{F}}$ associated with a non-terminal node of $R_l$, 
with a domain $\mbb{D}_i$, each outgoing edge of $l$ is represented by 
a literal  $\bb{x_i \bowtie S_i}$ s.t. $x_j \in{\fml{F}}$ is the variable 
of  feature $i$, $S_i \bowtie \mbb{D}_i$ and $\bowtie\:\in{\set{=, \subseteq, \in}}$.  
Hence we distinguish three cases for encoding $(x_i \bowtie S_i)$.
For the first case, feature $i$ is binary, and so the literal
$\bb{x_i \bowtie S_i}$ is True if  $x_i = 1$ and False if $x_i = 0$.
For the second case, feature $i$ is categorical, and so we introduce a
Boolean variable $z_j$  
for each value $v_j \in \mbb{D}_i$ s.t.\ $z_j = 1$ iff $x_i = v_j$. 
Assume $S_i =\set{v_1, \ldots, v_n}$, then we connect $\bb{x_i \in S_i}$ to variables $z_j, j = 1, \ldots, n$  
as follows: $\bb{x_i \in S_i} \lequiv (z_1 \vee \ldots \vee z_n)$ or  $\neg \bb{x_i \in S_i} \lequiv (z_1 \vee \ldots \vee z_n)$, 
depending on whether the current edge is going to left or right child-node. 
Finally, for the third case, feature $i$ is real-valued. 
Thus, $S_i$ is either an interval or a union of intervals.
Concretely, we  consider all the splitting thresholds 
of the feature $i$ existing in the model and we generate (in order) all the possible intervals. 
Each interval $I_j^i  \subseteq \mbb{D}_i$ is denoted 
by a Boolean variable $z_j$.  Let us assume $S_i = I_1^i \cup \ldots \cup I_n^i$ and 
$z_1, \ldots, z_n$ are variables associated with   $I_1^i \cup \ldots \cup I_n^i$, then 
we have $\bb{x_i \subseteq S_i} \lequiv (z_1 \vee \ldots \vee z_n)$ or 
$\neg \bb{x_i \subseteq S_i} \lequiv (z_1 \vee \ldots \vee z_n)$. 
Moreover, this encoding can be reduced into a more compact set of
constraints  (that also improves propagation in the SAT solver).
Indeed, if the number of intervals is large the encoding will be as
well. Hence, we propose to use auxiliary variables in the encoding.
Assume $\mbb{D}_i = I_1^i \cup I_2^i \cup I_3^i $, nodes $(x_i > s_{i,1})$ and 
$(x_i > s_{i,2})$  
such that  $\neg \bb{x_i > s_{i,1}} \lequiv z_1$, $ \bb{x_i > s_{i,1}} \lequiv ( z_2 \vee z_3)$, 
$\neg \bb{x_i > s_{i,2}} \lequiv (z_1 \vee z_2)$, $ \bb{x_i > s_{i,2}} \lequiv z_3$. 
Hence, 
this can be re-written as:  
 $\neg \bb{x_i > s_{i,1}} \lequiv z_1$,  $\bb{x_i > s_{i,1}} \lequiv (z_2 \vee \bb{x_i > s_{i,2}})$, 
 $\neg \bb{x_i > s_{i,2}} \lequiv (z_2 \vee \neg \bb{x_i > s_{i,1}})$, $\bb{x_i > s_{i,2}} \lequiv z_3$.
%

%
For every $T_t \in \mfrk{T}$,  we encode the paths 
 $R_l$ of $T_t$ as follows:
$$
    \left( \bigwedge\nolimits_{(x_i \bowtie S_i) \in R_l} \bb{x_i \bowtie S_i}  \wedge
        \bigwedge\nolimits_{\neg (x_i \bowtie S_i) \in R_l} \neg \bb{x_i \bowtie S_i} \right)
    \limply l_{tj}
$$
%
where $l_{tj}$ is a literal associated with $T_t$ in which 
the voted class is $c_j \in{\fml{K}}$.  Moreover, we enforce 
the condition that exactly one variable  $l_{tj}$ is {\it True} in 
$T_i$ --- only one class is returned by  $T_i$.  
This can be expressed as:\footnote{%
  We use the well-known cardinality networks~\cite{asin2011} for
  encoding all the cardinality constraints of the proposed encoding.}

$$ 
\left(\bigvee\nolimits_{j=1}^{K} l_{tj}\right)  ~ \wedge ~  \sum\nolimits_{j=1}^{K} l_{tj} \leq 1
$$
%

\begin{remark}
The size of the SAT encoding is linear in the number of clauses 
w.r.t. the size of the DT, yet exponential in the number of literals; an analogous 
observation holds for the MaxSAT encoding. 
While this exponential blowup in literals is an inherent characteristic 
of the encoding scheme, in practice, however, this does not pose a scalability 
concern, given that the TEs under consideration are of depth at most 3--8, keeping 
the encoding tractable.
\end{remark}

Finally, for capturing the majority voting used in RFmv, we need to
express the constraint that the counts (of individual tree selections)
for some class $c_k\not=c_j$ have to be large enough (i.e.\ greater
than, or greater than or equal to) when compared to the counts for
class $c_j$.
We start by showing how cardinality constraints can be used for
expressing such constraints. The proposed solution requires the use of
$K-1$ cardinality constraints, each comparing the counts of $c_j$ with
the counts of some other $c_k$.
Afterwards, we show how to reduce to 2 the number of cardinality
constraints used.

Let $c_j\in\fml{K}$ be the predicted class. The index $j$ is relevant
and ranges from 1 to $K$.
Moreover, let $1\le{k}<{j}\le{K}$.
Class $c_k$ is selected over class $c_j$ if:
\begin{equation} \label{eq:be01}
  \sum\nolimits_{t=1}^{n}l_{tk}>\sum\nolimits_{t=1}^{n}l_{tj} \lequiv
  \sum\nolimits_{t=1}^{n}l_{tk}+\sum\nolimits_{t=1}^{n}\neg{l_{tj}}\ge{n}
\end{equation}
Similarly, let $1\le{j}<{k}\le{K}$.
Class $c_k$ is selected over class $c_j$ if:
\begin{equation} \label{eq:be02}
  \sum\nolimits_{t=1}^{n}l_{tk}>\sum\nolimits_{t=1}^{n}l_{tj} \lequiv
  \sum\nolimits_{t=1}^{n}l_{tk}+\sum\nolimits_{t=1}^{n}\neg{l_{tj}}\ge{n+1} 
\end{equation}
Both~\eqref{eq:be01} and~\eqref{eq:be02} are cardinality constraints
that can be encoded e.g.\ with a cardinality network~\cite{asin2011}.

\paragraph{Special case for $|\fml{K}|=2$.}
If $|\fml{K}|=2$, then we can avoid direct comparison between the
counts of $c_0$ and $c_1$.
Concretely, if the prediction is $c_0$, then we are interesting in the
count of $c_1$'s to exceed the count of $c_0$'s. This can only happen
when:
$$
\sum_{t=1}^{n}l_{t1}\ge\left\lfloor\frac{n}{2}\right\rfloor+1\\
$$
Similarly, if the prediction is $c_1$, then we are interesting in the
count of $c_0$'s to exceed the count of $c_1$'s. This can only happen
when:
\[
\sum_{t=1}^{n}l_{i0}\ge\left\lfloor\frac{n}{2}\right\rfloor\\
\]
Clearly, these conditions simplify significantly the proposed
encodings.

\paragraph{General case for $|\fml{K}|>2$.}
It is possible to reduce the number of cardinality constraints as
follows.
Let us represent \eqref{eq:be01} and \eqref{eq:be02}, respectively, by
$\msf{Cmp}_{\prec}(z^{\prec}_1,\ldots,z^{\prec}_n)$ and
$\msf{Cmp}_{\succ}(z^{\succ}_1,\ldots,z^{\succ}_n)$. (Observe that the
encodings of these constraints differ (due to the different RHSes).)
Moreover, in~\eqref{eq:be01} and~\eqref{eq:be02} we replace 
 $(z^{\prec}_1,\ldots,z^{\prec}_n)$ and
$(z^{\succ}_1,\ldots,z^{\succ}_n)$ resp.\ by $l_{1k},\ldots,l_{nk}$,
for some $k$. The idea is that we will \emph{only} use two
cardinality constraints, one for \eqref{eq:be01} and one for
\eqref{eq:be02}.

Let $p_k=1$ iff $k$ is to be compared with $j$.
In this case,
$\msf{Cmp}_{\bowtie}(z^{\bowtie}_1,\ldots,z^{\bowtie}_n)$ (where
$\bowtie$ is either $\prec$ or $\succ$) is comparing the class counts
of $c_j$ with the class counts of some $c_k$.
Let us use the constraint $p_k\limply(z^{\bowtie}_t\leftrightarrow{l_{tk}})$,
with $k\in\{1,\ldots,{K}\}\setminus\{j\}$, and
$1\le{t}\le{n}$. This serves to allow associating the (free) variables
$(z^{\bowtie}_1,\ldots,z^{\bowtie}_n)$ with some set of literals
$(l_{1k},\ldots,l_{nk})$. 
Moreover, we also let $p_0\limply(z^{\succ}_t \leftrightarrow{1})$ and 
$p_j\limply(z^{\prec}_t \leftrightarrow{1})$, i.e.\ we allow $p_0$
and $p_j$ to pick \emph{guaranteed winners}, and such that
$\neg{p_0}\lor\neg{p_j}$.
Essentially, we allow for a guaranteed winner to be picked below $j$
or above $j$, but not both.
(Observe that $p_0$ is an auxiliary variable used in the encoding but 
does not represent a class of $\fml{K} = [K]$.) 
Clearly, we must pick one $k$, either below or above $j$, to pick a
class $c_k$, when comparing the counts of $c_k$ and $c_j$.
We do this by picking two winners, one below $j$ and one above $j$,
and such that at most one is allowed to be a guaranteed winner (either
0 or $j$):
\[
\left(\sum\nolimits_{r=0}^{j-1}p_r=1\right)\:\land\:\left(\sum\nolimits_{r=j}^{K}p_r=1\right)
\]
Observe that, by starting the sum at 0 and $j$, respectively, we allow
the picking of one guaranteed winner in each summation, if needed be.
To illustrate the  described  encoding above we consider again the RFmv 
running example.
\begin{example}
Consider again the RFmv $\mfrk{T}$ of the running example 
\cref{fig:RF(mv)} and the instance 
$\mbf{v} = (6.0, 3.5, 1.4, 0.2)$ , giving the prediction $c_1$ (i.e.\ ``Setosa''). 
Let us define the Boolean variables $\bb{x_1 < 5.55}$ and 
$\bb{x_1 \in I^1_k}, k\in[2]$ associated with feature  {\it sepal.length}, 
$\bb{x_2 < 2.75}$ and $\bb{x_2 \in I^2_k}, k\in[2]$ associated with feature   
{\it sepal.width},  $\bb{x_3 < 4.75}$  and $\bb{x_3 \in I^3_k}, k\in[2]$ 
associated with feature  {\it petal.length}, and 
$\bb{x_4 < 0.75}, \bb{x_4 < 1.55}, \bb{x_4 < 1.65}$ and  
 $\bb{x_4 \in I^4_k}, k\in[4]$ 
associated with features {\it petal.width}. 
 Also, to represent the classes $c_1, c_2, c_3$,  
 we associate Boolean variables  $l_{tj}, j\in[K]$ denoting 
 the classes for each tree $T_t$.
 Hence, the corresponding set of encoding constraints is:
  $$ 
  \begin{array}{rr}
    \bb{x_1 < 5.55} \land \bb{x_2 < 2.75} \rightarrow l_{12}, \;  \bb{x_1 < 5.55} \land \neg \bb{x_2 < 2.75} \rightarrow l_{11} \\   
    \neg\bb{x_1 < 5.55} \land \bb{x_4 < 1.55} \rightarrow l_{12}, \;  \neg\bb{x_1 < 5.55} \land \neg\bb{x_4 < 1.55} \rightarrow l_{13} \\
     \bb{x_4 < 0.75} \rightarrow l_{21}, \; \neg\bb{x_4 < 0.75} \land \bb{x_3 < 4.75} \rightarrow l_{22}, \; 
     \neg\bb{x_4 < 0.75} \land \neg\bb{x_3 < 4.75} \rightarrow l_{23} \\
     \bb{x_4 < 0.75} \rightarrow l_{31}, \; \neg\bb{x_4 < 0.75} \land \bb{x_4 < 1.65} \rightarrow l_{32}, \; 
     \neg\bb{x_4 < 0.75} \land \neg\bb{x_4 < 1.65} \rightarrow l_{33} \\    
    \\
    (l_{11} + l_{21} + l_{31}) \ge 2,\;  \bb{x_1 \in I^1_2},\; \bb{x_2 \in I^2_2},\;  \bb{x_3 \in I^3_1},\;  \bb{x_4 \in I^4_1}  & \hspace*{0.7cm} \Box
  \end{array}
  $$
 %
 %
 Observe that $\{ \bb{x_1 \in I^1_2}, \bb{x_2 \in I^2_2},  \bb{x_3 \in I^3_1},  \bb{x_4 \in I^4_1} \}$ denotes the set of
 the relaxable constraints whereas the remaining are the hard
 constraints.
 Note that in this example, the constraints are not addressed  in 
 a clausal form in order to simplify the understanding for the reader.   
\end{example}

\paragraph{Encoding Using a Comparator.}
We can represent $\sum\nolimits_{t=1}^{n}l_{tk}$ with some circuit
that sorts the bits $l_{tk}$ in unary, i.e.\ the output unary
representation represents a count of the number of $l_{tk}$ bits
assigned value 1. Let
\[
({s_1},\ldots,s_n)=\msf{Sort}(l_{1k},\ldots,l_{nk})
\]
where the MSB is $s_1$ and the LSB is $s_n$\footnote{%
  MSB denotes \emph{most significant bit}, and LSB denotes \emph{least
    significant bit}.}.
(Well-known examples include cardinality networks, sorting networks,
etc.)

Since we want to compare $(z_1,\ldots,z_M)$ with
$(l_{1j},\ldots,l_{nj})$, we are going to develop a comparator for the
output bits of the two cardinality networks.
Let,
$$
\begin{array}{l}
  ({s_1},\ldots,{s_n})=\msf{Sort}({z_1},\ldots,{z_n})\\[2pt]
  ({t_1},\ldots,{t_n})=\msf{Sort}({l_{1j}},\ldots,{l_{nj}})\\
\end{array}
$$
where $s_1,t_1$ are the MSB's and $s_n,t_n$ are the LSB's.

\begin{itemize}
  \item
    {\bf Case: $({s_1},\ldots,{s_n})<({t_1},\ldots,{t_n})$.}
    Let $eq_0=1$, $lt_0=0$ (respectively for \emph{equal to} and
    \emph{less than}). Then, we can recursively define $lt_i$ and $eq_i$,
    with $i=1,\ldots,n$, as follows:
    $$
    \begin{array}{l}
      lt_i \leftrightarrow
      lt_{i-1} \lor
      eq_{i-1} \land ( {t_i} \land \neg{s_i} ) \\[1pt]
      eq_i \leftrightarrow
      eq_{i-1} \land ( {t_i} \leftrightarrow {s_i} ) \\
    \end{array}
    $$
    Hence, we have $({s_1},\ldots,{s_n})<({t_1},\ldots,{t_n})$ iff
    $lt_{n}=1$. 
    
    \item 
    {\bf Case: $({s_1},\ldots,{s_n})\le({t_1},\ldots,{t_n})$.}
    Given the above, we have $({s_1},\ldots,{s_n})\le({t_1},\ldots,{t_n})$
    iff $lt_{n}=1\lor{eq_n}=1$.
    %
 \end{itemize}

\subsection{MaxSAT-based Explanations for TEs\protect\footnotetext{This work is an extension of our paper~\cite{iism-aaai22}.}} \label{sec:mxsat}
%
%
We start by detailing how to encode the trees of a TE $\mfrk{T}$ 
before describing the classification function $\kappa$. 
This paper exploits the propositional encoding used in recent
work for computing formal explanations for  RFs and  BTs 
(XGBoost~\cite{guestrin-kdd16a})~\cite{ims-ijcai21,iism-aaai22}.
To the best of our knowledge, there is no existing propositional 
encoding to explain RF with weighted voting. 
As a result, we aim to address this gap and present a generic 
MaxSAT encoding for TE capturing RFwv, RFmv and BT.
The encoding comprises: 1) the structure of a TE $\mfrk{T}$, and 2) the
weighted voting prediction $(\kappa(\mbf{x})\not=c)$.
%

\subsubsection{Propositional Encoding of a Tree} \label{sec:mxenc}
Given a tree ensemble $\mfrk{T}$ defining classifier $\kappa$, 
for each feature $i \in \fml{F}$
we can define the set of \emph{split points} $S_i \subset \mbb{D}_i$ that
appear for that feature in all trees, i.e. $S_i = \{ d ~|~ x_i < d
\text{~appears in some tree~} T_1, \ldots T_n \}$.
Suppose $[s_{i,1}, s_{i,2}, \ldots, s_{i,|S_i|}]$ are the split points in
$S_i$ sorted in increasing order.
Using this we can define set of disjoint intervals for feature $i$:
$I^i_1 \equiv [\min(\mbb{D}_i),  s_{i,1})$,
$I^i_2 \equiv  [s_{i,1}, s_{i,2})$, \ldots,
$I^i_{|S_i|+1} \equiv [s_{i,|S_i|}, \max(\mbb{D}_i)]$
where $\mbb{D}_i = I^i_1 \cup I^i_2 \cup \cdots \cup I^i_{|S_i|+1}$.
The key property of a TE is that  
$\forall \mbf{x} \in I^1_{e_1} \times I^2_{e_2} \times \cdots \times I^m_{e_m}$ 
s.t.\  $e_i$ is defined as the interval index for feature $i$ which includes
the value of $\mbf{v}$, i.e. $v_i \in I^i_{e_i}$, then we have   
$\kappa(\mbf{x}) =  \kappa(\mbf{v}) = c$.
That is we only need to reason about intervals rather than particular values
to explain the behaviour of a tree ensemble.
%


%
Next, we can ``name'' these $|S_i|+1$ intervals with the use of
auxiliary Boolean variables $\bb{x_i\in I^i_k}$, $i\in\mathcal{F}$ and
$k\in[|S_i|+1]$ s.t.\ $\bb{x_i\in I^i_k}=1$ iff $x_i\in I^i_k$.
These variables can be used to represent \emph{concrete} feature
values of a data instance $\mbf{v}\in\mbb{F}$, i.e.\ with a slight
abuse of notation we can use $\bb{x_i=v_i}\equiv \bb{x_i\in I^i_k}$ iff
$v_i\in I^i_k$.
Furthermore, let us use additional Boolean variables $\bb{x_i < \sep_{i,k}}$ for
each of the threshold values $\sep_{i,k}$, $k\in[m_i]$, for feature 
$i$, i.e.\ variable $\bb{x_i < \sep_{i,k}}=1$ iff $x_i < \sep_{i,k}$.
Then the domain $D_i$ can be expressed using the following constraints
\begin{align}
  \bb{x_i < \sep_{i,k}} \rightarrow \bb{x_i < \sep_{i,k+1}},\;\forall_{k\in [|S_i|-1]}
  \label{eq:order} \\
  \bb{x_i\in I^i_1} \leftrightarrow \bb{x_i < \sep_{i,1}} \label{eq:firstint} \\
  \bb{x_i\in I^i_{k+1}} \leftrightarrow(\neg \bb{x_i < \sep_{i,k}} \land
  \bb{x_i < \sep_{i,k+1}}),\;\forall_{k\in [|S_i|-1]} \label{eq:int}\\
   \bb{x_i\in I^i_{|S_i|+1}}\leftrightarrow \neg\bb{x_i < \sep_{i,|S_i|}} \label{eq:lastint} \\
  \bigvee\nolimits_{k\in[|S_i|+1]}  \bb{x_i\in I^i_{k}} \label{eq:al1int}
\end{align}
Note that in the case of one threshold value $\cond$ for a feature
$i\in\mathcal{F}$, \emph{two} distinct intervals exist and it
suffices to consider a single Boolean variable \bb{x_i < d} to represent it.
Also note that the encoding above can be related to the order encoding
of integer domains studied in the context of mapping CSP to
SAT~\cite{ansotegui-sat04}.
%
%

\begin{example} \label{ex:enc}
  In the BT model shown in Figure~\ref{fig:BT},
  there are 4 thresholds for feature \emph{petal.length}, namely,
  values 2.45, 3, 4.75, and 4.85.
  Assuming \emph{petal.length} is a 3rd feature, let us introduce
  variables $\bb{x_3 < 2.45}, \dots, \bb{x_3 < 4.85}$, and $\bb{x_3\in I^3_k}$ for $k\in[5]$.
  %
  %
  The following constraints can be used to express the domain of the
  possible values for feature 3:
  $$
  \begin{array}{rr}
    \bb{x_3 < 2.45}\rightarrow \bb{x_3 < 3},\;\bb{x_3 < 3}\rightarrow \bb{x_3 < 4.75},\;\bb{x_3 < 4.75}\rightarrow \bb{x_3 < 4.85} \\
    \bb{x_3\in I^3_2}\leftrightarrow(\neg\bb{x_3 < 2.45}\land \bb{x_3 < 3}),\;\bb{x_3\in I^3_3}\leftrightarrow(\neg\bb{x_3 < 3}\land \bb{x_3 < 4.75}) \\
    \bb{x_3\in I^3_4}\leftrightarrow(\neg\bb{x_3 < 4.75}\land \bb{x_3 < 4.85}),\;\bb{x_3\in I^3_1}\leftrightarrow \bb{x_3 < 2.45},\;\bb{x_3\in I^3_5}\leftrightarrow \neg\bb{x_3 < 4.85} \\
    (\bb{x_3\in I^3_1}\lor \bb{x_3\in I^3_2} \lor \bb{x_3\in I^3_3} \lor \bb{x_3\in I^3_4} \lor \bb{x_3\in I^3_5})& \hspace*{0.7cm} \Box
  \end{array}
  $$
\end{example}


A leaf node $\mfrk{n} \in T_i$ can be encoded by a Boolean
variable $t_r$ s.t.\ $t_r=1$ iff node $\mfrk{t}_r$ is reached.
Each leaf node $\mfrk{n}\in T_t$ corresponds to a path
$R_l$, an encoded by a Boolean variable $\bb{R_l}$ such that 
$\bb{R_l} = 1$ iff node $\mfrk{n}$ is reached.
Thus, we can encode each path $R_l$ as a conjunction of
literals over variables $\bb{x_i < d}, d \in S_i$ representing the conditions 
$x_i < d$ on the features $i$ appearing in $R_l$, and relate it with the
corresponding leaf node $\mfrk{n}$ using the following constraint:
\begin{equation} \label{eq:leaf}
    \left( \bigwedge\nolimits_{(x_i < d) \in R_l} \bb{x_i < d} \wedge
    \bigwedge\nolimits_{\neg (x_i < d) \in R_l} \neg \bb{x_i < d} \right) \leftrightarrow \bb{R_l} 
\end{equation}
Note that if multiple trees in an ensemble $\mfrk{T}$ share a path
$R_l$ (this often happens for tree ensembles trained with XGBoost), a
single variable $\bb{R_l}$ can be used to represent all of its occurrences.

Together with the above constraints, we must ensure that exactly one
path in each tree $T_t$ is executed at a time.
This can be enforced by using constraint
$\sum_{\bb{R_l} \in T_t}\bb{R_l}=1$.

\begin{example} \label{ex:leaf}
  As shown in Example~\ref{ex:enc}, feature 3 \emph{(petal.length)}
  has 5 intervals and so 4 associated variables $\bb{x_3 < \sep_{3,k}}$, $k\in[4]$,
  e.g.\ $\bb{x_3 < \sep_{3,1}}=1$ iff $x_3< 2.45$.
  As tree $T_1$ has 2 paths and so 2 leaf nodes, they can be
  represented using the constraints:
  \begin{align*}
    \bb{x_3 < 2.45} \leftrightarrow \bb{R_1},\;\;
    \neg\bb{x_3 < 2.45}\leftrightarrow \bb{R_2},\;\;
    \bb{R_1} + \bb{R_2} = 1
  \end{align*}
  where the corresponding leaf nodes $\mathfrak{n}_1$ and
  $\mathfrak{n}_2$ have weights $w_1=0.42762$ and $w_2=-0.21853$,
  respectively.
  Also, observe that variables $\bb{R_1}$ and $\bb{R_1}$ can be reused to encode
  the leaves of tree $T_4$.
  \qed
\end{example}

Now, let us denote the CNF formula encoding the paths and the leaf
nodes of a tree $T_t$ in ensemble $\mfrk{T}$, by $\fml{H}_{T_i}$,
following \eqref{eq:leaf}.
Also, let us denote the CNF formula encoding the domain of each
feature $i\in\fml{F}$ including constraints
\eqref{eq:order}--\eqref{eq:al1int} by $\fml{H}_{D_i}$.
Finally, in the following, we will use
\begin{equation} \label{eq:henc}
  \fml{H}=\bigwedge\nolimits_{T_t\in\mfrk{T}}{\fml{H}_{T_t}} \land
  \bigwedge\nolimits_{i\in\fml{F}}{\fml{H}_{D_i}}
\end{equation}
to represent the hard part of the MaxSAT encoding of TE $\mfrk{T}$.


%
\subsubsection{Dealing with Classifier Predictions} \label{sec:mxobj}
In order to formally reason about TEs, one has to additionally encode
the process of determining the prediction of the model obtained for an
input data instance according to~\eqref{eq:wght}.
For this, when using the language of SMT~\cite{inms-corr19}, it
suffices to introduce a real variable for each tree $T_t$ as well
as for the total weight $W(c,\mbf{x})$ of class $c$, and relate them with a
linear constraint \eqref{eq:wght}.
As a result, entailment queries \eqref{eq:axp} boil down to deciding
whether a misclassification is possible for the classifier
$\kappa(\mbf{x})$ given a candidate explanation $\fml{X}$, i.e.\ whether
there is a data instance $\mbf{v}'$ that complies with $\fml{X}$ such
that $W(c',\mbf{v}') \geq W(c,\mbf{v}')$ for some class 
$c'\in\fml{K},~c'\neq c$.

In contrast to SMT, in the case of propositional logic one has to
represent the linear constraints \eqref{eq:wght} integrating
\emph{all} individual leaf-node variables $\bb{R_l}$ multiplied by the
corresponding weight, i.e.\ summands $w_l \cdot \bb{R_l}$, to determine the
total weights of the classes.
The total weights of a class $c$ can be determined by 
$\sum_c=\sum_{T_{t}\in\mfrk{T}}{\sum_{\bb{R_l}\in T_{t}, cl(l) = c}{w_l\cdot \bb{R_l} }}$.
(This is valid as we use $\sum_{\bb{R_l}\in T_i}\bb{R_l}=1$ to enforce
exactly one path to be taken in each tree.)
Recall that the goal is to reason about $\mfrk{T}$ to determine whether 
there exist a counter-example $\mbf{x}$, 
$\kappa(\mbf{x}) \neq c$, that is consistent with $\fml{X}\subset\fml{F}$ (i.e.\ 
$\bigland\nolimits_{i\in\fml{X}} (x_i \in I^i_{e_i})$).
This can be achieved by verifying  $\sum_{c'} > \sum_c$ where $c$ 
is the target class and $c' \in \fml{K}, c' \neq c$ are the opponent 
classes.
This involves using either cardinality (if all the weights are 1) or
pseudo-Boolean (if the weights are arbitrary real numbers) encodings,
which may be quite expensive in practice.
What is worse, checking if a given candidate explanation $\fml{X}$ for
$\kappa(\mbf{v})=c$ entails the prediction requires to efficiently
reason about pairwise inequalities comparing the total weight of
class $c$, with the total weights of all the other classes
$c'\neq c$, each computed using the above cardinality or PB
constraints, i.e.\ $\sum_{c'} > \sum_c$.
\begin{example} \label{ex:ineq}
  For our running example, consider class 1 (``setosa''),
  represented by trees $T_1$ and $T_4$.
  Assume that the leaf nodes of $T_1$ are encoded as $\bb{R_1}$ with weight
  0.42762 and $\bb{R_2}$ with weight -0.21853.
  Observe that variables $\bb{R_1}$ and $\bb{R_2}$ can be reused to encode the
  leaves of $T_4$ with weights 0.29522 and -0.19674, respectively.
  Hence, the total weight of class ``setosa'' can be expressed as
  $\sum_1=0.42762\cdot \bb{R_1} - 0.21853\cdot \bb{R_2} + 0.29522\cdot \bb{R_1} -
  0.19674\cdot \bb{R_2}=0.72284\cdot \bb{R_1} - 0.41527\cdot \bb{R_2}$.
  By applying similar reasoning, we can assume that the total weight
  of class 3 (``virginica'') can be computed as $\sum_3=-0.41645\cdot
  \bb{R_3} + 0.16249\cdot \bb{R_4} + 0.72452\cdot \bb{R_5}$.
  \qed
\end{example}

We observe that in order to avoid the bottleneck of dealing with hard
clauses $\fml{H}$ together with inequalities $\sum_{c'} \geq \sum_c$ 
one can consider optimizing objective function  $\sum\nolimits_{c'} -
\sum\nolimits_{c}$ subject to hard clauses $\fml{H}$ defined
in~\eqref{eq:henc}.
We formulate this in the {\it weighted partial} MaxSAT formula as 
a pseudo-Boolean constraint:   
\begin{equation*}
\bigland\nolimits_{T_{t}\in\mfrk{T}}   \left[
\bigland\nolimits_{\bb{R_l}\in T_{t}, cl(l) = c'} (\bb{R_l}, w_l)  ~~~\wedge~~~ 
 \:\:  \bigland\nolimits_{\bb{R_l}\in T_{t}, cl(l) = c} (\bb{R_l}, -w_l)  \right] 
\end{equation*}
which defines the objective function to maximize:   
$\fml{S}_{c,c'}=\sum\nolimits_{c'} - \sum\nolimits_c$ for each
formula checking $(\kappa(\mbf{x}) = c' ) \neq c$ 
for a class $c'$ against the target $c$.
Note that it is not required to search for the optimal solution 
for $\fml{S}_{c,c'}$ since we are only interested in identifying 
a class $c'$ that would have a greater score that $c$.

\begin{proposition} \label{prop:opt}
  The optimal value of the objective function above is non-negative
  iff there is a data instance (which can be extracted from the
  corresponding optimal solution) that is classified by the model as
  class $c'$ instead of class $c$.  \qed
\end{proposition}
Also, note that if there are $K>2$ classes, one has to deal with $K-1$
optimization problems.
In general, in the context of computing an AXp for TE models, the
following holds:
\begin{corollary} \label{prop:axp}
  Given a data instance $\mbf{v}\in\mbb{F}$ s.t.\ $\kappa(\mbf{v})=c$,
  a subset of features $\fml{X}\subseteq \fml{F}$ is an AXp for
  $\mbf{v}$ iff the optimal values of the objective functions of the
  optimization problems for all classes $c'\in\fml{K}$,
  $c'\neq c$:
  \begin{align}
    \text{maximize} & & \fml{S}_{c,c'}=\sum\nolimits_{c'} -
    \sum\nolimits_{c} \label{eq:obj} \\
    \text{subject to} & & \fml{H} \land
    \bigwedge\nolimits_{i\in\fml{X}}{\bb{x_i=v_i}} \label{eq:hard}
  \end{align}
  are negative.
\qed
\end{corollary}

\begin{example} \label{ex:opt}
  The objective function to replace inequality $\sum_3\geq\sum_1$ from
  Example~\ref{ex:ineq} is $\sum_3-\sum_1=-0.41645\cdot \bb{R_3} +
  0.16249\cdot \bb{R_4} + 0.72452\cdot \bb{R_5} - 0.72284\cdot \bb{R_1} +
  0.41527\cdot \bb{R_2}$.
  In MaxSAT, it can be written as a set of weighted unary clauses:
  $$
  \fml{S}_{1,3}=\left\{
      \begin{array}{l}
        (\neg\bb{R_1},0.72284), (\bb{R_2},0.41527), \\
        (\neg\bb{R_3},0.41645), (\bb{R_4},0.16249), (\bb{R_5},0.72452) \\
      \end{array}
    \right\}
  $$
  %

  %
  %
\end{example}

\begin{remark}
  To represent a negative summand $-w_l\cdot \bb{R_l}$ of 
  $\sum_{c'}-\sum_{c}$ as a Boolean literal, one has to use a
  constraint $-w_l\cdot \bb{R_l} = w_l\cdot\neg\bb{R_l} - w_l$.
  Observe that constants $-w_l$ should be taken into account when
  determining the sign of the optimal value of objective function
  $\sum_{c'}-\sum_{c}$.
\end{remark}

\subsubsection{Entailment check with MaxSAT} 
As soon as the formulas \eqref{eq:obj}--\eqref{eq:hard} are
constructed for all classes $c'\neq c$, a MaxSAT solver can be
applied to check if feature subset $\fml{X}$ is an AXp for
$\kappa(\mbf{v})=c$.
This entailment check is detailed in Algorithm~\ref{alg:call}.
Concretely, given a candidate explanation $\fml{X}$ for
$\kappa(\mbf{v})=c$ made by a TE $\mfrk{T}$ encoded into the set
$\fml{H}$ of hard clauses and a number of objectives 
$\fml{S}$, the procedure iterates over all the relevant
objective functions (see line~\ref{ln:allobj}), each represented by a
set $\fml{S}_{c,c'}$ of weighted literals, and checks if another class
$c'$ can get a higher total score given the literals enforcing
$\bigwedge_{i\in\fml{X}}{(x_i=v_i)}$.
This is done by making a single MaxSAT call.
If such a misclassification is possible, i.e.\ when the optimal value
of the objective function is non-negative, $\fml{X}$ does not entail
prediction $c$.
Otherwise, it does.

\begin{algorithm}[t]
%
\SetInd{0.4em}{0.4em}
\SetStartEndCondition{ }{}{}%
\SetKwProg{Fn}{def}{\string:}{}
\SetKwFunction{Range}{range}
\SetKw{KwTo}{in}\SetKwFor{For}{for}{\string:}{}%
\SetKwIF{If}{ElseIf}{Else}{if}{:}{elif}{else:}{}%
\SetKwFor{While}{while}{:}{}%
\SetKwFor{ForEach}{foreach}{:}{}%
\AlgoDontDisplayBlockMarkers\SetAlgoNoEnd\SetAlgoNoLine%
\LinesNumbered
\DontPrintSemicolon%
\SetKw{KwNot}{not\xspace}
\SetKw{KwAnd}{and\xspace}
\SetKw{KwOr}{or\xspace}
\SetKw{KwBreak}{break\xspace}
\SetKwData{false}{{\small false}}
\SetKwData{true}{{\small true}}
\SetKwData{st}{{\small \slshape st}}
\SetKwFunction{CNF}{CNF}
\SetKwFunction{SAT}{SAT}
\SetKwInOut{Input}{input}\SetKwInOut{Output}{output}
\newcommand\commfont[1]{\footnotesize\em{#1}}
\SetKwComment{tcpy}{\# }{}
\SetCommentSty{commfont}
\SetKwFunction{tfunc}{{\sc ExtractAXp}} %
\SetKwFunction{efunc}{{\sc EntCheck}} %
\SetKwFunction{cfunc}{{\sc Encode}} %
\SetKwFunction{mfunc}{{\sc MaxSAT}} %
\SetKwFunction{ofunc}{{\sc ObjValue}} %
\Func \efunc{$\langle\fml{H}$, $\fml{S}\rangle$, $\fml{E} = (\fml{M},(\mbf{v},c))$, $\fml{X}$} \;

\Indp
\KwIn{
  $\fml{H}$: Hard clauses, \hspace{0.7em}$\fml{S}$: Objective functions, \\
  \hspace{3.43em}$\mbf{v}$: Input instance, \hspace{0.3em}$c$: Pred., i.e.\ $\kappa(\mbf{v})=c$, \\
  \hspace{3.22em}$\fml{X}$: Candidate explanation
}
\KwOut{
  \true or \false
}
\setcounter{AlgoLine}{0}
\BlankLine
{
  \ForEach(\tcpy*[f]{all relevant obj. functions for $c$}\label{ln:allobj}){$\fml{S}_{c,c'}\in\fml{S}$}{
      $\mu\gets\mfunc(\fml{H}\land\bigwedge\nolimits_{i\in\fml{X}}{\bb{x_i=v_i}}, \fml{S}_{c,c'})$\;

      \uIf(\tcpy*[f]{non-negative objective?}){$\ofunc(\mu)\geq 0$}{
        \Return \false \tcpy*[r]{misclassification reached}
      }
    }
    \BlankLine
    \Return \true \tcpy*[r]{$\fml{X}$ indeed entails prediction $c$}
}
\Indm
%

  \caption{Entailment check with MaxSAT} \label{alg:call}
\end{algorithm}

\subsubsection{Incremental MaxSAT with assumptions.}
\label{sec:assump}
Although any \emph{off-the-shelf} MaxSAT
solver~\cite{morgado-cj13,ansotegui-aij13,bacchus-cp11} can be
exploited in Algorithm~\ref{alg:call} as is (assuming that the chosen
MaxSAT solver can cope with Partial Weighted CNF formulas with real
weights),
computing a single AXp may involve quite a large number of oracle
calls, and invoking a new MaxSAT solver from scratch at each iteration
of~\cref{alg:del} (resp.\  \cref{alg:cxp-del}) may be computationally expensive.
An alternative is to create all the necessary MaxSAT oracles
once, at the initialization stage of the explanation approach, and
reuse them \emph{incrementally} as needed with varying sets
of assumptions 
$\fml{A}\triangleq\{\bb{x_i=v_i}\;|\;i\in\fml{X}\}$ --- in a way 
similar to the well-known MiniSat-like incremental
assumption-based interface used in SAT solvers~\cite{een-entcs03}.

\begin{algorithm}[t]
%

\def\HiLi{\leavevmode\rlap{\hbox to \hsize{\color{tdgray1!20}\leaders\hrule height .8\baselineskip depth .7ex\hfill\color{white}{\hspace{.8em}}}}}
\definecolor{hlcolor}{HTML}{e7e8e7}
\sethlcolor{hlcolor}

\SetInd{0.4em}{0.4em}
\SetStartEndCondition{ }{}{}%
\SetKwProg{Fn}{def}{\string:}{}
\SetKwFunction{Range}{range}
\SetKw{KwTo}{in}\SetKwFor{For}{for}{\string:}{}%
\SetKwIF{If}{ElseIf}{Else}{if}{:}{elif}{else:}{}%
\SetKwFor{While}{while}{:}{}%
\SetKwFor{ForEach}{foreach}{:}{}%
\AlgoDontDisplayBlockMarkers\SetAlgoNoEnd\SetAlgoNoLine%
\LinesNumbered
\DontPrintSemicolon%
\SetKw{KwNot}{not\xspace}
\SetKw{KwAnd}{and\xspace}
\SetKw{KwOr}{or\xspace}
\SetKw{KwBreak}{break\xspace}
\SetKwData{false}{{\small false}}
\SetKwData{true}{{\small true}}
\SetKwData{st}{{\small \slshape st}}
\SetKwData{cost}{{\small \slshape cost}}
\SetKwFunction{CNF}{CNF}
\SetKwFunction{SAT}{SAT}
\SetKwInOut{Input}{input}\SetKwInOut{Output}{output}
\newcommand\commfont[1]{\footnotesize\em{#1}}
\SetKwComment{tcpy}{\# }{}
\SetCommentSty{commfont}
\SetKwFunction{tfunc}{{\sc ExtractAXp}} %
\SetKwFunction{efunc}{{\sc EntCheck}} %
\SetKwFunction{cfunc}{{\sc GetCore}} %
\SetKwFunction{wfunc}{{\sc CoreWt}} %
\SetKwFunction{Mfunc}{{\sc MaxSAT}} %
\SetKwFunction{mfunc}{{\sc GetModel}} %
\SetKwFunction{lfunc}{{\sc Record}} %
\SetKwFunction{sfunc}{{\sc SAT}} %
\SetKwFunction{rfunc}{{\sc Process}} %
\SetKwFunction{dfunc}{{\sc ValidCores}} %
\Func \Mfunc{$\phi$, \hl{\strut$\fml{A}$}} \;

\Indp
\KwIn{
  $\phi$: Partial CNF (hard clauses and literal objectives), \\
  \hspace{3.015em}\hl{\strut $\fml{A}$: Set of assumption literals}
}
\KwOut{
  $\mu$: MaxSAT model
}
\setcounter{AlgoLine}{0}
\BlankLine
{
  $\cost\gets 0$\tcpy*[r]{initially, cost is 0}\label{ln:icinit}
  \BlankLine
  \HiLi $\fml{C}\gets\dfunc(\phi,\fml{A})$\tcpy*[r]{get valid unsat. cores}\label{ln:detect}
  \HiLi \ForEach(\tcpy*[f]{iterate over known cores $\kappa$}\label{ln:eachknown}){$\kappa\in\fml{C}$}{
    \HiLi $\cost \gets \cost + \wfunc(\kappa)$\tcpy*[r]{update cost}\label{ln:incr1}
    \HiLi $\phi \gets \rfunc(\phi ,\kappa)$\tcpy*[r]{process $\kappa$ and update $\phi$}\label{ln:proc1}
  }
  \BlankLine
  \While(\label{ln:iloop}\tcpy*[f]{until $\phi$ gets satisfiable}){$\sfunc(\phi,\;$\hl{$\fml{A}$}$)=\false$}{
    $\kappa\gets\cfunc(\phi)$\tcpy*[r]{new unsatisfiable core}\label{ln:icore}
    $\cost \gets \cost + \wfunc(\kappa)$\tcpy*[r]{update cost}\label{ln:incr2}
    $\phi \gets \rfunc(\phi ,\kappa)$\tcpy*[r]{process $\kappa$ and update $\phi$}\label{ln:iproc2}

    \HiLi $\lfunc(\phi,\fml{A},\kappa)$\tcpy*[r]{record $\kappa$ for the future}\label{ln:record}
  }
  \BlankLine
  \Return $\mfunc(\phi)$\tcpy*[r]{$\phi$ is now satisfiable}\label{ln:imodel}
}
\Indm
%

  \caption{Incremental core-guided MaxSAT} \label{alg:imxs}
\end{algorithm}

A possible setup of a generic incremental \emph{core-guided} MaxSAT
solver is shown in Algorithm~\ref{alg:imxs} where the flow of a
(non-incremental) core-guided solver is \emph{augmented} with several
highlighted parts, which briefly overview the changes needed for
incrementality.
(Please ignore lines~\ref{ln:detect}--\ref{ln:proc1} and
\ref{ln:record} for now.)
In general, a core-guided solver operates as a loop iterating as long
as the formula it is dealing with is unsatisfiable
(line~\ref{ln:iloop}).\footnote{The reader is referred
to~\cite{morgado-cj13,ansotegui-aij13} for details on
core-guided MaxSAT.}
(At the beginning, the solver aims at satisfying all the hard and 
literal objectives of the input formula $\phi$.)
Each iteration of the loop extracts an unsatisfiable core $\kappa$
(line~\ref{ln:icore}) and processes it by relaxing the clauses in
$\kappa$ and adding new constraints allowing some of the clauses in
$\kappa$ to be falsified.
The \emph{weight} of $\kappa$ contributes to the total \textsf{\small
\slshape cost} of the MaxSAT solution.
As soon as the formula gets satisfiable, the solver stops and reports
its satisfying assignment.

Given the above, incrementality in a core-guided solver with the use
of the set $\fml{A}$ of assumptions can be made to work by handing
$\fml{A}$ over to the underlying SAT solver at every iteration of the
MaxSAT algorithm (see the highlighted modification made in
line~\ref{ln:iloop}).
This way all the clauses learnt by the SAT solver during the entire
MaxSAT oracle call can be reused later in the following MaxSAT calls.

Moreover, as each call to Algorithm~\ref{alg:call} ``equates'' to an
entailment query~\eqref{eq:axp}, where a reasoner aims at refuting a
given input formula, one may want to extract a subset of assumptions
in $\fml{A}$ that are deemed responsible for the
entailment~\eqref{eq:axp} to hold.
This can be done by aggregating all the subsets
$\fml{A}'\subseteq\fml{A}$ appearing in individual unsatisfiable cores
$\kappa$ extracted by the underlying SAT solver (see
line~\ref{ln:icore}), i.e.\ $\fml{A}'=\kappa\cap\fml{A}$, such that
the union of all such subsets $\fml{A}'$, i.e.\
$\fml{A}^*=\cup\fml{A}'$, is guaranteed to be responsible for the
entailment \eqref{eq:axp} to hold.
Notice that this is similar to the \emph{core extraction mechanism} of
SAT solvers.
It can be used in practice to avoid iterating over all features in
$\fml{F}$ (see line~\ref{ln:iter} of Algorithm~\ref{alg:del}) and
instead focus on those features that ``matter''.

\subsubsection{Reusing unsatisfiable cores.} \label{sec:cores}
To further improve the MaxSAT oracle, we can
\emph{reuse unsatisfiable cores} detected in the previous MaxSAT
calls, similar to core caching studied in the context of \emph{minimal
correction subset} (MCS) enumeration~\cite{previti-sat17}.
Indeed, Algorithm~\ref{alg:del} provides all the MaxSAT solvers
dealing with objective $\fml{S}_{c,c'}$ with exactly the same
formula $\phi\equiv\fml{H}\land\fml{S}_{c,c'}$ but a varying set of
assumptions $\fml{A}$.
Hence, some of the cores extracted from $\phi$ for assumptions
$\fml{A}_1$ may be reused with assumptions $\fml{A}_2$, s.t.\
$\fml{A}_1\cap\fml{A}_2\neq\emptyset$, thus saving a potentially large
number of SAT oracle calls in Algorithm~\ref{alg:imxs}.

Although implementations of core reuse may vary, a high-level view on
the corresponding changes in the solver are shown in
lines~\ref{ln:detect}--\ref{ln:proc1} and~\ref{ln:record} of
Algorithm~\ref{alg:imxs}.\footnote{The pseudocode and the description
of core reuse omit the low-level details and provide a basic overview
of the general idea.}
Namely, each core $\kappa$ is \emph{recorded} for future use
(line~\ref{ln:record}).
Next time Algorithm~\ref{alg:imxs} is called, it starts by examining
the set of recorded cores and determining, which of them may be reused
in the current call (see~line~\ref{ln:detect}).
All of such cores are processed the standard way
(lines~\ref{ln:eachknown}--\ref{ln:proc1}), each saving a single SAT
call.
Afterwards, the MaxSAT solver proceeds as normal.

Let us recall that each core $\kappa_i$ found at iteration $i$ of a
MaxSAT algorithm comprises a subset $\fml{S}'$ of 
weighted literal objectives of 
$\fml{S}$ that are unsatisfiable together with the hard clauses
$\fml{H}\subseteq\phi$.
In our incremental MaxSAT solver, $\kappa_i$ may also include a subset
$\fml{A}'$ of the assumption literals $\fml{A}$.
Additionally, in the case of the algorithms based on 
\emph{cardinality objective functions}~\cite{mdms-cp14}, $\kappa_i$ may include
literals ``representing'' some other (previously found) unsatisfiable
cores $\kappa_j$, $j<i$.
This way, each core $\kappa_i$ can be seen as \emph{depending} on
(1)~the corresponding subset $\fml{S}'$ of literal objectives,
(2)~assumptions $\fml{A}'$, and (3)~cores $\kappa_j$, $j<i$.
As a result, thanks to the fact that each of these are represented by
Boolean literals,\footnote{Practically, each $(w_i, l_i) \in\fml{S}$
is represented by a unique selector literal $s$, i.e.\ we use clause
$l_i\lor\neg{s}$.} it is this dependency of $\kappa_i$
that can be recorded for future use.
Namely, let $\fml{D}_i$ be the set of all such literals that
$\kappa_i$ depends on.
Then we can use a \emph{separate} SAT solver to store the dependency
in the form of a clause
$\left(\bigwedge_{l\in\fml{D}_i}{l}\right)\rightarrow \zeta_i$, where
$\zeta_i$ is a Boolean variable introduced to represent $\kappa_i$.
This way, detecting which cores can be reused for formula $\phi$ under
assumption literals $\fml{A}$, i.e.\ a call to $\dfunc(\phi,\fml{A})$,
can be done by applying unit propagation in the external SAT solver
given all the literal objectives $\fml{S}\subseteq\phi$ and assumptions
$\fml{A}$.
Note that there is no need to make a full-blown SAT call here as
literals $\zeta_i$ for all reusable cores $\kappa_i$ will be
propagated in the right order, i.e.\ $\zeta_j$ will precede $\zeta_i$
for $j<i$, given $\fml{S}$ and $\fml{A}$.

\subsubsection{Distance-based stratification.} \label{sec:strat}
In practice, weighted formulas are challenging for core-guided MaxSAT
algorithms as they often suffer from making too many
iterations~\ref{ln:iloop}--\ref{ln:record} due to the effect of
\emph{clause weight splitting}.
As a result, a number of heuristics were proposed to address this
problem, namely \emph{Boolean lexicographic optimization}
(BLO)~\cite{msagl-amai11} and \emph{diversity-based
stratification}~\cite{ansotegui-cp13}.
Both techniques aim at splitting the set of 
weighted literal objectives into multiple
levels, based on their weight, and solving a series of MaxSAT
problems, each time adding into consideration the clauses of the next
(smaller weighted) level.
Although powerful in general, both techniques 
often fail in the setting of TE explainability. \footnote{The BLO
condition is too strict while diversity-based stratification fails
because the diversity of the weights of literal objectives is too high,
i.e.\ each literal often has a \emph{unique} real-valued weight.}

As a result, we developed an alternative heuristic to stratify 
weighted literal objectives referred to as \emph{distance-based stratification}.
The idea is to (1) construct a stratum $\fml{L}\subseteq\fml{S}$ by
traversing the weights from highest to lowest and (2) associate weight
$w$ with the current stratum $\fml{L}$ if it is closer to the mean of
clause weights already in $\fml{L}$ than to the mean of clause weights
smaller than $w$, i.e. when
$$
\left|w-\frac{
  \sum\nolimits_{(c_i,w_i)\in \fml{L}}w_i
}
{
  \left|\fml{L}\right|
}\right|
<
\left|w-\frac{
  \sum\nolimits_{(c_j,w_j)\in \fml{S}\land w_j<w}w_j
}
{
  \left|\{(c_j,w_j)\in\fml{S}\land w_j<w \}\right|
}
\right|
$$

\begin{example} \label{ex:strat}
  Consider the set $\fml{S}_{1,3}$ of weighted literal objectives from
  Example~\ref{ex:opt}.
  Diversity-based heuristic~\cite{ansotegui-cp13} fails to stratify
  them since all of them have unique weights.
  Distance-based stratification makes 3 strata:
  $\fml{L}_1=\{(\bb{R_5},0.72452), (\neg\bb{R_1},0.72284)\}$,
  $\fml{L}_2=\{(\neg\bb{R_3},0.41645),$ $(\bb{R_2},0.41527)\}$, and
  $\fml{L}_3=\{(\bb{R_4},0.16249)\}$.
  \qed
\end{example}

\subsubsection{Early termination.} \label{sec:eterm}
Although core-guided MaxSAT solvers are quite effective in practice,
the proposed approach may invoke a MaxSAT engine a significant number
of times.
As solving optimization problems to optimality is often expensive, it
makes approximate solutions to these problems a viable and efficient
alternative.
Motivated by these observations, each MaxSAT call in the proposed
approach can be terminated \emph{before} an exact optimal solution is
computed.

Consider again Algorithm~\ref{alg:imxs} and recall that the objective
function in each MaxSAT call is of the form $\sum_{c'}-\sum_{c}$.
This means that as soon as the \textsf{\small\slshape cost} of the
solution is updated in line~\ref{ln:incr1} or \ref{ln:incr2} and its
new value exceeds the maximum value of $\sum_{c'}$, the optimal value
of the objective function is guaranteed to be negative, i.e.\ the
entailment~\eqref{eq:axp} holds.
In this case, the MaxSAT solver can terminate immediately.
On the other hand, recall that our solver applies stratification to
handle weighed literal objectives efficiently.
Therefore, each stratification level once solved is finished with the
working formula $\phi$ being satisfiable (this corresponds to
line~\ref{ln:imodel} in Algorithm~\ref{alg:imxs}).
Its satisfying assignment can serve to calculate an
under-approximation of the value of the objective function.
If the approximate solution is non-negative, the MaxSAT solver can
terminate immediately because this solution is evidence that the
entailment~\eqref{eq:axp} does not hold.

\subsubsection{Multi-class disjunction entailment check.} \label{sec:vsids}
The entailment check performs a disjunctive test over all opponent classes
$c' \neq c$, that is, over all objective functions $\fml{S}_{c,c'}$.
If any such objective evaluates to a negative value, the entailment holds.
This procedure is implemented in Algorithm~\ref{alg:call} by sequentially
traversing the set of objectives $\fml{S}$.
Consequently, the traversal order is relevant, as detecting a negative
objective $\fml{S}_{c,c'}$ early allows one to reduce the number of calls
to the MaxSAT oracle.

Moreover, when using the MaxSAT incremental mode, detecting
a negative objective $\fml{S}_{c,c'}$ at an early stage --- ideally one that was
already considered in the previous iteration --- enables the solver to reuse
previous decision assignments; 
thereby avoiding the solver restarting search from scratch, 
instead it adjusts the current assignment by flipping a few 
literals $\bb{x_i = v_i}$ for $i \in \fml{X}$ and re-evaluates the objective
value. This behavior is particularly beneficial for CXp generation, where 
the number of positive entailment checks (i.e., cases in which
$\fml{S}_{c,c'} < 0$, so the outcome is true) increases compared to AXp computation, and 
where adversarial points encountered during CXp construction tend to be close 
to each other and induce the same opponent class $c'$.

Based on these observations, we propose: (i) setting preferred variable
polarities, and (ii) defining a heuristic traversal order over adversarial 
classes $c' \in \fml{K}$, $c' \neq c$.
At each iteration, when a new feature $i \in \fml{F}$ is selected for testing,
if the entailment check succeeds, we reuse the Boolean assignment returned by
the MaxSAT solver for the subsequent iteration and set the variable polarities
of the formula $\phi$ accordingly.
Additionally, we propose  a \emph{VSIDS-like} heuristic~\cite{Malik-dac01} for class  
ordering, inspired by the SAT-solving insight that entities involved in recent successful 
or conflicting events are more likely to remain relevant and should therefore be prioritized, 
with decay enabling the strategy to adapt over time.
Concretely, the strategy maintains a binary heap over classes with activity scores 
that are incremented when a class wins and periodically decayed after 
N calls, so that frequently winning classes are promoted while old winners  
are gradually deprioritized.

Our empirical analysis for computing CXps shows that  adversarial points 
encountered during the explanation extraction are concentrated in a small number of classes  
of $\fml{K}$ (e.g., two to three classes) and appear in long contiguous sequences belonging 
to the same class. This observation suggests that a simpler, naïve heuristic --- such as moving 
the most recently winning class to the front of the traversal --- may already be sufficient 
to significantly limit the number of oracle calls. 
By contrast, a VSIDS-like heuristic becomes more advantageous in benchmarks with many potential 
winners whose relative importance changes dynamically over time.

\section{Smallest Explanations}

\subsection{Computing Smallest AXps}
We base our work on the (state of the art) algorithm introduced in~\cite{iplms-cp15} 
for computing smallest MUSes, as well as on the results of~\cite{inams-aiia20} 
establishing about the correspondence  between AXps and MUSes, and 
CXps and MCSes.  
%
%
Subsequently, we exploit this relationship to adapt the approach 
of~\cite{iplms-cp15} for computing minimum-size AXps.
Intuitively, the algorithm for extracting the smallest AXp is grounded 
in the duality between AXps and CXps.  
As aforementioned early, by duality every AXp can be viewed as a minimal 
hitting set of all CXps, and conversely, each CXp corresponds to a hitting set of AXps. 
This relationship allows us to compute the smallest AXp by iteratively constructing 
\emph{minimum} hitting sets (smallest size MHS)~\cite{imms-ecai16} over the collection 
of CXps discovered so far.
Note that a variant of the approach leveraging minimal hitting set (MHS) 
iterative calls allows to enumerate CXps/AXps~\cite{inams-aiia20}.

The procedure begins with no CXps available. At each step, the algorithm 
computes a minimum hitting set over the current set of CXps and 
then considers the subset of hypotheses identified by this hitting set. If this subset forms 
a valid Weak AXp, then it is guaranteed to be the smallest AXp explanation, and 
the process terminates. 
If the subset is not a Weak AXp, the algorithm expands it into a maximal non-explanatory 
subset, whose complement is by definition a new (Weak) CXp candidate. This  
(Weak) CXp candidate is blocked, and the process repeats. Because 
the hitting set is always of minimum size, the first time that a candidate subset 
is found to be a valid explanation, it is necessarily the smallest one.

\subsection{Computing Smallest CXps}
Computing smallest (or minimum-size) CXp can be viewed 
as maximizing/minizing a solution in constrained optimization 
problems (MaxSAT, MaxSMT, etc).
Concretely, given a data input $\mbf{v}\in\mbb{F}$ s.t.\ $\kappa(\mbf{v})=c$, we construct   
an optimization problem for every class $c' \in \fml{K},\, c' \not= c$: 
  \begin{align}
    \text{minimize} & & \sum\nolimits_{i\in\fml{F}}{\neg\bb{x_i=v_i}}  \label{eq:obj2} \\
    \text{subject to} & & \fml{H} \,  \land \, \fml{S}_{c,c'} > 0 \label{eq:hard2} 
  \end{align}
where, as aforementioned earlier, 
$\fml{H}$ encodes TE structure, and 
$\fml{S}_{c,c'}=\sum\nolimits_{c'} - \sum\nolimits_{c} > 0$  enforces 
the opponent class $c'\in\fml{K}$, $c'\not=c$ to be predicted,
and  the objective function~\cref{eq:obj2} targets minimality of   
the set of features to be freed $\fml{Y} \subseteq \fml{F}$ 
(resp.\ features $(\fml{F}\setminus\fml{Y})$ to be fixed).

Observe that for the case $K=2$, one has to deal with one optimization 
problem such that \eqref{eq:hard2} corresponds to $\fml{H} \land (\fml{S}_{1,2} \geq 0)$ or 
$\fml{H} \land (\fml{S}_{2,1} > 0)$. In contrast, for the case of $K > 2$, one has 
to deal with $K - 1$ optimization problems, i.e. 
for each class $c'\in\fml{K}\setminus\set{c}$. Therefore, to find the smallest CXp, 
we can either construct $(K - 1)$ MaxSAT formulations, or combine them to a single 
large formula using auxiliary variables. 
In the former option, we compute for each problem the optimal solution and 
 then smallest one is picked as the final CXp; 
This can be done in parallel, or in sequential. In sequential setting we can add 
an additional \emph{AtMost} constraint to force the solver to generate a solution smaller 
than the minimum encountered so far. 
In the latter option, we create auxiliary variables that picks an opposite class $c'$ 
and activate the corresponding constraint  $\sum\nolimits_{c'} - \sum\nolimits_{c} > 0$.  

\begin{proposition}
    Let $\kappa(\mbf{v}) = c_j$ be the predicted class of a multi-class model $\fml{M}$ with 
    class set  $\mathcal{K}$.
    For each opponent class $c_\iota \in \mathcal{K}, \, \iota\not=j$, let
    $\fml{Y}_{\iota}$ be a CXp solution of the optimization problem defined 
    by~\eqref{eq:obj2}–\eqref{eq:hard2}, i.e., minimizing the number of (negative) literal
    objectives $\bigwedge\nolimits_{i\in\fml{Y}} \neg\bb{x_i = v_i}$  subject to 
    the constraint $\fml{H} \land \fml{S}_{j,\iota} > 0$.
    Then the smallest such 
\[
    \fml{Y}^* = \arg\min_{\iota \in \fml{K} \setminus \{j\}}
    \, |\fml{Y}_{\iota}|
\]
    is the smallest (minimum-size) CXp  
    of the prediction $\kappa(\mbf{v}) = c_j$. 
\end{proposition}

\begin{proof}
    Let $\mbb{O}$ be the set of optimal solutions found, $|\mbb{O}| = K-1$. Assume 
    by contradiction that $\fml{Y}^* = \arg\min_{\mbb{O}} |\fml{Y}|$ IS NOT a smallest 
    size CXp, i.e. there is a $\fml{Y}'$ smaller than $\fml{Y}^*$.  
    Then, there is an adversarial point $\mbf{w}$ for  $\fml{Y}'$ such that one 
    of the opponent classes $c_\iota$ is predicted, i.e. $\kappa(\mbf{w}) = c_\iota$.
    This implies that there must be a solution for the MaxSAT problem for class $c_\iota$ 
    better than that found in $\mbb{O}$. Contradiction.
\end{proof}

Another key observation about \cref{eq:hard2} is that the inequality $\sum\nolimits_{c'} - \sum\nolimits_{c} > 0$  
can be expressed as Pseudo-Boolean (PB) constraints, therefore we can apply 
a PB solver (e.g.\  RoundingSAT~\cite{Nordstrom-ijcai18}, OR-Tools~\cite{or-tools}), or 
in a clausal-form and apply 
an off-the-shelf  MaxSAT solver (e.g.\ RC2~\cite{imms-jsat19}).
We emphasize that for the case of RFmv, we are dealing with cardinality 
constraints rather than Pseudo-Boolean constraints for the case 
of RFwv and BT. 
%
(Also note that extending the SAT encoding of RFmv 
outlined in~\cref{sec:satrf} MaxSAT --- to compute minimum 
solution --- is straightforward.)


\section{Beyond Explanation: Fairness \& Robustness Verification}
\label{sec:fair}
This section advances beyond explanation by presenting a logical encoding 
framework for verifying robustness and individual fairness in RFmv classifiers.

\subsection{Individual fairness}
We follow recent accounts of formal fairness verification~\cite{icshms-cp20}, we 
model classifiers as functions over protected and non-protected attributes. 
While~\cite{icshms-cp20} considers TEs (XGboost) and decision sets, and focuses 
on fairness through unawareness (FTU), we adopt the robust notion of individual
fairness (IF).
 
We consider a classification model 
$\mbf{M}=(\fml{F},\mbb{F},\fml{K},\kappa)$, and  $\fml{P} \subset \fml{F}$ denote 
the set of protected attributes (e.g., gender, race), and 
$\fml{F} \setminus \fml{P}$ the set of non-protected attributes. 

The model $\mbf{M}$ is said to satisfy \emph{individual fairness} (IF) w.r.t. protected 
attributes  if for all pairs 
of arbitrary instances $\mbf{x}, \mbf{y} \in \mbb{F}$, the following holds:
\begin{equation}\label{eq:IF}
    \forall (i \in \fml{F} \setminus P).\, (x_i = y_i)
    \ \limply\ 
    (\kappa(\mbf{x}) = \kappa(\mbf{y}))
\end{equation}

Namely, a model is individually fair if two individuals who are identical in all non-protected 
attributes, but may differ in protected ones, receive the same prediction. 
A violation occurs when such a pair produces different outputs, indicating potential individual bias.

\subsection{Adversarial robustness}
%
We consider a classification problem
$\mbf{M}=(\fml{F},\mbb{F},\fml{K},\kappa)$, an instance $(\mbf{v},c)$,
with $\mbf{v}\in\mbb{F}$ and $c=\kappa(\mbf{v})$, and a distance
value $\delta>0$.

We say that there exists an adversarial example (AEx) (or
alternatively that the classifier is not $\delta$-robust on
$\mbf{v}$) if the following logic statement holds true,
\begin{equation} \label{eq:locrob}
\exists(\mbf{x}\in\mbb{F}).\left(||\mbf{x}-\mbf{v}||_p\le\delta\right)\land\left(\kappa(\mbf{x})\not=\kappa(\mbf{v})\right)
\end{equation}
The logic statement asks for a point $\mbf{x}$ which is less than
$\delta$ ($p$-norm) distant from $\mbf{v}$, and such that the
prediction changes. Evidently, we could target for a specific class
change, or even consider changes within a subset of the possible
classes.
(The distance measure $l_p$  refers to as the Minkowski 
distance. Most popular $p$-norm studied in robustness verification  
are the Hamming distance $l_0$, 
 the Manhattan distance $l_1$, the Euclidean
distance $l_2$ and the Chebyshev distance $l_{\infty}$.)

\subsection{Verification Framework}
We leverage  the SAT-based encoding for explaining RFmv 
to build a formal verification framework of individual fairness and 
adversarial robustness properties in RFmv. 
Concretely, we re-use the tree ensemble encoding (i.e. trees and 
feature-domain encodings) outlined in~\cref{sec:satrf} and 
we add further constraints to express the condition 
$\kappa(\mbf{x}) \neq \kappa(\mbf{y})$ for fairness/robustness 
verification.
The current encoding requires two encoding copies of 
the ensemble $\mfrk{T}$, 
one defined on $\kappa(\mbf{x})$ and the other on $\kappa(\mbf{y})$.
%
We tag
the variables in the first replica with a $a$ subscript, and the
variables in the second replica with a $b$ subscript.
Next,  we detail the encoding of majority vote for $\kappa(\mbf{x})$ and 
then enforcing $\kappa(\mbf{x}) \ne \kappa(\mbf{y})$. Clearly, encoding 
a majority class for  $\kappa(\mbf{y})$ is not necessary since the picked 
class in $\kappa(\mbf{x})$ is known and it suffices not to have the same 
class picked in  $\kappa(\mbf{y})$ to verify the non-robust condition.
Afterwards, we present the $l_0$ norm constraining the points 
$\mbf{x}$ and $\mbf{y}$. 
Note that in adversarial robustness query $\mbf{x}$ is a fixed point and 
$\mbf{y}$ is the output, whereas in individual fairness query the tuple 
$(\mbf{x}, \mbf{y})$ is the outcome.

\paragraph{Encode majority class for $\kappa(\mbf{x})$.} 
Let $z_{i}^a$ and $b^a_i$ be auxiliary boolean variables 
serving to compare  a candidate class $c_j$ against 
all $c_k$,  $c_k\in\fml{K}$. Then,  the picked class $c_j$ must have 
the largest score in $\kappa(\mbf{x})$:
      \begin{align}\label{eq:k(x)}
	\bigwedge\nolimits_{k\in [K]} 
	\left(\sum\nolimits_{i=1}^{M}z_{i}^a+\sum\nolimits_{i=1}^{M}\neg{p^a_{ik}} + b^a_i\ge{M+1}\right) &
	\nonumber \\
	\bigwedge\nolimits_{j\in [K]} \left(\bigwedge\nolimits_{k\in[j]} s^a_j\limply\neg{b^a_{j-k}} \right), &
	~~~~ c_j \prec c_{j-k}
	\nonumber \\
	\bigwedge\nolimits_{j\in [K]} \left( s^a_j\limply{b^a_{K-j}} \right), &
	~~~~  c_j \succ c_{K-j}
	\nonumber \\
	s^a_j\limply(z_i^a\leftrightarrow{p^a_{ij}})  
	\wedge
	 s^a_j\limply(w_i^b\leftrightarrow{p^b_{ij}}) 
	 & 
	\nonumber \\
	\bigwedge\nolimits_{j\in [K]} s^a_j\limply(\bigvee\nolimits_{k\ne j} s^b_k) 
	\wedge
	\left( \bigwedge\nolimits_{j\in [K]} s^a_j\limply\neg{ s^b_j} \right) 
	\wedge		
	\left( \sum\nolimits_{n=1}^{K} s^a_n=1 \right)  &
      \end{align}
Observe that $b^a_j = 1$ iff $c_k \prec c_j$ and $ s^a_j = 1$. Thus, for the case 
where $c_k \succ c_j$ the cardinality constraint enforces the inequality
$\left( \sum\nolimits_{i=1}^{M}p_{i,k}^a > \sum\nolimits_{i=1}^{M}\neg{p^a_{ij}} \right)$, 
then RHS equals to $M+1$,
while for the case of  $c_k \prec c_j$ it enforces the inequality
$\left( \sum\nolimits_{i=1}^{M}p_{i,k}^a \ge \sum\nolimits_{i=1}^{M}\neg{p^a_{ij}} \right)$
and the RHS is equals to $M$.
Moreover,  $s^a_j\limply(z_i^a\leftrightarrow{p^a_{ij}}) $ serves to pick a class $c_j$ 
in $\kappa(\mbf{x})$ and checks if it is the majority vote.
In that case, the same class must not be predicted in the second 
copy  $\kappa(\mbf{y})$ of $\mfrk{T}$.

\paragraph{Encode picking a different class in $\kappa(\mbf{y})$.} 
We encode  $\kappa(\mbf{y})$ such that it picks a different prediction
from  $\kappa(\mbf{x})$.
Let    $z_{i}^b$ and $w_{i}^b$ be auxiliary boolean variables serving 
to associate, resp.,  a number of votes for $c_k$ of a candidate class $c_k$ and 
number of votes for $c_j$ such that $c_j$ is  the class majority in  $\kappa(\mbf{x})$. 
Hence, the main idea is  to enforce the number of votes 
for $c_j$ not to be the largest one for $\kappa(\mbf{y})$:
      \begin{align}\label{eq:k(y)}
	\sum\nolimits_{i=1}^{M}z_{i}^b+\sum\nolimits_{i=1}^{M}\neg{w^b_i} + b^b\ge{M+1} 
	\nonumber \\
	\bigwedge\nolimits_{j\in [K]} \left(\bigwedge\nolimits_{k\in[j]} s^b_j\wedge s^a_{j-k} \limply\neg{b^b} \right)
	\nonumber \\
	\bigwedge\nolimits_{j\in [K]} \left( s^b_j\wedge s^a_{K-j}\limply{b^b} \right)
	\wedge	
	 s^b_k\limply(z_i^b\leftrightarrow{p^b_{ik}}) 
	\nonumber \\
	\bigwedge\nolimits_{j\in [K]} s^b_j\limply(\bigvee\nolimits_{k\ne j} s^a_k)
	\wedge		
	\left( \sum\nolimits_{n=1}^{K} s^b_n=1 \right) 
      \end{align}
Observe that  $s^b_k\limply(z_i^{b}\leftrightarrow{p^b_{ik}})$, $k\in[K]$ serves 
to pick a candidate class $c_k$ in $\kappa(\mbf{y})$ to compare its score with  
the picked class in $\kappa(\mbf{x})$.
Furthermore,  $s^a_j\limply(w_i^{b}\leftrightarrow{p^b_{ij}})$ serves to activate 
the associated variables $p^b_{ij}$ of the predicted class $c_j$ by $\kappa(\mbf{x})$ 
in~\cref{eq:k(x)} using the auxiliary variables $w_i^b$.
(Notice that in~\cref{eq:k(x)}, we use $M$ auxiliary variables $b^a_i$, whilst 
for  $\kappa(\mbf{y})$ (i.e.\ \cref{eq:k(y)}) we only need 
a single variable $b^b$ as we have one cardinality constraint to satisfy.)

Note that in case of binary classification, one can optimize the outlined 
encoding and obtain a simplified encoding that applies only 
two cardinality constraints.
\paragraph{Special case for binary class ($|\fml{K}| = 2$).}
 Encoding $\kappa(\mbf{x}) \neq \kappa(\mbf{y})$ for binary classifier 
($\fml{K} = \set{c_1, c_2}$), can be achieved using two inequalities:
 \[
     \left(\sum\nolimits_{i=1}^{M} p_{i}^a  \ge\left\lfloor\frac{M}{2}\right\rfloor \right) 
     \wedge
     \left(\sum\nolimits_{i=1}^{M} \neg{p_{i}^b}   \ge\left\lfloor\frac{M}{2}\right\rfloor+1 \right)
 \]
	   Observe that for $\kappa(\mbf{x})$ class $c_1$ is enforced to be picked 
	   and for $\kappa(\mbf{y})$ class $c_2$ is picked. Also, note that for binary class 
	   it is sufficient to use one Boolean variable $p_i$ to represent the predicted 
	   class in tree $\fml{T}_i$. Thus $p_i = 1$ when $c_1$ is predicted otherwise 
	   $c_2$ is picked and $p_i = 0$.

\paragraph{Encode $l_0$-distance.}
 Next we show how we encode  $||\mbf{x}-\mbf{y}||_0\le\delta$. 
%
	\begin{align*}
	    \bigwedge\nolimits_{i\in\fml{F}} \left( \bigwedge\nolimits_{I^i_k\in \mbb{D}_i} 
	     (\bb{x_i\in I^i_k}^a \wedge \bb{y_i\in I^i_k}^b) \limply \bb{x_i = y_i} \right)
	   \nonumber \\[3pt]
	   \wedge
	   \left( (\bb{x_i\in I^i_k}^a \wedge \neg \bb{y_i\in I^i_k}^b) \limply \neg \bb{x_i = y_i} \right)
	   \nonumber \\[3pt] 
	   \wedge
	   \left( \sum\nolimits_{i\in\fml{F}} \bb{x_i = y_i} \ge (m - \delta) \right)
	\end{align*}
Observe that it suffices to have $\bb{x_i\in I^i_k}^a=1$ and $\bb{y_i\in I^i_k}^b=0$ to propagate  
$\neg \bb{x_i = y_i}$, since that exactly one  $\bb{x_i\in I^i_k}^a$ can be true for each 
feature $i$ and there can not be another pair $(\bb{x_i\in I^i_l}^a \wedge \bb{y_i\in I^i_l}^b)$ true.
Moreover, for fairness verification we add the condition on unprotected 
features: $ \bigwedge\nolimits_{i\in\fml{F}\setminus\fml{P}}  \bb{x_i = y_i}.$	   	 
	 

\section{Experiments} \label{sec:res}
This section presents a summary of empirical assessment
of the different encodings: SAT, MaxSAT and SMT that we proposed
in this work to compute abductive and contrastive
explanations for TEs
(i.e.\ RFmv, RFwv and BT) 
trained on some of the widely studied datasets.

\subsection{Experimental setup}
The experiments are conducted on a MacBook Pro with a Dual-Core Intel
Core~i5 2.3GHz CPU with 8GByte RAM running macOS Montery.
The time limit for each test (instance data) is fixed to 300 seconds for computing  
one AXp/CXp per sample (resp.\  600 seconds for minimum CXp) across all 
tree ensemble benchmarks.

\paragraph{Benchmarks.}
The assessment is performed on a selection of 31 publicly available
datasets, which originate from UCI Machine Learning Repository
\cite{Dua2019} and Penn Machine Learning Benchmarks~\cite{Olson2017PMLB}.
Benchmarks comprise  binary and multidimensional classification datasets.
The number of classes in the benchmark suite varies from 2 to 11.
The number of features (resp.~data samples)  varies from 4 to 64
(106 to 58000, resp.) with the average being 21.9 (resp.~5131.09).
When training TEs for the selected datasets, we used 80\% of
the dataset instances (20\% used for test data). For assessing
explanation tools, we randomly picked fractions of the dataset,
depending on the dataset size.
%
The number of trees in each learned model is set to 100 for RFmv
and $50\times K$ for RFwv and BTs, while tree depth
varies between 3 and 6.
Concretely, BTs are trained with  XGBoost~\cite{guestrin-kdd16a} having
50 trees per class, each of depth at most 3--4; RFs  are
trained with scikit-learn~\cite{scikitlearn}  having 100 trees in the case of
RFmv and 50 tree per class in the case of RFwv, each tree of depth
at most 3--6.
(Note that we have tested other values for
the number of trees ranging from 30 to 200, and we fixed it to 100
for RFmv (resp. 50 for RFwv and BTs) since with 100 (resp. 50) trees
RFmv (resp. RFwv and BT) models achieve the best
train/test accuracies.)
As a result, the accuracy of the trained models varies between 74\%
to 100\% .
Besides, our formal explainers are set to compute a single or all  AXp/CXp
per data instance from the selected set of instances.
Note that they deal with the same data instances, and compute the
same explanation given an instance.
%
%
%

\paragraph{Prototypes implementation.}
%
%
We developed a Python prototype for reasoning about TEs, which makes 
extensive use of the latest versions
of the PySMT and PySAT toolkits~\cite{gm-smt15,itk-sat24}.
The source code implements a MaxSAT-based explainer\footnote{Available at
\url{https://github.com/alexeyignatiev/xreason/}.} for BT and 
RFwv, as well as a SAT-based explainer\footnote{Available at
\url{https://github.com/izzayacine/RFxpl/}.} tailored for RFmv.
In addition, an SMT-based explainer~\cite{inms-corr19} is provided as a baseline 
for comparison.
In order to learn RFmv models,
we have overridden the implemented RF learner
in scikit-learn so that it reflects  
the original algorithm described in \cite{Breiman01}.\footnote{%
The RF model implemented by scikit-learn uses probability estimates
to predict a class, whereas in the original proposal for
RFs~\cite{Breiman01}, the prediction is based on majority vote.}
Furthermore, PySAT \cite{imms-sat18,itk-sat24} is used to instrument incremental
SAT oracle calls.
As a result, for the case of computing minimal CXp (resp.\ smallest CXp), 
we directly invoke  
the provided algorithm LBX~\cite{mpms-ijcai15} for MCS 
(resp.\ RC2~\cite{imms-jsat19} for smallest MCS/MaxSAT) 
 in PySAT. 
The prototype of the MaxSAT-based explainer build on RC2 and 
augments it with all the proposed heuristics. In turns, RC2 relies on 
state of the art SAT engines, including  
Glucose~3~\cite{audemard-sat13} and MiniSAT~\cite{een-entcs03,minisat-gh} as 
underlying oracle solvers.
The SMT-based counterpart taken directly from~\cite{inms-corr19} and
uses Z3~\cite{bjorner-tacas08} as the underlying SMT engine.
(Note that we additionally tested CPLEX~\cite{cplex2009v12} as an oracle for the original
SMT encoding. As it was significantly outperformed by the Z3-based
solution, MIP is excluded from consideration below.)
Both competitors support the computation of one AXp and also their
enumeration~\cite{inams-aiia20}.
Although the explainers support \emph{QuickXplain-like} AXp
extraction~\cite{junker-aaai04}, it did not prove helpful in our
initial results, and so the standard deletion-based
Algorithm~\ref{alg:del} was applied instead.
Regarding the CXp extraction algorithm in MaxSAT/SMT setting, we build 
on PySAT implementation 
of a mixture algorithm of deletion-based (illustrated in~\cref{alg:cxp-del}) and 
well-known CLD algorithm~\cite{mshjpb-ijcai13} to include MaxSAT and SMT 
reasoners (originally designed to support solely  SAT solvers). 
More concretely, the implementation follows the basic linear search 
for MCS extraction augmented with \emph{clause D} (CLD) oracle calls.
%
(Note that we do not include LBX algorithm in the MaxSAT-based explainer as 
its adaptation is 
not straightforward and would require substantial modifications to the internal 
mechanisms of the MaxSAT solver, which also might slow 
down the solver.)

Additionally to the aforementioned SAT-/SMT- and MaxSAT-based AXp/CXp
computation, we compared these formal rigorous explanation approaches
to the computation of a heuristic explanation with the use of
Anchor~\cite{guestrin-aaai18}.
Although heuristic explanations are known to suffer from a number of
issues~\cite{inms-corr19,nsmims-sat19,lakkaraju-aies20a,lakkaraju-aies20b},
it is still of interest to see how the proposed approach stands
against one of the best heuristic explainers from the scalability
perspective.
(Note that all Anchor's parameters are set to their default values.)
The comparison results with Anchor are summarized in the appendix.
On the other hand, another formal explainer: 
PyXAI~\cite{pyxai-ijcai24b} toolkit for explaining TEs is not included 
in the baseline comparison.  
As shown in a recent study on validating formal explainers~\cite{hiims-corr25}, PyXAI 
may deliver incorrect or redundant explanations (AXp/CXp) for RFmv and BT and 
does not support RFwv, making it unsuitable for a meaningful comparison. 
(Further details on the related works with PyXAI implementations are discussed 
in~\cref{ssec:relw})


\subsection{Results}
When comparing the runtimes of our approach against 
the formal method baseline,\footnote{%
We exclude in this section heuristic method baselines (e.g. Anchor), however 
comparison results are reported in the appendix.} we present 
most of the results 
using scatter and cactus plots instead of tables. 
These visualizations convey relative performance and 
scalability trends more clearly, allowing per-instance and 
cumulative comparisons to be easily interpreted. 
In addition, a table summarizing the characteristics of the datasets,
including the number of features $m$, the number of classes $k$, and the
tree ensemble depth, is also provided.
Tables containing the detailed numerical results 
corresponding to these plots are provided in the appendix. 
Since these tables  are large and would hinder readability, visual 
representations offer a more compact and 
accessible overview of the experimental outcomes.

Furthermore, some of the results presented below were previously reported in our earlier 
conference papers~\cite{inms-corr19,ims-ijcai21,iism-aaai22}. The present journal version substantially extends these results 
by providing new experiments (in particular AXps/CXps in RFwv and smallest CXps in TEs), additional datasets, and a unified 
analysis MaxSAT framework.

\setlength{\tabcolsep}{5pt}

\begin{table}[ht]
\centering
\resizebox{0.875\textwidth}{!}{
\begin{tabular}{
l
>{\text{(}}
S[table-format=2.0, table-space-text-pre=(]
S[table-format=2.0, table-space-text-post=)]
<{\text{)}}
S[table-format=1]
S[table-format=4.0]
S[table-format=2.0]
>{\text{(}}
S[table-format=5.0, table-space-text-pre=(]
S[table-format=5.0, table-space-text-post=)]
<{\text{)}}
S[table-format=2.1]
S[table-format=1.2]
S[table-format=2.1]
S[table-format=1.2]
c
}
\toprule[1.2pt]
\multirow{2}{*}{\bf Dataset} & 
\multicolumn{2}{c}{ }  &  
\multicolumn{3}{c}{\bf RFmv} & 
\multicolumn{2}{c}{\bf SAT Enc}   & 
\multicolumn{2}{c}{\bf AXp }  &  
\multicolumn{2}{c}{\bf CXp } &
\multirow{2}{*}{\bf AEx }  \\
  \cmidrule[0.8pt](lr{.75em}){4-6}
  \cmidrule[0.8pt](lr{.75em}){7-8}
  \cmidrule[0.8pt](lr{.75em}){9-10}
  \cmidrule[0.8pt](lr{.75em}){11-12}
& 
{\bf m} & {\bf K} & 
{\bf D}  & {\bf \#N} & {\bf \%A} & 
{\bf \#var} & {\bf \#cl} & 
{\bf Len }  & {\bf Time } & 
{\bf Len }  & {\bf Time } &  \\
\toprule[1.2pt]

ann-thyroid  &  21  &  3  &  4  &  2192  &  98  &  17854  &  29230  &  2.1  &  0.13  &  1.3  &  0.18  &  0.277 \\
appendicitis  &  7  &  2  &  6  &  1920  &  90  &  5181  &  10085  &  4.0  &  0.03  &  2.4  &  0.06  &  0.094 \\
banknote  &  4  &  2  &  5  &  2772  &  97  &  8068  &  16776  &  2.0  &  0.02  &  1.0  &  0.08  &  0.190 \\
biodegradation  &  41  &  2  &  5  &  4420  &  88  &  11007  &  23842  &  16.8  &  0.29  &  4.5  &  0.19  &  0.278 \\
ecoli  &  7  &  5  &  5  &  3860  &  90  &  20081  &  34335  &  4.0  &  0.10  &  1.2  &  0.20  &  0.384 \\
glass2  &  9  &  2  &  6  &  2966  &  90  &  7303  &  15194  &  5.0  &  0.03  &  1.7  &  0.09  &  0.187 \\
heart-c  &  13  &  2  &  5  &  3910  &  85  &  5594  &  11963  &  6.0  &  0.04  &  2.1  &  0.08  &  0.159 \\
ionosphere  &  34  &  2  &  5  &  2096  &  87  &  7174  &  14406  &  22.1  &  0.03  &  6.0  &  0.10  &  0.122 \\
iris  &  4  &  3  &  6  &  1446  &  93  &  16346  &  25603  &  2.0  &  0.03  &  1.3  &  0.13  &  0.212 \\
karhunen  &  64  &  10  &  5  &  6198  &  91  &  36708  &  70224  &  35.2  &  2.78  &  10.6  &  0.87  &  1.152 \\
magic  &  10  &  2  &  6  &  9840  &  84  &  29530  &  66776  &  6.0  &  0.14  &  1.6  &  0.41  &  0.654 \\
mofn-3-7-10  &  10  &  2  &  6  &  8776  &  92  &  2926  &  8646  &  3.0  &  0.01  &  1.9  &  0.07  &  0.259 \\
new-thyroid  &  5  &  3  &  5  &  1766  &  100  &  17443  &  28134  &  3.0  &  0.05  &  1.5  &  0.17  &  0.262 \\
pendigits  &  16  &  10  &  6  &  12004  &  95  &  30522  &  59922  &  9.9  &  0.94  &  1.9  &  0.53  &  1.093 \\
phoneme  &  5  &  2  &  6  &  8962  &  82  &  21899  &  49840  &  3.0  &  0.09  &  1.7  &  0.29  &  0.550 \\
ring  &  20  &  2  &  6  &  6188  &  89  &  19114  &  42362  &  11.0  &  0.25  &  2.2  &  0.28  &  0.392 \\
segmentation  &  19  &  7  &  4  &  1966  &  90  &  21288  &  35381  &  8.0  &  0.31  &  3.3  &  0.31  &  0.725 \\
shuttle  &  9  &  7  &  3  &  1460  &  99  &  18669  &  29478  &  2.0  &  0.14  &  2.1  &  0.21  &  0.385 \\
sonar  &  60  &  2  &  5  &  2614  &  88  &  9938  &  20537  &  36.0  &  0.09  &  7.0  &  0.15  &  0.187 \\
spambase  &  57  &  2  &  5  &  4614  &  92  &  13055  &  28284  &  18.2  &  0.11  &  3.8  &  0.21  &  0.301 \\
spectf  &  44  &  2  &  5  &  2306  &  88  &  6707  &  13449  &  19.8  &  0.07  &  6.4  &  0.10  &  0.145 \\
texture  &  40  &  11  &  5  &  5724  &  87  &  34293  &  64187  &  23.2  &  0.93  &  3.0  &  0.66  &  2.697 \\
threeOf9  &  9  &  2  &  3  &  920  &  100  &  2922  &  4710  &  1.0  &  0.01  &  1.0  &  0.03  &  0.047 \\
twonorm  &  20  &  2  &  5  &  6266  &  94  &  21198  &  46901  &  12.0  &  0.10  &  3.2  &  0.08  &  0.135 \\
vowel  &  13  &  11  &  6  &  10176  &  90  &  44523  &  88696  &  8.1  &  1.15  &  2.1  &  0.87  &  1.693 \\
waveform-21  &  21  &  3  &  5  &  6238  &  84  &  29991  &  57515  &  10.1  &  0.34  &  2.2  &  0.39  &  0.533 \\
waveform-40  &  40  &  3  &  5  &  6232  &  83  &  30438  &  58380  &  15.2  &  0.88  &  3.6  &  0.46  &  0.683 \\
wdbc  &  30  &  2  &  4  &  2028  &  96  &  7813  &  15742  &  12.0  &  0.05  &  4.6  &  0.10  &  0.156 \\
wine-recognition  &  13  &  3  &  3  &  1188  &  97  &  17718  &  28421  &  4.9  &  0.07  &  2.1  &  0.16  &  0.261 \\
wpbc  &  33  &  2  &  5  &  2432  &  76  &  9078  &  18675  &  20.1  &  0.65  &  8.2  &  0.14  &  0.190 \\
xd6  &  9  &  2  &  6  &  8288  &  100  &  2922  &  8394  &  3.0  &  0.01  &  1.5  &  0.06  &  0.274 \\

\bottomrule[1.2pt]
\end{tabular}
}
\caption{%
\footnotesize{%
  Detailed performance evaluation of the SAT encoding for computing 
  AXp's and CXp's for RFmv. 
The table shows results for 31 datasets, i.e.\ those for which test
data accuracy is no less than 76\%.
Columns {\bf m} and {\bf K} report, respectively, the
number of features and classes in the dataset.
Columns {\bf D}, {\bf\#N} and {\bf\%A} show, respectively, each tree's
max.~depth, total number of nodes and test accuracy of an RF classifier.
Columns  {\bf \#var} and {\bf \#cl} show the number of variables and 
clauses of a CNF formula encoding an RFmv classifier along with any
instance to analyze.
Column {\bf Time}   shows the average  runtime  for extracting 
an explanation.
Column {\bf Len} reports the average explanation length
(i.e. average number of features contained in the explanations). 
%
 {\bf AEx} reports the average runtime for robustness verification 
 to adversarial examples.    
} }
\label{tab:RFs-sat}
\end{table}

\paragraph{SAT explanations for RFmv.}
\autoref{tab:RFs-sat} summarizes the results of assessing the performance 
of our SAT-based explainer (\texttt{RFxpl}) for RFmv on the selected datasets. 
(Note that the AXp results in~\autoref{tab:RFs-sat} were previously reported 
in~\cite{ims-ijcai21},  whereas the CXp results are newly reported in this extended 
work.)
As can be observed, with three exceptions, the average running time for 
computing AXps is below 1~sec.\ per instance, indicating that the approach 
scales well in practice across a variety of datasets. 
In terms of worst-case behavior, a small number of outliers can be observed, 
which is expected given that AXp computation corresponds to a 
$\ddp$-hard problem. These outliers predominantly arise on datasets with 
a larger number of classes, where the underlying entailment checks become 
more challenging. 
Moreover, we observe consistently strong performance for CXp computation: 
for all considered datasets, the average runtime remains below 1~sec.\ per 
input instance. This suggests that, despite requiring the exploration of 
opponent classes, CXp generation remains efficient in the RFmv setting. 
Overall, the results demonstrate that the SAT-based explainer is effective 
and scalable for explaining RFmv on realistic datasets.

\paragraph{SMT \& MaxSAT explanations for BTs.}
Figure~\ref{fig:perf}~(reported in~\cite{iism-aaai22}) shows the results of 
AXps assessment.
As can be observed in the cactus plot of Figure~\ref{fig:cact}, the
MaxSAT-based explainer outperforms the SMT approach by a large
margin for producing AXps.
In particular, it is able to explain each of the considered instances
with the runtime per instance varying from 0.001 to 41.843 sec.\ and
the average runtime being 3.183 sec.
The runtime of the SMT-based approach varies from 0.01 to 1481.708
sec.\ with the average runtime being 16.672 sec.
The detailed instance-by-instance performance comparison of the two
competitors is provided in the scatter plot shown in
Figure~\ref{fig:det}.
We should note that on average the MaxSAT approach is 5--10 times
faster than its SMT-based rival, which makes it a viable and scalable
alternative to the state of the art.
Another observation we made is that in general the harder the
benchmarks get, the larger the gap becomes between the performance of
the two approaches, e.g.  texture dataset has 11 classes and we observe that
the SMT method may take 24 min.\ to deliver an explanation while MaxSAT
terminates in less than 1 min.\ in all tested instances.
This may be found somewhat surprising given that the MaxSAT reasoner
has to deal with much larger encoding (in terms of the number of
clauses and variables) and has to make a much larger number of
\emph{decision} oracle calls.
Finally, the MaxSAT approach proves to be more robust in terms of the
average performance whilst the performance of the SMT-based explainer
is somewhat unstable, even when explaining the predictions made by the
same classifier for different instances of the same dataset.

\begin{figure}[ht]
  \begin{subfigure}[b]{0.6\textwidth}
    \centering
    \includegraphics[width=0.94\textwidth]{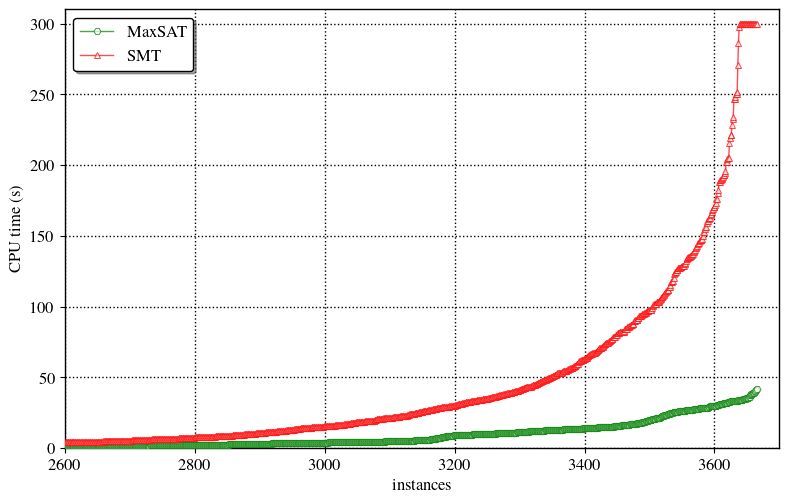}
    \caption{Raw performance}
    \label{fig:cact}
  \end{subfigure}%
  \begin{subfigure}[b]{0.4\textwidth}
    \centering
    \includegraphics[width=0.91\textwidth]{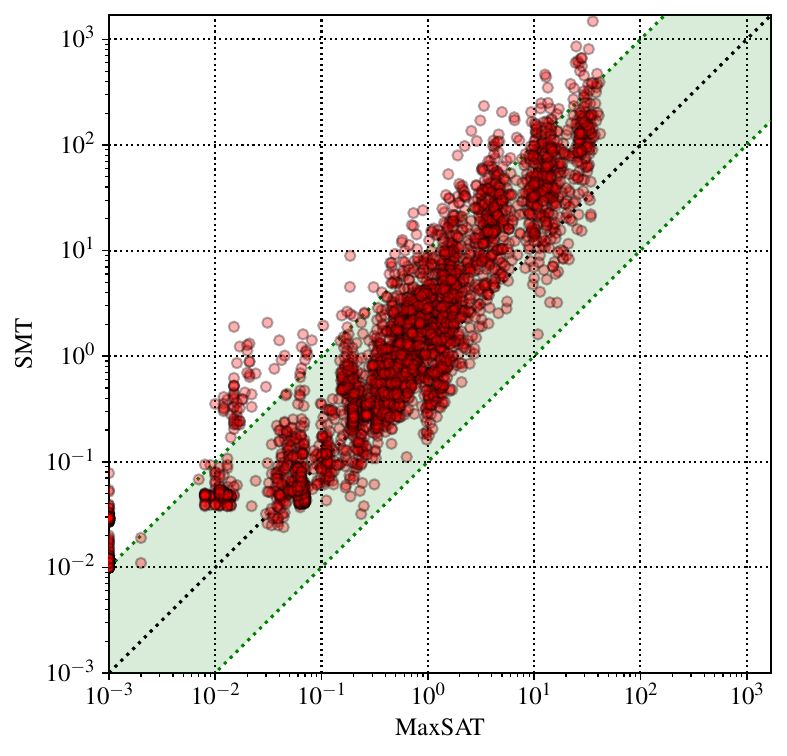}
    \caption{Detailed runtime comparison}
    \label{fig:det}
  \end{subfigure}
  \caption{Assessing  MaxSAT and SMT encodings of BTs for computing AXps.}
  \label{fig:perf}
\end{figure}

The results for CXp, summarized in~\cref{fig:BT-cxp}, reveal a more 
nuanced comparison between the MaxSAT- and SMT-based approaches.
For smaller and less complex datasets, in particular binary-class
benchmarks where CXps are often of size one, the SMT-based explainer
can be competitive or even faster.
However, as the datasets become larger and more challenging --- particularly in
multi-class settings --- the MaxSAT-based explainer demonstrates markedly better
scalability.

Such timeouts are prevalent on larger benchmarks with higher feature counts
and class cardinalities, where the MaxSAT approach consistently completes.
A key factor of this improved scalability is largely due to our combined 
multi-class-ordering and
variable-polarity heuristics (see~Sec.\ref{sec:vsids}), which allow the
incremental MaxSAT solver to reuse solver state and identify counterfactual
points more efficiently. This is especially beneficial for CXp generation,
where adversarial samples arise more frequently than in AXp computation,
reducing the cost of repeated oracle calls, i.e.\ satisfied Boolean assignments 
over formula $\phi$, s.t. the optimal values are positive. 
(Note that RC2 employs the MiniSat solver~\cite{minisat-gh} in order 
to preserve the variable-polarity heuristic; Glucose~3 is not used in this setting, as 
it interferes with this heuristic.)

Overall, while SMT performs well on simpler cases, the MaxSAT-based
approach offers superior scalability for computing CXps on complex
benchmarks.

\begin{figure}[ht]
  \begin{subfigure}[b]{0.6\textwidth}
    \centering
    \includegraphics[width=0.94\textwidth]{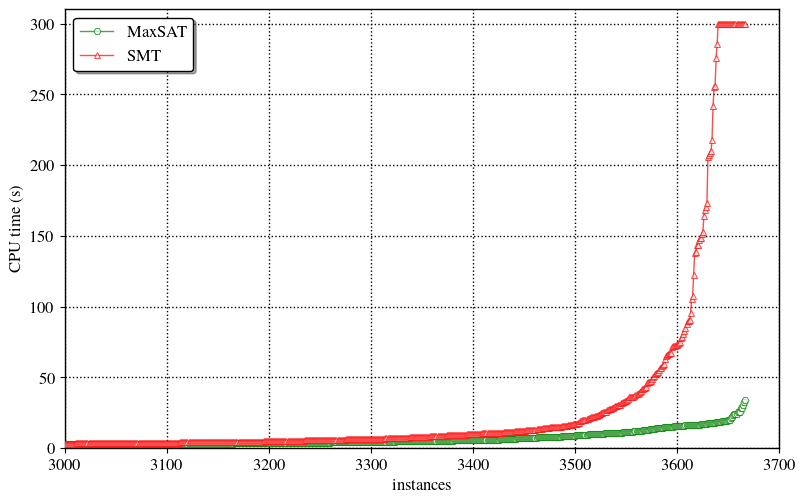}
    \caption{Raw performance} \label{fig:cactus-BT-CXp}
  \end{subfigure}%
  \begin{subfigure}[b]{0.4\textwidth}
    \centering
    \includegraphics[width=0.91\textwidth]{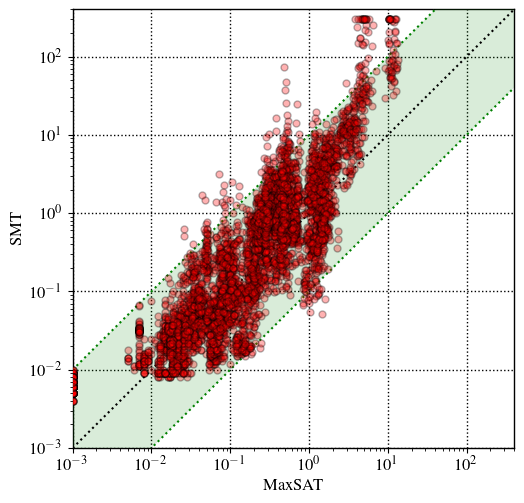}
    \caption{Detailed runtime comparison}  \label{fig:scatter-BT-CXp}
  \end{subfigure}
  \caption{Assessing  MaxSAT and SMT encodings of BTs for computing CXps.}
  \label{fig:BT-cxp}
\end{figure}


\paragraph{MaxSAT explanations for RFwv.}
Figure~\ref{fig:RFwv-mx} compares the SMT- and MaxSAT-based approaches for
computing AXps for RFwv.
The results show a clear advantage of the MaxSAT-based explainer, which
outperforms SMT on most datasets.
While SMT remains competitive on simpler instances, its performance
degrades significantly as model complexity increases.
For datasets with many classes or larger ensembles, the SMT-based approach
exhibits substantially higher runtimes and several timeout cases, whereas
MaxSAT consistently produces AXps within a minute 
(e.g.\ \textit{karhunen}, \textit{pendigits}, and \textit{texture}, runtimes
remain within practical limits i.e. below $30$~sec.).
The instance-level comparison further highlights the stability of MaxSAT,
which achieves lower and more predictable runtimes across inputs.
On average, MaxSAT is several times faster than SMT, with the performance
gap widening on the most challenging datasets.


\begin{figure}[ht]
  \begin{subfigure}[b]{0.5\textwidth}
    \centering
    \includegraphics[width=0.88\textwidth]{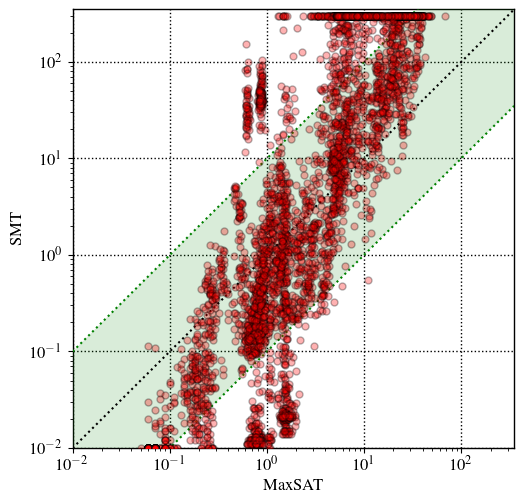}
    \caption{AXp runtimes for RFwv}
    %
    %
  \end{subfigure}%
  \begin{subfigure}[b]{0.5\textwidth}
    \centering
    \includegraphics[width=0.88\textwidth]{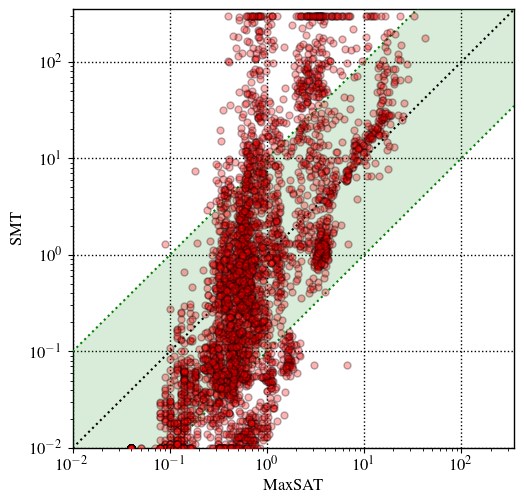}
    \caption{CXp runtimes for RFwv}
    %
    %
  \end{subfigure}
  \caption{Assessing  MaxSAT and SMT encodings of RFwv for computing AXps and CXps.}
  \label{fig:RFwv-mx}
\end{figure}

The CXp results for RFwv exhibit the same overall pattern observed earlier
for BTs.
Somewhat surprisingly, SMT performs better on CXp than AXp, as incremental 
solving allows reuse of previous satisfying assignments, whereas frequent 
unsatisfiable cores in AXp largely negate incrementality and force repeated 
solver restarts.
On smaller and less complex datasets --- especially in binary-class settings
where CXps are often of size one --- the SMT-based explainer remains competitive
and can be faster.
As the datasets become larger and more challenging, in particular in
multi-class configurations, the MaxSAT-based approach scales noticeably
better.
In these cases, SMT increasingly encounters timeout instances, whereas
MaxSAT consistently computes CXps within the allotted time.
This recurring trend across both RFwv and BTs highlights the effectiveness
of incremental MaxSAT solving combined with class-ordering and
variable-polarity heuristics.
These mechanisms allow solver state to be reused efficiently and enable
rapid identification of adversarial points, which is particularly
beneficial for CXp generation.

%
\paragraph{Ablation Study of MaxSAT Approach.}
We conduct an ablation study of our MaxSAT approach to evaluation the gain of the proposed 
incremental MaxSAT algorithm with assumptions~\cref{sec:assump} and  (depicted in~\cref{sec:cores}) and 
Reusing unsatisfiable cores~\cref{alg:imxs}, namely turning off  the incremental mode. 

We conduct an ablation study to evaluate the impact of the proposed 
incremental MaxSAT algorithm~\cref{alg:imxs} with assumptions~\cref{sec:assump} 
and the reuse of unsatisfiable cores~\cref{sec:cores}. 
To this end, we compare the full incremental approach (\cref{alg:imxs}) with 
a non-incremental variant where the incremental mode is disabled.
\cref{fig:ablation-mxs} shows the comparison of computing AXps for BTs
when ablating the incremental mode. Overall, the incremental approach 
consistently improves performance. 
Across all evaluated instances, incremental MaxSAT is faster on 98.6\% 
instances. 
In terms of runtime, incremental mode achieves an average finite speedup 
of $2.6\times$ compared to non-incremental mode. Moreover, 75\% of the instances obtain 
at least a $2.6\times$ improvement, and 90\% of the instances 
achieve more than $3.5\times$ speedup.
These results confirm that enabling incremental MaxSAT substantially improves efficiency, as 
it benefits from reusing previously extracted unsatisfiable cores and learnt clauses.

Furthermore, we ablate the heuristic described in \cref{sec:vsids} on multi-class
datasets to assess its impact on CXp computation. 
As can be see from \cref{fig:ablation-vsids}, overall, the heuristic significantly 
improves performance.
The largest improvements are observed on the \emph{waveform-40} and 
\emph{waveform-21} datasets, where our heuristic reduces the average runtime 
from 24.36~sec to 3.54~sec ($6.9\times$ speedup) and from 19.44~sec to 2.59~sec 
($7.5\times$ speedup), respectively. Substantial gains are also observed 
on \emph{karhunen}, s.t. the average runtime decreases from 12.71~sec 
to 4.40~sec (about $2.9\times$ speedup).
For simpler datasets, absolute differences are smaller, yet the heuristic 
still yields consistent relative improvements (often around $2\times$).
Overall, these results indicate that prioritizing recently successful opponent 
classes helps detect negative objectives earlier and reduces the number 
of MaxSAT calls during CXp generation.

\begin{figure}[ht]
  \centering
  \begin{subfigure}[b]{0.45\textwidth}
    \centering
    \includegraphics[width=0.96\textwidth]{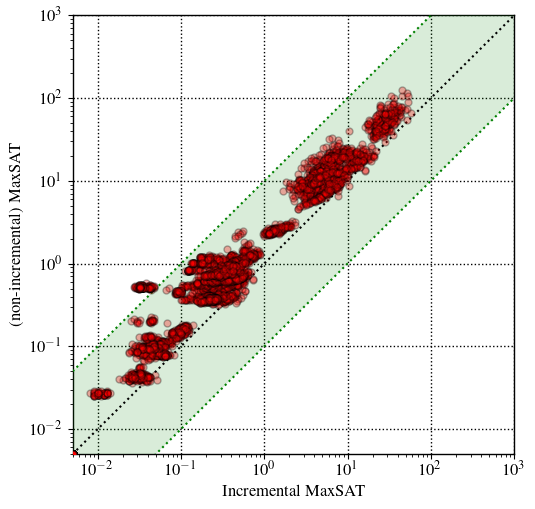}
    \caption{Ablation of incremental MaxSAT}
    \label{fig:ablation-mxs}
  \end{subfigure}
  \hfill
  \begin{subfigure}[b]{0.45\textwidth}
    \centering
    \includegraphics[width=\textwidth]{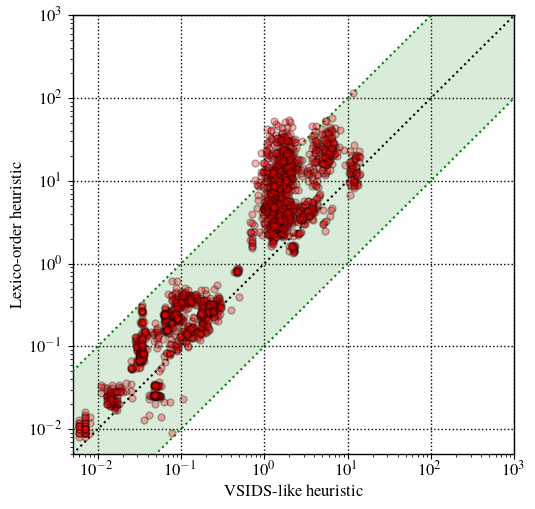}
    \caption{Ablation of Multi-class disjunction heuristic (for CXp)}
      \label{fig:ablation-vsids}
  \end{subfigure}
  \caption{Ablation study of MaxSAT Approach for BTs.}
\end{figure}

%
\paragraph{Minimum Contrastive Explanations.}
We conduct a qualitative assessment of smallest CXps to examine 
succinctness and relevance of searching optimal explanations. 
Additionally,  we evaluate the scalability of the proposed 
approach --- MaxSAT encoding and (SOTA) MaxSAT solvers 
performances.   
The comparison between CXps and minimum-CXps  
highlights a clear trade-off between runtime efficiency and 
explanation succinctness. 
(Note that we also carried out preliminary experiments on SMT encoding 
and our results clearly show that (Max)SMT-based approach does not scale 
for smallest CXp computation.)


%
\begin{figure}[ht]
  \centering
  \begin{subfigure}[b]{0.45\textwidth}
    \centering
    \includegraphics[width=0.96\textwidth]{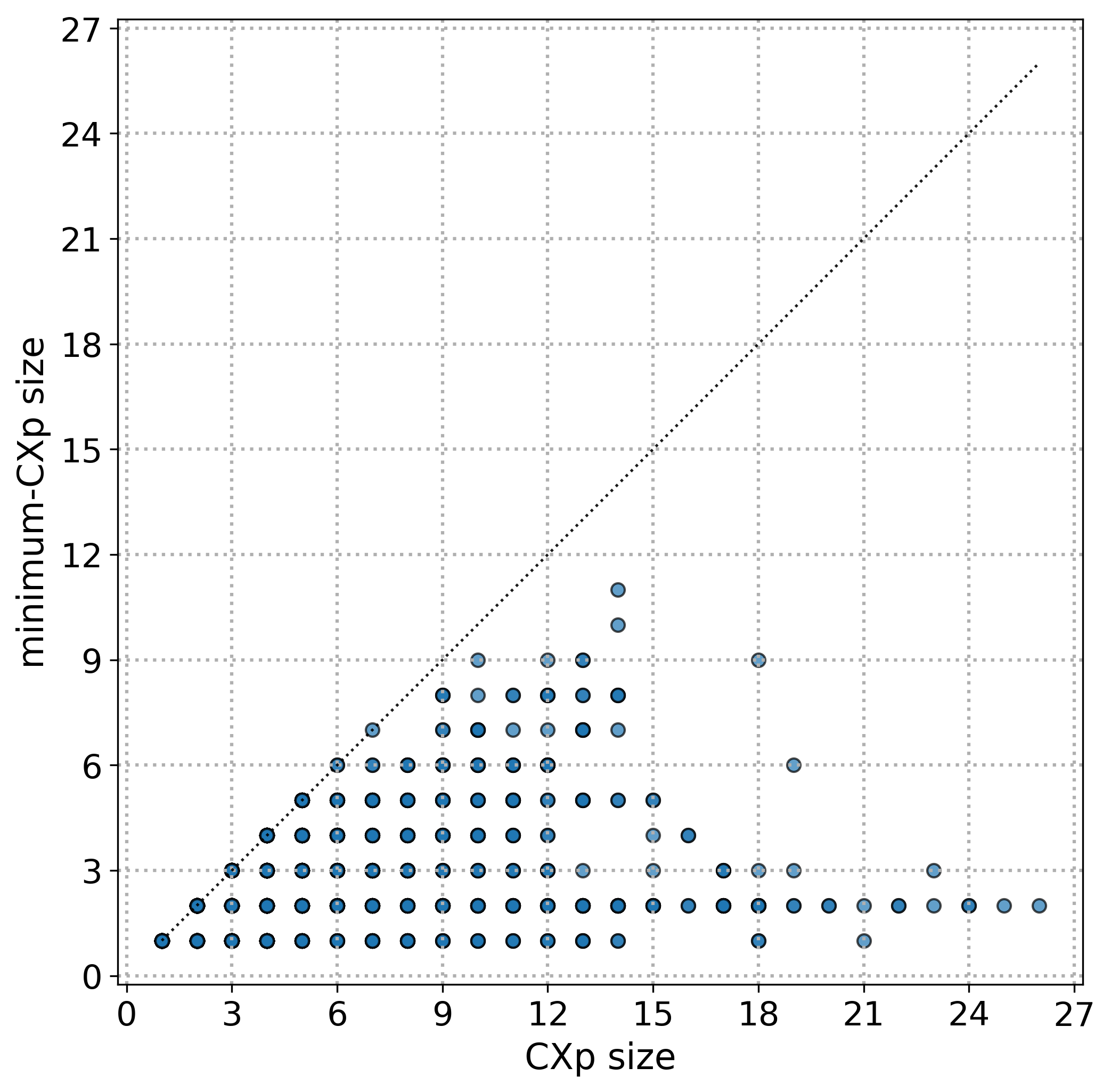}
    \caption{Detailed CXps vs. minimum-CXps length for RFmv}
    %
    %
  \end{subfigure}
  \hfill
  \begin{subfigure}[b]{0.45\textwidth}
    \centering
    \includegraphics[width=\textwidth]{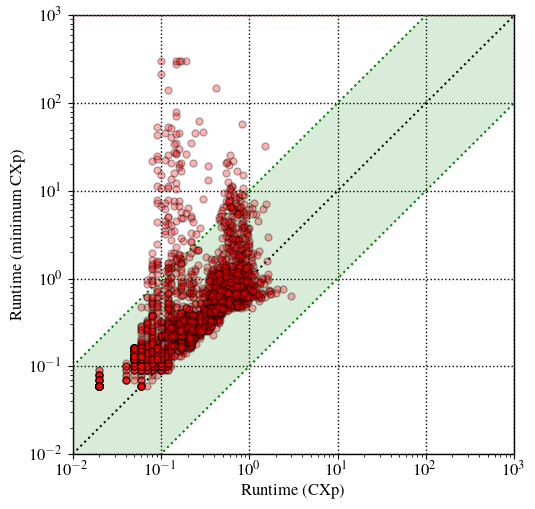}
    \caption{Detailed runtime CXps vs. minimum-CXps for RFmv}
    %
    %
  \end{subfigure}
  \caption{Assessing  minimum CXps for RFmv.}
  \label{fig:RFmv-mCXp}
\end{figure}

We begin with RFmv and report the results in~\cref{fig:RFmv-mCXp}. 
On average, computing a minimum-CXp is about $5\times$ slower than 
computing a CXp ($1.55$~sec.\ vs.\ $0.31$~sec.), reflecting the additional 
oracle calls required by the minimization process. 
This extra cost yields substantially more compact explanations: minimum-CXps 
are on average about $40\%$ smaller (mean size $1.79$ vs.\ $3$).  
Overall, the results confirm that the runtime overhead of minimum-CXp is 
compensated by improved succinctness and interpretability.

%

%
\begin{figure}[ht]
  \centering
  \begin{subfigure}[b]{0.4\textwidth}
    \centering
    \includegraphics[width=0.96\textwidth]{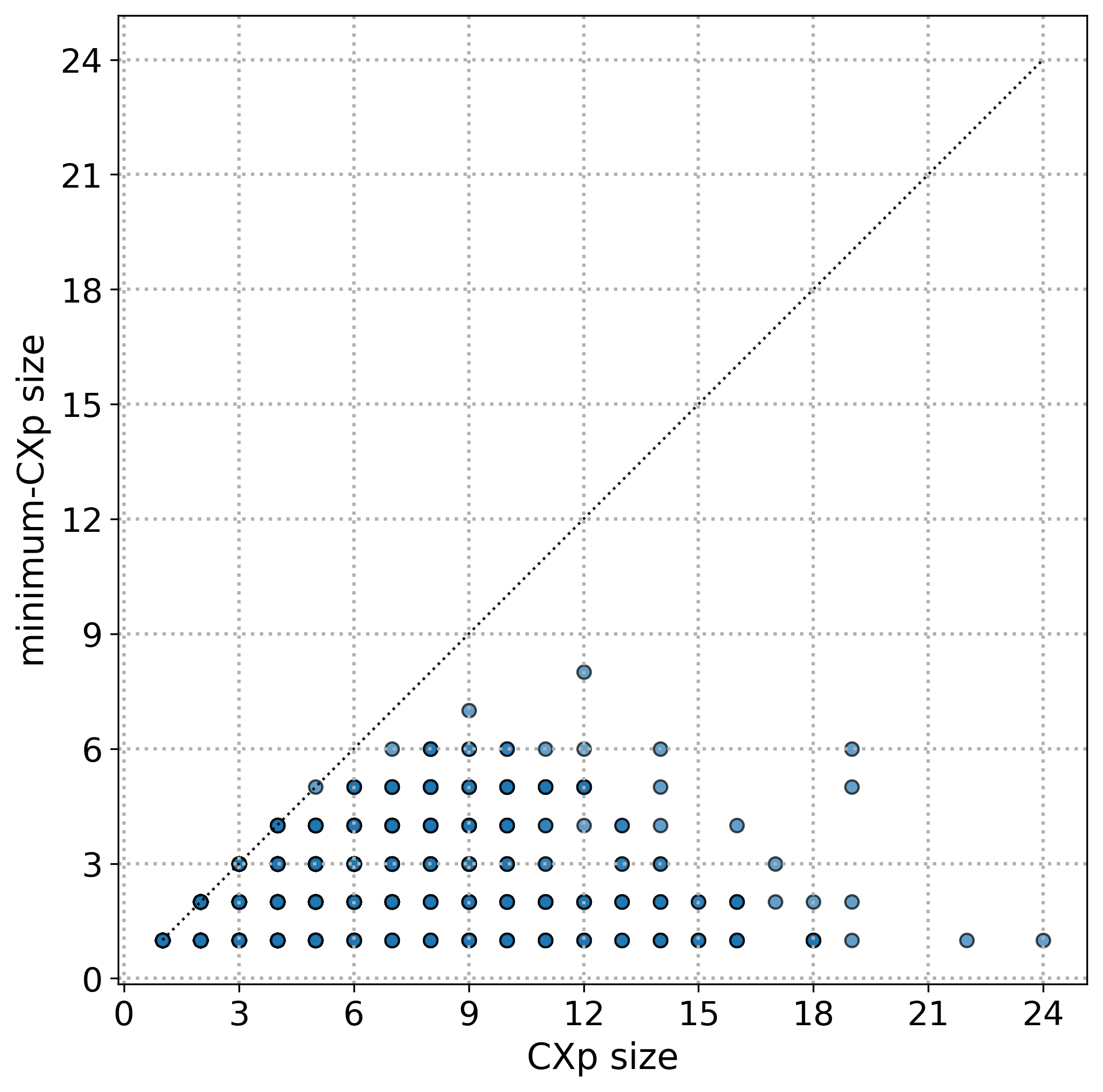}
    \caption{Detailed CXps vs. minimum-CXps length for RFwv}
  \end{subfigure}
  \hfill
  \begin{subfigure}[b]{0.4\textwidth}
    \centering
    \includegraphics[width=\textwidth]{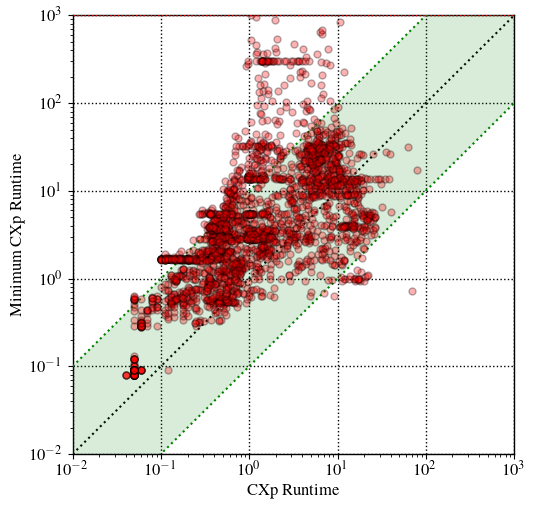}
    \caption{Detailed runtime CXps vs. minimum-CXps for RFwv}
  \end{subfigure}
  \caption{Assessing  minimum CXps for RFwv.}
  \label{fig:RFwv-mCXp}
\end{figure}

The results of smallest CXps for RFwv are shown
in~\cref{fig:RFwv-mCXp}.
As expected, enforcing minimum cardinality incurs a significant
computational overhead.
Across all datasets, CXps are computed in $3.51$~sec.\ on average
(median $1.0$~sec.), whereas minimum-CXps require $18.10$~sec.\ on average
(median $3.5$~sec.), corresponding to an overhead of roughly $5\times$ in
mean runtime.
In about $85\%$ of the instances, minimum-CXp is slower than CXp, with the
gap being particularly pronounced on challenging datasets such as
\emph{spectf} and \emph{wpbc}, where runtimes increase by almost two orders
of magnitude.
Despite this cost, minimum-CXp provides a clear and consistent benefit in
terms of explanation size.
On average, CXps contain $3.31$ features (median $2$), whereas minimum-CXps
reduce this to $1.66$ features (median $1$), i.e., roughly a $50\%$
reduction.
Over all instances, minimum-CXp is strictly smaller than CXp in $51.1\%$ of
the cases and never larger.
Moreover, minimum-CXp completely eliminates very large explanations:
while $6.3\%$ of CXps contain at least $10$ features, no minimum-CXp exceeds
this length.
The effect is particularly striking on difficult datasets such as
\emph{karhunen}, where average explanation size drops from $8.14$ to $1.46$
features and long explanations are entirely removed.
Overall, for RFwv, minimum-CXp trades additional computation time for
substantially more compact explanations, especially on the hardest
benchmarks.

\begin{figure}[ht]
  \centering
  \begin{subfigure}[b]{0.4\textwidth}
    \centering
    \includegraphics[width=0.95\textwidth]{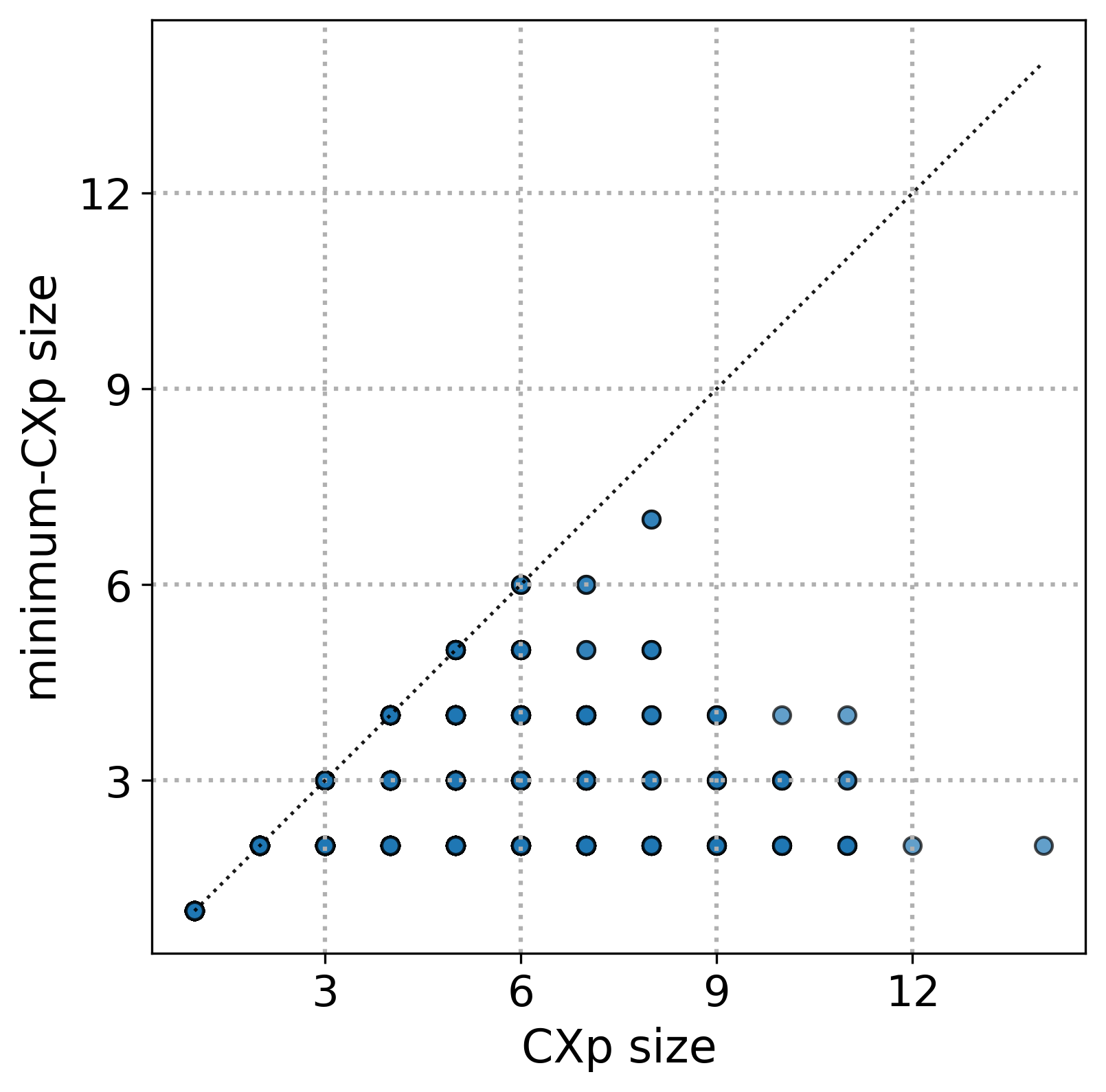}
    \caption{CXps vs. minimum-CXps length}
  \end{subfigure}
  \hfill
  \begin{subfigure}[b]{0.4\textwidth}
    \centering
    \includegraphics[width=\textwidth]{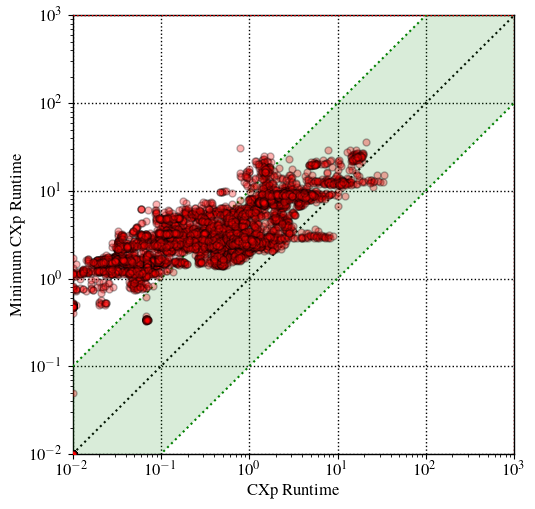}
    \caption{CXps vs. minimum-CXps runtimes}
  \end{subfigure}
  \caption{Assessing  minimum CXps in BTs.}
  \label{fig:BT-mCXp}
\end{figure}

Results for BTs are summarized in~\cref{fig:BT-mCXp}.
As with RFwv, enforcing minimality increases runtime: 
CXp requires $1.30$~sec.\ on average (median $0.40$~sec.),
whereas minimum-CXp takes $4.30$~sec.\ on average (median $2.62$~sec.),
corresponding to a $\sim3\times$ overhead.
Minimum-CXp is slower in nearly all instances, although absolute runtimes
remain modest on easier datasets.
The impact on explanation size is moderate on average but substantial for
large explanations.
Globally, the mean CXp size decreases from $1.88$ to $1.52$ features, while
CXp and minimum-CXp coincide in $82.6\%$ of the instances.
However, minimum-CXp almost entirely removes the long tail: explanations of
size $\ge 5$ drop from $7.5\%$ of instances for CXp to $0.7\%$ for
minimum-CXp, and no minimum-CXp exceeds $7$ features.
This effect is particularly visible on datasets such as \emph{karhunen},
\emph{sonar}, and \emph{wpbc}, where minimum-CXp sharply reduces both the
average and maximum explanation size.

In summary, for all TEs, minimum-CXp consistently produces much
more compact explanations at the expense of higher runtime.
While the overhead can be substantial on some benchmarks, the elimination
of long and potentially hard-to-interpret explanations highlights the
practical value of enforcing minimal cardinality, especially on challenging models.

\paragraph{Comparison of MaxSAT solvers (OR-Tools vs.\ RC2).}
We compare the performance of two MaxSAT (Pseudo-Boolean backends) solvers --- OR-Tools
and RC2 --- for computing minimum-CXps on BT and RFwv.
The corrected results are summarized in~\cref{fig:MxS-mCXp}.

\begin{figure}[ht]
  \centering
  \begin{subfigure}[b]{0.4\textwidth}
    \centering
    \includegraphics[width=\textwidth]{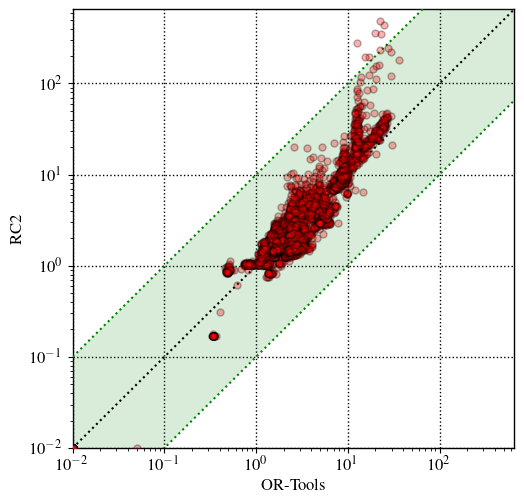}
    \caption{Detailed runtime RC2 vs. OR-Tools for BTs}
    \label{fig:BT-cxp0}
  \end{subfigure}
  \hfill
  \begin{subfigure}[b]{0.4\textwidth}
    \centering
    \includegraphics[width=\textwidth]{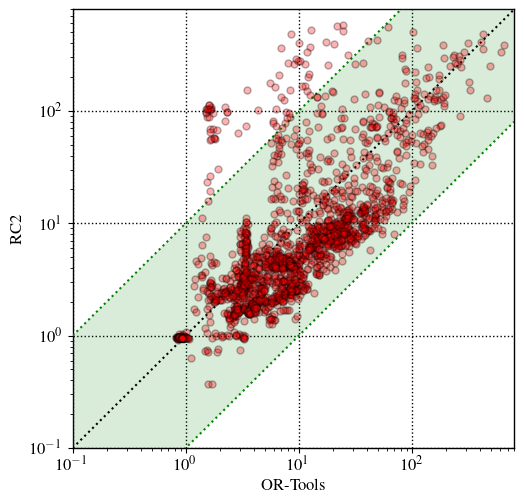}
    \caption{Detailed runtime RC2 vs. OR-Tools for RFwv}
    \label{fig:BT-cxp1}
  \end{subfigure}
  \caption{Scalability of RC2 and OR-Tools oracles for computing minimum CXp.}
  \label{fig:MxS-mCXp}
\end{figure}

For BTs, OR-Tools consistently outperforms RC2 across nearly all datasets.
Both average and median runtimes are lower for OR-Tools, and the solver
exhibits more stable behavior, particularly on harder benchmarks.
In contrast, RC2 incurs higher runtimes and greater variance, with several
instances approaching the timeout limit (600 sec.).
This performance gap becomes more pronounced as the size of the encoding
grows, indicating that OR-Tools handles the Pseudo-Boolean constraints
$\fml{S}_{c,c'}$ in~\ref{eq:hard2} more efficiently.
A similar trend is observed for RFwv.
While both solvers perform comparably on small and simple datasets,
OR-Tools scales better on larger and multi-class RFwv models.
RC2 shows a noticeable slowdown as tree depth and class cardinality
increase, whereas OR-Tools maintains consistently lower runtimes.
This suggests that OR-Tools is better suited to get the optimal solutions 
the optimization problems formulated by minimum-CXp computation.
Overall, these results indicate that OR-Tools provides superior performance
and scalability for minimum-CXp computation on both BTs and RFwv.
While RC2 remains efficient for incremental feasible solution checks (entailment   
check if the objective value is positive, namely checking the existence 
of an adversarial point $\mbf{x}$ that is not necessarily the closest to $\mbf{v}$.) 
when finding minimal explanation,  
OR-Tools emerges as the preferred backend for optimal solution search to find smallest 
CXp explanations.

\paragraph{Fairness \& robustness check.}
We evaluated individual fairness verification on four standard tabular benchmarks 
from the fairness literature --- Adult, COMPAS, German Credit, and Ricci --- spanning 
domains such as income prediction, criminal recidivism, credit scoring, and 
public-sector employment. 
Each dataset is associated with well-known protected attributes: race and gender for 
Adult and COMPAS, age and gender for German Credit, and race for Ricci.
For each dataset, we trained RFmv classifiers composed of 100 trees, varying the maximum
 tree depth to achieve the best test accuracy (see~\cref{tab:fairness-RF}). 
 %
Despite these accuracy levels, fairness outcomes vary significantly across 
datasets. The German Credit model satisfies individual fairness under 
the considered protected attributes, whereas Adult, COMPAS, and Ricci all 
exhibit individual fairness violations. Notably, Ricci remains \emph{unfair} despite 
achieving perfect predictive accuracy, underscoring that high accuracy alone does 
not guarantee fairness. 
Conversely, German Credit demonstrates that moderate accuracy 
does not preclude individual fairness.

From a computational perspective, the SAT-based verification 
approach is highly efficient. For all datasets and model configurations, individual 
fairness verification completes in negligible time --- on the order of a few 
milliseconds --- despite  feature dimensionality increase after apply 
one hot encoding on categorical data. 

The runtimes for robustness verification, reported in~\cref{tab:RFs-sat} 
(column {\bf AEx}), using the SAT-based approach (see~\cref{sec:fair})
on RFmv are consistently very low, ranging from a few tens of milliseconds 
up to roughly 2.7 sec.\ for the largest ensembles. 
As with the fairness verification experiments, this demonstrates that formal 
SAT-based reasoning scales efficiently with model size and complexity, making
 both robustness and individual fairness checks practical in the same pipeline.

\begin{table}[ht]
\centering
\resizebox{0.725\textwidth}{!}{
\begin{tabular}{
l
>{\text{(}}
S[table-format=2.0, table-space-text-pre=(]
S[table-format=1.0, table-space-text-post=)]
<{\text{)}}
l
S[table-format=1.0]
S[table-format=2.1]
>{\text{(}}
S[table-format=5.0, table-space-text-pre=(]
S[table-format=5.0, table-space-text-post=)]
<{\text{)}}
S[table-format=1.3]
l
}
\toprule[1.2pt]
\multirow{2}{*}{\bf Dataset} &
\multicolumn{3}{c}{ } &
\multicolumn{2}{c}{\bf RFmv} &
\multicolumn{2}{c}{\bf SAT Enc} &
\multicolumn{2}{c}{\bf IF} \\
\cmidrule[0.8pt](lr{.75em}){5-6}
\cmidrule[0.8pt](lr{.75em}){7-8}
\cmidrule[0.8pt](lr{.75em}){9-10}
& {\bf m} & {\bf K} & \multicolumn{1}{c}{\bf $\fml{P}$} &
{\bf D} & {\bf \%A} & {\bf \#var} & {\bf \#cl} &
{\bf Time} & {\bf Result}
\\
\toprule[1.2pt]
%
%
adult   & 12 & 2 & race, gender & 6 & 78.4 &  6374  & 18650 & 0.128 & unfair \\
compas  & 11 & 2 & race, gender & 5 & 58.8 &  6169  & 12968 & 0.073 & unfair \\
german  & 21 & 2 & age, gender  & 3 & 70.0 & 16383  & 24942 & 0.421 & fair   \\
ricci   &  5 & 2 & race         & 3 & 100 &  6233  & 11382 & 0.052 & unfair \\
\bottomrule[1.2pt]
\end{tabular}
}
\caption{%
\footnotesize{%
Fairness verification benchmarks.
Columns {\bf m} and {\bf K} report the number of features and classes.
Column {\bf $\fml{P}$} lists the protected attributes considered for
individual fairness.
Columns {\bf \#var} and {\bf \#cl} show the size of the CNF encoding.
Column {\bf Time} reports total runtime (encoding and verification).
Column {\bf Result} indicates whether the classifier satisfies individual
fairness (IF).
} }
\label{tab:fairness-RF}
\end{table}

\section{Discussion}  \label{sec:conc}

\subsection{Strengths and Limitations of (Max)SAT-based TE reasoning}

\paragraph{Reflection on Boolean Encoding.}
A notable property of the proposed propositional encoding of trees is that, while 
the number of clauses scales linearly with the size of trees in the ensemble, 
the number of literals grows exponentially. 
The MaxSAT encoding exhibits the same behaviour. This exponential growth 
in the encoding size is an inherent consequence of enumerating feature combinations 
across tree branches. Nevertheless, for common sizes of TEs in practical 
applications --- restricted to a maximum depth of 2 to 6 --- the resulting encodings 
remain of a manageable size, and the exponential blowup does not constitute 
a practical limitation.

\paragraph{Observed Scalability.}
Despite the encoding limitations noted above,
the empirical results demonstrate that
the SAT/MaxSAT framework scales effectively
across a broad range of TE configurations.
For single explanation computation (AXp/CXp),
average runtimes remain well below $30$~sec.\
across all 31 benchmarks and all three TE variants,
with the vast majority of instances resolved in under $5$~sec.
For minimum-CXp computation,
the MaxSAT optimization layer
introduces an overhead of $3$--$5\times$ over standard CXp,
yet absolute runtimes remain tractable
for the ensemble sizes and tree depths considered.
For formal verification (fairness, robustness),
the SAT-based approach completes in milliseconds,
confirming that verification admits simpler encodings
than explanation and scales accordingly.
Taken together, these results
establish MaxSAT-based reasoning as
the method of choice for rigorous TE analysis
at practical scales,
while motivating further work on
encoding schemes that mitigate the literal blowup
for deeper or larger ensemble configurations.

\paragraph{Practical Guidelines for Practitioners.}
Based on the empirical evidence,
we offer the following guidance
on when and how to apply
logic-based TE explanation.

\begin{description}
 \item [Recommended settings.]
    The proposed approach is well-suited for TEs
    trained with depths $d \leq 6$
    and up to $T = 100$ trees per class,
    which is the standard configuration
    for most off-the-shelf TE learners.
    Under these conditions,
    both AXp and CXp computation are tractable.
    Formal verification tasks
    (individual fairness, robustness)
    are significantly cheaper than explanation
    and remain practical even for models
    at the upper end of the complexity range considered.
    When explanation succinctness is a priority,
    minimum-CXp via MaxSAT is recommended:
    it reduces average explanation size by $40$--$50\%$
    at an overhead of $3$--$5\times$ in runtime,
    which remains acceptable in most deployments.

    \item [Limits of the current approach.]
    The approach faces two compounding challenges
    as TE size grows beyond the configurations studied.
    First, the exponential growth in the number of literals
    with tree depth renders the encoding
    intractable for $d \geq 8$--$10$,
    independent of the solver used.
    Second, increasing the number of trees $n$
    inflates the number of clauses proportionally,
    and increasing the number of classes $K$
    multiplies the number of MaxSAT soft constraints,
    leading to super-linear growth in solver time.

\end{description}

\subsection{Related Work} \label{ssec:relw}

\paragraph{(Weak)AXp for BTs (classifiers/regressors)~\cite{Audemard-aistats23,Audemard-ijcai23}}
Audemard et al.~\cite{Audemard-aistats23,Audemard-ijcai23} revisit the computation 
of weak AXps for BT models.
In~\cite{Audemard-aistats23}, Audemard et al. focus on BTs for classification, 
whereas their subsequent work~\cite{Audemard-ijcai23} extends this method to BTs for regression.
The underlying idea is to compute a weak AXp $\fml{X}'$ as a polynomial-time 
over-approximation of a genuine AXp $\fml{X}$.
The heuristic avoids NP-oracle calls but provides no minimality guarantee and 
may return only an approximate explanation, i.e., the full feature set $\fml{F}$.
Accordingly, the proposed procedure is best viewed as a preprocessing step prior 
to exact AXp extraction, applicable to any explainer for efficiently generating a weak AXp.
%

%
The heuristic of Audemard et al.~\cite{Audemard-aistats23} tests feature 
removal by aggregating per-tree score bounds and checking whether 
the prediction remains stable, extending this idea to multi-class settings 
via class-wise score comparisons. 
This results in an $m \times n$ or $m \times n \times K$ procedure.
Empirical evaluations show that augmenting XReason~\cite{iism-aaai22} 
with this heuristic~\cite{Audemard-aistats23} 
yields modest speedups --- typically a few seconds --- with more noticeable 
effects on large binary datasets containing thousands of binary features.

Audemard et al.~\cite{Audemard-ijcai23} adapt the over-approximation 
heuristic to regression BTs.
In fact, regression problem can be naturally reduced to a classification problem 
by discretizing the output range of regressors into disjoint intervals,
which enables the same approach for explaining classification BTs to be applied to regression BTs.
As a result, computing an explanation for a regressor BT focuses 
on ensuring the output value falls within (or outside) a desired target interval.
%
The main distinction between the contribution of~\cite{Audemard-aistats23} 
and~\cite{Audemard-ijcai23}  lies in the second phase of the methodology 
to minimize the final explanation. 
Specifically, \cite{Audemard-ijcai23} adopt a dichotomic search 
algorithm~\cite{sais-ecai06,Karimi-aistats20} 
leveraging a MILP-based encoding of the BT model.    
Notably, the paper does not assess the scalability of the MILP encoding against existing 
SOTA SMT/MaxSAT encodings. However, prior work~\cite{inms-corr19,ignatiev-ijcai20} 
has already highlighted --- aligning with our findings --- SMT typically outperforms MILP on 
tree ensembles, which inherently give rise to large numbers of Boolean variables, unlike 
models such as neural networks where MILP is more apt.
Additionally, it is worth noting that  PyXAI tool of Audemard et al.~\cite{pyxai-ijcai24b} does not 
currently implement algorithms 
for directly computing AXps for BT classifiers and regressors.

\paragraph{CXp for RFmv \& BT (classifiers/regressors)~\cite{Audemard-ecai23, Audemard-ecai24}}
The proposed work in~\cite{Audemard-ecai23} targets  minimum CXps for binary RFmv classifiers. 
We recall that computing CXps for (binary and multi-class) RFmv was first
addressed 
in~\cite{ims-ijcai21}, leveraging 
a SAT encoding.  
Given this encoding, any off-the-shelf MaxSAT solver can be directly instantiated 
to compute minimum (optimal) CXps, since the optimal MaxSAT solution corresponds 
to minimum cardinality MCS, which is well known to characterize a CXp~\cite{inams-aiia20}.
Although~\cite{ims-ijcai21} does not explicitly report minimum-size CXp, the connection 
between MCS and MaxSAT is standard in the literature.
The contribution of~\cite{Audemard-ecai23} is essentially a re-casting of this observation: the paper 
describes how MaxSAT is applied to obtain minimum CXps from the underlying SAT encoding.
In addition, the authors sketch a Marco-like enumeration~\cite{lpmms-cj16,inams-aiia20} scheme for generating 
(minimum) CXps, though the enumeration procedure itself is not implemented 
in PyXAI~\cite{pyxai-ijcai24b}.
The paper also presents an encoding of both categorical and numerical 
feature domains analogous to the one devised earlier in~\cite{iism-aaai22}.
(Note that the paper claims to employ an L$_0$ norm for categorical features and L$_1$ for 
numerical ones, nevertheless, the encoding outlined in the paper 
(as well as the implementation found in PyXAI) 
applies an L$_0$ nom (hamming distance) treatment uniformly to all feature types.)

Crucially, recent validation work~\cite{hiims-corr25} demonstrates that PyXAI~\cite{pyxai-ijcai24b}
produces incorrect AXp and CXp explanations for RFmv classifiers.
These findings suggest that the PyXAI explainers for both RFmv are affected by bugs.
Consequently, we do not use PyXAI as a baseline in our empirical evaluation.

In~\cite{Audemard-ecai24}, Audemard et al. revisit the MILP encoding 
of AXp for BTs~\cite{Audemard-ijcai23} to adapt it to CXp for regression BTs.
Unlike their prior work~\cite{Audemard-ijcai23} paper, where a dichotomic 
algorithm is implemented  to compute a minimal feature set (AXp) that falsifies 
the MILP formula, in this work the goal is to find 
the minimum solution (smallest CXp) that restore consistency in the formula.  
Hence, this is achieved by instrumenting an simple oracle call to a (any off-the-shelf) 
MILP solver to resolve the formula.
The main novelty in this work is the defined objective function of the MILP problem.
Basically, the applied objective function maximizes a weighted sum of feature-valued: 
\,$ \sum w_k \,\bb{x_i\in I^i_k}$, where $w_k \ge 1$ and represent the L$_1$ norm 
distance of numerical domain interval $I^i_k$ (and L$_0$ norm distance 
to the rest of categorical features, i.e. $w_k = 1$),  rather than a sum of literals 
such that all weights $w_k = 1$.  
In PyXAI~\cite{pyxai-ijcai24b}, users can specify custom weights for numerical 
features (L$_1$ norm); by default, all features -- whether 
interval-based or categorical -- are assigned a weight of 1 under the L$_0$ norm.

\cite{pyxai-ijcai24b} apply a MARCO-like algorithm to enumerate solutions produced 
by the MIP solver in an attempt to find lower-cost candidates. 
This procedure appears redundant, as the solver is already configured 
to compute the globally optimal solution in the initial call. Furthermore, while the paper 
defines minimal CXps, it does not explicitly discuss minimum-size explanations, despite 
the solver being instrumented to return them.

Again~\cite{hiims-corr25} reported several cases of CXps produced by PyXAI \footnote{ %
The implemented scripts responsible of generating CXps for BTs are found  
in the supplementary material of the  IJCAI-23~\cite{Audemard-ijcai23} paper.} 
for BT classifiers are either  incorrect (not WCXp) or redundant explanations 
(not minimal).
These findings suggests that the PyXAI explainer for BTs is affected by bugs 
and therefore it cannot be included in our empirical assessment baseline.

\paragraph{(Preferred) WAXp for RFmv~\cite{Audemard-aaai22, Audemard-ijcai22}}
In a related but orthogonal line of work on explaining RFmv models, 
Audemard et al.~\cite{Audemard-aaai22} tailor the SAT-based framework of~\cite{ims-ijcai21} 
to compute weak AXps for (fully) binary RFmv classifier.
Their notion of a \emph{majority reason} corresponds to a union of weak implicants 
(weak AXps) drawn from individual decision trees (DTs) in the forest, constructed 
so that the overall majority vote of the RF is preserved.
The paper also proposes to compute \emph{minimum majority reason}\footnote{%
A minimum majority reason should not be confused with a standard AXp.
It does not necessarily satisfy the strict minimality condition of an AXp.} by 
encoding the problem into a partial MaxSAT formula.   
Audemard et al.~\cite{Audemard-ijcai22} extend their work on \emph{majority reason}~\cite{Audemard-aaai22} 
adding different logical constraints expressing user preferences on weak AXp 
(or majority reason) for binary RFmv. 

\subsection{Conclusions \& Future Research}
This paper investigates the computation of rigorous, logic-based
explanations for tree ensembles, concretely for random forests and for
boosted trees. The paper considers both abductive and contrastive
explanations.
For random forests, the paper details three different approaches,
depending on how voting is instrumented. One approach is based on an
SMT encoding, another is based on a pure propositional (SAT) encoding,
and a third one is based on exploiting MaxSAT.
For boosted trees, the paper details one approach based on encoding to
SMT and another one based on MaxSAT.

The experimental results, obtained on a number of well-known datasets,
demonstrate the practical scalability of computing explanations for
fairly realistic tree ensembles. The tree ensembles considered are
such that larger ensembles would not yield improvements in prediction
accuracy. For such tree ensembles, the experiments demonstrate that
rigorous explanations are computed more efficiently than (possibly
unsound) explanations obtained with Anchor~\cite{guestrin-aaai18}.

Given that the algorithms proposed for explaining tree ensembles
involve calls to an (F)NP oracle (e.g.\ a SAT, SMT or MaxSAT oracle),
scalability could be a concern for significantly larger ensembles.
This could happen for datasets for which accuracy would require tree
ensemble sizes larger than what the methods proposed in this paper can
explain. Further optimizations, e.g.\ more optimized encodings,
breaking of symmetries, etc. is the subject of future research.

\newpage
\section*{Acknowledgments}
This work was supported in part by the National Research Foundation, 
Prime Minister’s Office, Singapore under its Campus for Research 
Excellence and Technological Enterprise (CREATE) programme,
by the Australian Research Council through the OPTIMA ITTC IC200100009, 
by the Spanish Government under
grant PID 2023-152814OB-I00, and ICREA starting funds.
The last author (JMS) also acknowledges the extra incentive provided
by the ERC in not funding this research.

\bibliographystyle{ACM-Reference-Format}

\newtoggle{mkbbl}

\settoggle{mkbbl}{false}

\iftoggle{mkbbl}{
    \bibliography{refs}

\begin{thebibliography}{91}


\ifx \showCODEN    \undefined \def \showCODEN     #1{\unskip}     \fi
\ifx \showDOI      \undefined \def \showDOI       #1{#1}\fi
\ifx \showISBNx    \undefined \def \showISBNx     #1{\unskip}     \fi
\ifx \showISBNxiii \undefined \def \showISBNxiii  #1{\unskip}     \fi
\ifx \showISSN     \undefined \def \showISSN      #1{\unskip}     \fi
\ifx \showLCCN     \undefined \def \showLCCN      #1{\unskip}     \fi
\ifx \shownote     \undefined \def \shownote      #1{#1}          \fi
\ifx \showarticletitle \undefined \def \showarticletitle #1{#1}   \fi
\ifx \showURL      \undefined \def \showURL       {\relax}        \fi
\providecommand\bibfield[2]{#2}
\providecommand\bibinfo[2]{#2}
\providecommand\natexlab[1]{#1}
\providecommand\showeprint[2][]{arXiv:#2}

\bibitem[Ans{\'{o}}tegui et~al\mbox{.}(2013b)]%
        {ansotegui-cp13}
\bibfield{author}{\bibinfo{person}{Carlos Ans{\'{o}}tegui},
  \bibinfo{person}{Maria~Luisa Bonet}, \bibinfo{person}{Joel Gab{\`{a}}s},
  {and} \bibinfo{person}{Jordi Levy}.} \bibinfo{year}{2013}\natexlab{b}.
\newblock \showarticletitle{Improving {WPM2} for (Weighted) Partial {MaxSAT}}.
  In \bibinfo{booktitle}{\emph{{CP}}}. \bibinfo{pages}{117--132}.
\newblock


\bibitem[Ans{\'{o}}tegui et~al\mbox{.}(2013a)]%
        {ansotegui-aij13}
\bibfield{author}{\bibinfo{person}{Carlos Ans{\'{o}}tegui},
  \bibinfo{person}{Maria~Luisa Bonet}, {and} \bibinfo{person}{Jordi Levy}.}
  \bibinfo{year}{2013}\natexlab{a}.
\newblock \showarticletitle{SAT-based MaxSAT algorithms}.
\newblock \bibinfo{journal}{\emph{Artif. Intell.}}  \bibinfo{volume}{196}
  (\bibinfo{year}{2013}), \bibinfo{pages}{77--105}.
\newblock


\bibitem[Ans{\'{o}}tegui and Many{\`{a}}(2004)]%
        {ansotegui-sat04}
\bibfield{author}{\bibinfo{person}{Carlos Ans{\'{o}}tegui} {and}
  \bibinfo{person}{Felip Many{\`{a}}}.} \bibinfo{year}{2004}\natexlab{}.
\newblock \showarticletitle{Mapping Problems with Finite-Domain Variables into
  Problems with Boolean Variables}. In \bibinfo{booktitle}{\emph{{SAT}}}.
  \bibinfo{pages}{1--15}.
\newblock


\bibitem[As{\'{\i}}n et~al\mbox{.}(2011)]%
        {asin2011}
\bibfield{author}{\bibinfo{person}{Roberto As{\'{\i}}n},
  \bibinfo{person}{Robert Nieuwenhuis}, \bibinfo{person}{Albert Oliveras},
  {and} \bibinfo{person}{Enric Rodr{\'{\i}}guez{-}Carbonell}.}
  \bibinfo{year}{2011}\natexlab{}.
\newblock \showarticletitle{Cardinality Networks: a theoretical and empirical
  study}.
\newblock \bibinfo{journal}{\emph{Constraints}} \bibinfo{volume}{16},
  \bibinfo{number}{2} (\bibinfo{year}{2011}), \bibinfo{pages}{195--221}.
\newblock


\bibitem[Audemard et~al\mbox{.}(2022a)]%
        {Audemard-ijcai22}
\bibfield{author}{\bibinfo{person}{Gilles Audemard}, \bibinfo{person}{Steve
  Bellart}, \bibinfo{person}{Louenas Bounia},
  \bibinfo{person}{Fr{\'{e}}d{\'{e}}ric Koriche}, \bibinfo{person}{Jean{-}Marie
  Lagniez}, {and} \bibinfo{person}{Pierre Marquis}.}
  \bibinfo{year}{2022}\natexlab{a}.
\newblock \showarticletitle{On Preferred Abductive Explanations for Decision
  Trees and Random Forests}. In \bibinfo{booktitle}{\emph{IJCAI}}.
  \bibinfo{pages}{643--650}.
\newblock


\bibitem[Audemard et~al\mbox{.}(2022b)]%
        {Audemard-aaai22}
\bibfield{author}{\bibinfo{person}{Gilles Audemard}, \bibinfo{person}{Steve
  Bellart}, \bibinfo{person}{Louenas Bounia},
  \bibinfo{person}{Fr{\'{e}}d{\'{e}}ric Koriche}, \bibinfo{person}{Jean{-}Marie
  Lagniez}, {and} \bibinfo{person}{Pierre Marquis}.}
  \bibinfo{year}{2022}\natexlab{b}.
\newblock \showarticletitle{Trading Complexity for Sparsity in Random Forest
  Explanations}. In \bibinfo{booktitle}{\emph{AAAI}}.
  \bibinfo{pages}{5461--5469}.
\newblock


\bibitem[Audemard et~al\mbox{.}(2023a)]%
        {Audemard-ijcai23}
\bibfield{author}{\bibinfo{person}{Gilles Audemard}, \bibinfo{person}{Steve
  Bellart}, \bibinfo{person}{Jean-Marie Lagniez}, {and} \bibinfo{person}{Pierre
  Marquis}.} \bibinfo{year}{2023}\natexlab{a}.
\newblock \showarticletitle{Computing abductive explanations for boosted
  regression trees}. In \bibinfo{booktitle}{\emph{IJCAI}}.
  \bibinfo{pages}{3432--3441}.
\newblock


\bibitem[Audemard et~al\mbox{.}(2020)]%
        {marquis-kr20}
\bibfield{author}{\bibinfo{person}{Gilles Audemard},
  \bibinfo{person}{Fr{\'{e}}d{\'{e}}ric Koriche}, {and} \bibinfo{person}{Pierre
  Marquis}.} \bibinfo{year}{2020}\natexlab{}.
\newblock \showarticletitle{On Tractable {XAI} Queries based on Compiled
  Representations}. In \bibinfo{booktitle}{\emph{{KR}}}.
  \bibinfo{pages}{838--849}.
\newblock


\bibitem[Audemard et~al\mbox{.}(2024a)]%
        {Audemard-ecai24}
\bibfield{author}{\bibinfo{person}{Gilles Audemard},
  \bibinfo{person}{Jean{-}Marie Lagniez}, {and} \bibinfo{person}{Pierre
  Marquis}.} \bibinfo{year}{2024}\natexlab{a}.
\newblock \showarticletitle{On the Computation of Contrastive Explanations for
  Boosted Regression Trees}. In \bibinfo{booktitle}{\emph{ECAI}},
  Vol.~\bibinfo{volume}{392}. \bibinfo{pages}{1083--1091}.
\newblock


\bibitem[Audemard et~al\mbox{.}(2023b)]%
        {Audemard-ecai23}
\bibfield{author}{\bibinfo{person}{Gilles Audemard},
  \bibinfo{person}{Jean{-}Marie Lagniez}, \bibinfo{person}{Pierre Marquis},
  {and} \bibinfo{person}{Nicolas Szczepanski}.}
  \bibinfo{year}{2023}\natexlab{b}.
\newblock \showarticletitle{On Contrastive Explanations for Tree-Based
  Classifiers}. In \bibinfo{booktitle}{\emph{ECAI}},
  Vol.~\bibinfo{volume}{372}. \bibinfo{pages}{117--124}.
\newblock


\bibitem[Audemard et~al\mbox{.}(2024b)]%
        {pyxai-ijcai24b}
\bibfield{author}{\bibinfo{person}{Gilles Audemard},
  \bibinfo{person}{Jean{-}Marie Lagniez}, \bibinfo{person}{Pierre Marquis},
  {and} \bibinfo{person}{Nicolas Szczepanski}.}
  \bibinfo{year}{2024}\natexlab{b}.
\newblock \showarticletitle{PyXAI: An {XAI} Library for Tree-Based Models}. In
  \bibinfo{booktitle}{\emph{IJCAI}}. \bibinfo{pages}{8601--8605}.
\newblock


\bibitem[Audemard et~al\mbox{.}(2013)]%
        {audemard-sat13}
\bibfield{author}{\bibinfo{person}{Gilles Audemard},
  \bibinfo{person}{Jean{-}Marie Lagniez}, {and} \bibinfo{person}{Laurent
  Simon}.} \bibinfo{year}{2013}\natexlab{}.
\newblock \showarticletitle{Improving {Glucose} for Incremental {SAT} Solving
  with Assumptions: Application to {MUS} Extraction}. In
  \bibinfo{booktitle}{\emph{SAT}}. \bibinfo{pages}{309--317}.
\newblock


\bibitem[Audemard et~al\mbox{.}(2023c)]%
        {Audemard-aistats23}
\bibfield{author}{\bibinfo{person}{Gilles Audemard},
  \bibinfo{person}{Jean-Marie Lagniez}, \bibinfo{person}{Pierre Marquis}, {and}
  \bibinfo{person}{Nicolas Szczepanski}.} \bibinfo{year}{2023}\natexlab{c}.
\newblock \showarticletitle{Computing abductive explanations for boosted
  trees}. In \bibinfo{booktitle}{\emph{AISTATS}}. \bibinfo{pages}{4699--4711}.
\newblock


\bibitem[Barcel{\'{o}} et~al\mbox{.}(2020)]%
        {barcelo-nips20}
\bibfield{author}{\bibinfo{person}{Pablo Barcel{\'{o}}},
  \bibinfo{person}{Mika{\"{e}}l Monet}, \bibinfo{person}{Jorge P{\'{e}}rez},
  {and} \bibinfo{person}{Bernardo Subercaseaux}.}
  \bibinfo{year}{2020}\natexlab{}.
\newblock \showarticletitle{Model Interpretability through the lens of
  Computational Complexity}. In \bibinfo{booktitle}{\emph{NeurIPS}}.
\newblock


\bibitem[Barrett et~al\mbox{.}(2009)]%
        {BarrettSST09}
\bibfield{author}{\bibinfo{person}{Clark~W. Barrett}, \bibinfo{person}{Roberto
  Sebastiani}, \bibinfo{person}{Sanjit~A. Seshia}, {and}
  \bibinfo{person}{Cesare Tinelli}.} \bibinfo{year}{2009}\natexlab{}.
\newblock \showarticletitle{Satisfiability Modulo Theories}.
\newblock In \bibinfo{booktitle}{\emph{Handbook of Satisfiability}}.
  \bibinfo{series}{Frontiers in Artificial Intelligence and Applications},
  Vol.~\bibinfo{volume}{185}. \bibinfo{pages}{825--885}.
\newblock


\bibitem[Biere et~al\mbox{.}(2021)]%
        {sat-handbook}
\bibfield{editor}{\bibinfo{person}{Armin Biere}, \bibinfo{person}{Marijn
  Heule}, \bibinfo{person}{Hans van Maaren}, {and} \bibinfo{person}{Toby
  Walsh}} (Eds.). \bibinfo{year}{2021}\natexlab{}.
\newblock \bibinfo{booktitle}{}. \bibinfo{series}{Frontiers in Artificial
  Intelligence and Applications}, Vol.~\bibinfo{volume}{336}.
  \bibinfo{publisher}{IOS Press}.
\newblock


\bibitem[Birnbaum and Lozinskii(2003)]%
        {lozinskii-jetai03}
\bibfield{author}{\bibinfo{person}{Elazar Birnbaum} {and}
  \bibinfo{person}{Eliezer~L. Lozinskii}.} \bibinfo{year}{2003}\natexlab{}.
\newblock \showarticletitle{Consistent subsets of inconsistent systems:
  structure and behaviour}.
\newblock \bibinfo{journal}{\emph{J. Exp. Theor. Artif. Intell.}}
  \bibinfo{volume}{15}, \bibinfo{number}{1} (\bibinfo{year}{2003}),
  \bibinfo{pages}{25--46}.
\newblock


\bibitem[Breiman(2001)]%
        {Breiman01}
\bibfield{author}{\bibinfo{person}{Leo Breiman}.}
  \bibinfo{year}{2001}\natexlab{}.
\newblock \showarticletitle{Random Forests}.
\newblock \bibinfo{journal}{\emph{Mach. Learn.}} \bibinfo{volume}{45},
  \bibinfo{number}{1} (\bibinfo{year}{2001}), \bibinfo{pages}{5--32}.
\newblock


\bibitem[Breiman et~al\mbox{.}(1984)]%
        {breiman-bk84}
\bibfield{author}{\bibinfo{person}{Leo Breiman}, \bibinfo{person}{J.~H.
  Friedman}, \bibinfo{person}{R.~A. Olshen}, {and} \bibinfo{person}{C.~J.
  Stone}.} \bibinfo{year}{1984}\natexlab{}.
\newblock \bibinfo{booktitle}{\emph{Classification and Regression Trees}}.
\newblock \bibinfo{publisher}{Wadsworth}.
\newblock
\showISBNx{0-534-98053-8}


\bibitem[Chen and Guestrin(2016)]%
        {guestrin-kdd16a}
\bibfield{author}{\bibinfo{person}{Tianqi Chen} {and} \bibinfo{person}{Carlos
  Guestrin}.} \bibinfo{year}{2016}\natexlab{}.
\newblock \showarticletitle{{XGBoost}: {A} Scalable Tree Boosting System}. In
  \bibinfo{booktitle}{\emph{{KDD}}}. \bibinfo{pages}{785--794}.
\newblock


\bibitem[Cplex(2009)]%
        {cplex2009v12}
\bibfield{author}{\bibinfo{person}{IBM~ILOG Cplex}.}
  \bibinfo{year}{2009}\natexlab{}.
\newblock \showarticletitle{V12. 1: User’s Manual for CPLEX}.
\newblock \bibinfo{journal}{\emph{International Business Machines Corporation}}
  \bibinfo{volume}{46}, \bibinfo{number}{53} (\bibinfo{year}{2009}),
  \bibinfo{pages}{157}.
\newblock


\bibitem[Darwiche and Hirth(2020)]%
        {darwiche-ecai20}
\bibfield{author}{\bibinfo{person}{Adnan Darwiche} {and}
  \bibinfo{person}{Auguste Hirth}.} \bibinfo{year}{2020}\natexlab{}.
\newblock \showarticletitle{On the Reasons Behind Decisions}. In
  \bibinfo{booktitle}{\emph{ECAI}}. \bibinfo{pages}{712--720}.
\newblock
\urldef\tempurl%
\url{https://doi.org/10.3233/FAIA200158}
\showDOI{\tempurl}


\bibitem[Darwiche and Marquis(2002)]%
        {darwiche-jair02}
\bibfield{author}{\bibinfo{person}{Adnan Darwiche} {and}
  \bibinfo{person}{Pierre Marquis}.} \bibinfo{year}{2002}\natexlab{}.
\newblock \showarticletitle{A Knowledge Compilation Map}.
\newblock \bibinfo{journal}{\emph{J. Artif. Intell. Res.}}
  \bibinfo{volume}{17} (\bibinfo{year}{2002}), \bibinfo{pages}{229--264}.
\newblock


\bibitem[Davies and Bacchus(2011)]%
        {bacchus-cp11}
\bibfield{author}{\bibinfo{person}{Jessica Davies} {and}
  \bibinfo{person}{Fahiem Bacchus}.} \bibinfo{year}{2011}\natexlab{}.
\newblock \showarticletitle{Solving {MAXSAT} by Solving a Sequence of Simpler
  {SAT} Instances}. In \bibinfo{booktitle}{\emph{{CP}}}.
  \bibinfo{pages}{225--239}.
\newblock


\bibitem[de~Moura and Bj{\o}rner(2008)]%
        {bjorner-tacas08}
\bibfield{author}{\bibinfo{person}{Leonardo~Mendon{\c{c}}a de Moura} {and}
  \bibinfo{person}{Nikolaj Bj{\o}rner}.} \bibinfo{year}{2008}\natexlab{}.
\newblock \showarticletitle{{Z3:} An Efficient {SMT} Solver}. In
  \bibinfo{booktitle}{\emph{{TACAS}}}. \bibinfo{pages}{337--340}.
\newblock


\bibitem[Dua and Graff(2017)]%
        {Dua2019}
\bibfield{author}{\bibinfo{person}{Dheeru Dua} {and} \bibinfo{person}{Casey
  Graff}.} \bibinfo{year}{2017}\natexlab{}.
\newblock \bibinfo{title}{{UCI} Machine Learning Repository}.
\newblock
\newblock
\urldef\tempurl%
\url{http://archive.ics.uci.edu/ml}
\showURL{%
\tempurl}


\bibitem[E{\'{e}}n and S{\"{o}}rensson(2003)]%
        {een-entcs03}
\bibfield{author}{\bibinfo{person}{Niklas E{\'{e}}n} {and}
  \bibinfo{person}{Niklas S{\"{o}}rensson}.} \bibinfo{year}{2003}\natexlab{}.
\newblock \showarticletitle{Temporal induction by incremental {SAT} solving}.
\newblock \bibinfo{journal}{\emph{Electron. Notes Theor. Comput. Sci.}}
  \bibinfo{volume}{89}, \bibinfo{number}{4} (\bibinfo{year}{2003}),
  \bibinfo{pages}{543--560}.
\newblock


\bibitem[Elffers and Nordstr{\"{o}}m(2018)]%
        {Nordstrom-ijcai18}
\bibfield{author}{\bibinfo{person}{Jan Elffers} {and} \bibinfo{person}{Jakob
  Nordstr{\"{o}}m}.} \bibinfo{year}{2018}\natexlab{}.
\newblock \showarticletitle{Divide and Conquer: Towards Faster Pseudo-Boolean
  Solving}. In \bibinfo{booktitle}{\emph{IJCAI}}.
  \bibinfo{publisher}{ijcai.org}, \bibinfo{pages}{1291--1299}.
\newblock


\bibitem[Friedman(2001)]%
        {friedman-tas01}
\bibfield{author}{\bibinfo{person}{Jerome~H. Friedman}.}
  \bibinfo{year}{2001}\natexlab{}.
\newblock \showarticletitle{Greedy Function Approximation: A Gradient Boosting
  Machine}.
\newblock \bibinfo{journal}{\emph{The Annals of Statistics}}
  \bibinfo{volume}{29}, \bibinfo{number}{5} (\bibinfo{year}{2001}),
  \bibinfo{pages}{1189--1232}.
\newblock


\bibitem[Gario and Micheli(2015)]%
        {gm-smt15}
\bibfield{author}{\bibinfo{person}{Marco Gario} {and} \bibinfo{person}{Andrea
  Micheli}.} \bibinfo{year}{2015}\natexlab{}.
\newblock \showarticletitle{{PySMT}: a Solver-Agnostic Library for Fast
  Prototyping of {SMT}-Based Algorithms}. In \bibinfo{booktitle}{\emph{SMT
  Workshop}}.
\newblock


\bibitem[Hemery et~al\mbox{.}(2006)]%
        {sais-ecai06}
\bibfield{author}{\bibinfo{person}{Fred Hemery}, \bibinfo{person}{Christophe
  Lecoutre}, \bibinfo{person}{Lakhdar Sais}, {and}
  \bibinfo{person}{Fr{\'{e}}d{\'{e}}ric Boussemart}.}
  \bibinfo{year}{2006}\natexlab{}.
\newblock \showarticletitle{Extracting {MUCs} from Constraint Networks}. In
  \bibinfo{booktitle}{\emph{ECAI}}, Vol.~\bibinfo{volume}{141}.
  \bibinfo{pages}{113--117}.
\newblock


\bibitem[Huang et~al\mbox{.}(2023a)]%
        {hcmpms-tacas23}
\bibfield{author}{\bibinfo{person}{Xuanxiang Huang}, \bibinfo{person}{Martin~C.
  Cooper}, \bibinfo{person}{Ant{\'{o}}nio Morgado}, \bibinfo{person}{Jordi
  Planes}, {and} \bibinfo{person}{Jo{\~{a}}o Marques{-}Silva}.}
  \bibinfo{year}{2023}\natexlab{a}.
\newblock \showarticletitle{Feature Necessity {\&} Relevancy in {ML} Classifier
  Explanations}. In \bibinfo{booktitle}{\emph{TACAS}},
  Vol.~\bibinfo{volume}{13993}. \bibinfo{pages}{167--186}.
\newblock


\bibitem[Huang et~al\mbox{.}(2022)]%
        {hiicams-aaai22}
\bibfield{author}{\bibinfo{person}{Xuanxiang Huang}, \bibinfo{person}{Yacine
  Izza}, \bibinfo{person}{Alexey Ignatiev}, \bibinfo{person}{Martin~C. Cooper},
  \bibinfo{person}{Nicholas Asher}, {and} \bibinfo{person}{Joao
  Marques-Silva}.} \bibinfo{year}{2022}\natexlab{}.
\newblock \showarticletitle{Tractable Explanations for {d-DNNF} Classifiers}.
  In \bibinfo{booktitle}{\emph{{AAAI}}}.
\newblock


\bibitem[Huang et~al\mbox{.}(2021)]%
        {hiims-kr21}
\bibfield{author}{\bibinfo{person}{Xuanxiang Huang}, \bibinfo{person}{Yacine
  Izza}, \bibinfo{person}{Alexey Ignatiev}, {and} \bibinfo{person}{Jo{\~{a}}o
  Marques{-}Silva}.} \bibinfo{year}{2021}\natexlab{}.
\newblock \showarticletitle{On Efficiently Explaining Graph-Based Classifiers}.
  In \bibinfo{booktitle}{\emph{KR}}. \bibinfo{pages}{356--367}.
\newblock


\bibitem[Huang et~al\mbox{.}(2025)]%
        {hiims-corr25}
\bibfield{author}{\bibinfo{person}{Xuanxiang Huang}, \bibinfo{person}{Yacine
  Izza}, \bibinfo{person}{Alexey Ignatiev}, {and} \bibinfo{person}{Joao
  Marques-Silva}.} \bibinfo{year}{2025}\natexlab{}.
\newblock \showarticletitle{Uncovering Bugs in Formal Explainers: A Case Study
  with PyXAI}.
\newblock \bibinfo{journal}{\emph{CoRR}}  \bibinfo{volume}{abs/2511.03169}
  (\bibinfo{year}{2025}).
\newblock
\urldef\tempurl%
\url{https://arxiv.org/abs/2511.03169}
\showURL{%
\tempurl}


\bibitem[Huang et~al\mbox{.}(2023b)]%
        {hims-aaai23}
\bibfield{author}{\bibinfo{person}{Xuanxiang Huang}, \bibinfo{person}{Yacine
  Izza}, {and} \bibinfo{person}{Jo{\~{a}}o Marques{-}Silva}.}
  \bibinfo{year}{2023}\natexlab{b}.
\newblock \showarticletitle{Solving Explainability Queries with Quantification:
  The Case of Feature Relevancy}. In \bibinfo{booktitle}{\emph{AAAI}}.
  \bibinfo{pages}{3996--4006}.
\newblock


\bibitem[Hyafil and Rivest(1976)]%
        {rivest-ipl76}
\bibfield{author}{\bibinfo{person}{Laurent Hyafil} {and}
  \bibinfo{person}{Ronald~L. Rivest}.} \bibinfo{year}{1976}\natexlab{}.
\newblock \showarticletitle{Constructing Optimal Binary Decision Trees is
  {NP}-Complete}.
\newblock \bibinfo{journal}{\emph{Inf. Process. Lett.}} \bibinfo{volume}{5},
  \bibinfo{number}{1} (\bibinfo{year}{1976}), \bibinfo{pages}{15--17}.
\newblock


\bibitem[Ignatiev(2020)]%
        {ignatiev-ijcai20}
\bibfield{author}{\bibinfo{person}{Alexey Ignatiev}.}
  \bibinfo{year}{2020}\natexlab{}.
\newblock \showarticletitle{Towards Trustable Explainable {AI}}. In
  \bibinfo{booktitle}{\emph{{IJCAI}}}. \bibinfo{pages}{5154--5158}.
\newblock


\bibitem[Ignatiev et~al\mbox{.}(2020a)]%
        {icshms-cp20}
\bibfield{author}{\bibinfo{person}{Alexey Ignatiev}, \bibinfo{person}{Martin~C.
  Cooper}, \bibinfo{person}{Mohamed Siala}, \bibinfo{person}{Emmanuel Hebrard},
  {and} \bibinfo{person}{Joao Marques{-}Silva}.}
  \bibinfo{year}{2020}\natexlab{a}.
\newblock \showarticletitle{Towards Formal Fairness in Machine Learning}. In
  \bibinfo{booktitle}{\emph{CP}}. \bibinfo{pages}{846--867}.
\newblock


\bibitem[Ignatiev et~al\mbox{.}(2022)]%
        {iism-aaai22}
\bibfield{author}{\bibinfo{person}{Alexey Ignatiev}, \bibinfo{person}{Yacine
  Izza}, \bibinfo{person}{Peter~J. Stuckey}, {and} \bibinfo{person}{Jo{\~{a}}o
  Marques{-}Silva}.} \bibinfo{year}{2022}\natexlab{}.
\newblock \showarticletitle{Using MaxSAT for Efficient Explanations of Tree
  Ensembles}. In \bibinfo{booktitle}{\emph{AAAI}}. \bibinfo{publisher}{{AAAI}
  Press}, \bibinfo{pages}{3776--3785}.
\newblock


\bibitem[Ignatiev et~al\mbox{.}(2016)]%
        {imms-ecai16}
\bibfield{author}{\bibinfo{person}{Alexey Ignatiev}, \bibinfo{person}{Antonio
  Morgado}, {and} \bibinfo{person}{Joao Marques{-}Silva}.}
  \bibinfo{year}{2016}\natexlab{}.
\newblock \showarticletitle{Propositional Abduction with Implicit Hitting
  Sets}. In \bibinfo{booktitle}{\emph{{ECAI}}}. \bibinfo{pages}{1327--1335}.
\newblock


\bibitem[Ignatiev et~al\mbox{.}(2018)]%
        {imms-sat18}
\bibfield{author}{\bibinfo{person}{Alexey Ignatiev}, \bibinfo{person}{Antonio
  Morgado}, {and} \bibinfo{person}{Joao Marques{-}Silva}.}
  \bibinfo{year}{2018}\natexlab{}.
\newblock \showarticletitle{{PySAT}: {A} {Python} Toolkit for Prototyping with
  {SAT} Oracles}. In \bibinfo{booktitle}{\emph{{SAT}}}.
  \bibinfo{pages}{428--437}.
\newblock


\bibitem[Ignatiev et~al\mbox{.}(2019a)]%
        {imms-jsat19}
\bibfield{author}{\bibinfo{person}{Alexey Ignatiev}, \bibinfo{person}{Antonio
  Morgado}, {and} \bibinfo{person}{Joao Marques{-}Silva}.}
  \bibinfo{year}{2019}\natexlab{a}.
\newblock \showarticletitle{{RC2:} an Efficient {MaxSAT} Solver}.
\newblock \bibinfo{journal}{\emph{J. Satisf. Boolean Model. Comput.}}
  \bibinfo{volume}{11}, \bibinfo{number}{1} (\bibinfo{year}{2019}),
  \bibinfo{pages}{53--64}.
\newblock


\bibitem[Ignatiev et~al\mbox{.}(2020b)]%
        {inams-aiia20}
\bibfield{author}{\bibinfo{person}{Alexey Ignatiev}, \bibinfo{person}{Nina
  Narodytska}, \bibinfo{person}{Nicholas Asher}, {and} \bibinfo{person}{Joao
  Marques{-}Silva}.} \bibinfo{year}{2020}\natexlab{b}.
\newblock \showarticletitle{From Contrastive to Abductive Explanations and Back
  Again}. In \bibinfo{booktitle}{\emph{AI*IA}}. \bibinfo{pages}{335--355}.
\newblock


\bibitem[Ignatiev et~al\mbox{.}(2019b)]%
        {inms-aaai19}
\bibfield{author}{\bibinfo{person}{Alexey Ignatiev}, \bibinfo{person}{Nina
  Narodytska}, {and} \bibinfo{person}{Joao Marques{-}Silva}.}
  \bibinfo{year}{2019}\natexlab{b}.
\newblock \showarticletitle{Abduction-Based Explanations for Machine Learning
  Models}. In \bibinfo{booktitle}{\emph{{AAAI}}}. \bibinfo{pages}{1511--1519}.
\newblock


\bibitem[Ignatiev et~al\mbox{.}(2019c)]%
        {inms-nips19}
\bibfield{author}{\bibinfo{person}{Alexey Ignatiev}, \bibinfo{person}{Nina
  Narodytska}, {and} \bibinfo{person}{Joao Marques{-}Silva}.}
  \bibinfo{year}{2019}\natexlab{c}.
\newblock \showarticletitle{On Relating Explanations and Adversarial Examples}.
  In \bibinfo{booktitle}{\emph{NeurIPS}}. \bibinfo{pages}{15857--15867}.
\newblock
\urldef\tempurl%
\url{http://papers.nips.cc/paper/9717-on-relating-explanations-and-adversarial-examples}
\showURL{%
\tempurl}


\bibitem[Ignatiev et~al\mbox{.}(2019d)]%
        {inms-corr19}
\bibfield{author}{\bibinfo{person}{Alexey Ignatiev}, \bibinfo{person}{Nina
  Narodytska}, {and} \bibinfo{person}{Joao Marques{-}Silva}.}
  \bibinfo{year}{2019}\natexlab{d}.
\newblock \showarticletitle{On Validating, Repairing and Refining Heuristic
  {ML} Explanations}.
\newblock \bibinfo{journal}{\emph{CoRR}}  \bibinfo{volume}{abs/1907.02509}
  (\bibinfo{year}{2019}).
\newblock


\bibitem[Ignatiev et~al\mbox{.}(2015)]%
        {iplms-cp15}
\bibfield{author}{\bibinfo{person}{Alexey Ignatiev},
  \bibinfo{person}{Alessandro Previti}, \bibinfo{person}{Mark~H. Liffiton},
  {and} \bibinfo{person}{Joao Marques{-}Silva}.}
  \bibinfo{year}{2015}\natexlab{}.
\newblock \showarticletitle{Smallest {MUS} Extraction with Minimal Hitting Set
  Dualization}. In \bibinfo{booktitle}{\emph{{CP}}}. \bibinfo{pages}{173--182}.
\newblock


\bibitem[Ignatiev and Silva(2021)]%
        {IgnatievS21}
\bibfield{author}{\bibinfo{person}{Alexey Ignatiev} {and}
  \bibinfo{person}{Jo{\~{a}}o P.~Marques Silva}.}
  \bibinfo{year}{2021}\natexlab{}.
\newblock \showarticletitle{SAT-Based Rigorous Explanations for Decision
  Lists}. In \bibinfo{booktitle}{\emph{SAT}},
  \bibfield{editor}{\bibinfo{person}{Chu{-}Min Li} {and} \bibinfo{person}{Felip
  Many{\`{a}}}} (Eds.), Vol.~\bibinfo{volume}{12831}.
  \bibinfo{publisher}{Springer}, \bibinfo{pages}{251--269}.
\newblock


\bibitem[Ignatiev et~al\mbox{.}(2024)]%
        {itk-sat24}
\bibfield{author}{\bibinfo{person}{Alexey Ignatiev}, \bibinfo{person}{Zi~Li
  Tan}, {and} \bibinfo{person}{Christos Karamanos}.}
  \bibinfo{year}{2024}\natexlab{}.
\newblock \showarticletitle{Towards Universally Accessible {SAT} Technology}.
  In \bibinfo{booktitle}{\emph{SAT}}. \bibinfo{pages}{4:1--4:11}.
\newblock


\bibitem[Izza et~al\mbox{.}(2020)]%
        {iims-corr20}
\bibfield{author}{\bibinfo{person}{Yacine Izza}, \bibinfo{person}{Alexey
  Ignatiev}, {and} \bibinfo{person}{Joao Marques{-}Silva}.}
  \bibinfo{year}{2020}\natexlab{}.
\newblock \showarticletitle{On Explaining Decision Trees}.
\newblock \bibinfo{journal}{\emph{CoRR}}  \bibinfo{volume}{abs/2010.11034}
  (\bibinfo{year}{2020}).
\newblock
\showeprint[arxiv]{2010.11034}


\bibitem[Izza et~al\mbox{.}(2025a)]%
        {iirmss-ijcai25}
\bibfield{author}{\bibinfo{person}{Yacine Izza}, \bibinfo{person}{Akexey
  Ignatiev}, \bibinfo{person}{Sasha Rubin}, \bibinfo{person}{Joao
  Marques-Silva}, {and} \bibinfo{person}{Peter~J. Stuckey}.}
  \bibinfo{year}{2025}\natexlab{a}.
\newblock \showarticletitle{Most General Explanations of Tree Ensembles}. In
  \bibinfo{booktitle}{\emph{IJCAI-25}}. \bibinfo{pages}{5463--5471}.
\newblock


\bibitem[Izza et~al\mbox{.}(2025b)]%
        {iirmss-corr25}
\bibfield{author}{\bibinfo{person}{Yacine Izza}, \bibinfo{person}{Alexey
  Ignatiev}, \bibinfo{person}{Sasha Rubin}, \bibinfo{person}{Joao
  Marques-Silva}, {and} \bibinfo{person}{Peter~J. Stuckey}.}
  \bibinfo{year}{2025}\natexlab{b}.
\newblock \showarticletitle{Most General Explanations of Tree Ensembles
  (Extended Version)}.
\newblock \bibinfo{journal}{\emph{CoRR}}  \bibinfo{volume}{abs/2505.10991}
  (\bibinfo{year}{2025}).
\newblock


\bibitem[Izza et~al\mbox{.}(2024)]%
        {iisms-aaai24}
\bibfield{author}{\bibinfo{person}{Yacine Izza}, \bibinfo{person}{Alexey
  Ignatiev}, \bibinfo{person}{Peter~J. Stuckey}, {and}
  \bibinfo{person}{Jo{\~{a}}o Marques{-}Silva}.}
  \bibinfo{year}{2024}\natexlab{}.
\newblock \showarticletitle{Delivering Inflated Explanations}. In
  \bibinfo{booktitle}{\emph{AAAI}}. \bibinfo{pages}{12744--12753}.
\newblock


\bibitem[Izza and Marques{-}Silva(2021)]%
        {ims-ijcai21}
\bibfield{author}{\bibinfo{person}{Yacine Izza} {and} \bibinfo{person}{Joao
  Marques{-}Silva}.} \bibinfo{year}{2021}\natexlab{}.
\newblock \showarticletitle{On Explaining Random Forests with {SAT}}. In
  \bibinfo{booktitle}{\emph{{IJCAI}}}. \bibinfo{pages}{2584--2591}.
\newblock


\bibitem[Junker(2004)]%
        {junker-aaai04}
\bibfield{author}{\bibinfo{person}{Ulrich Junker}.}
  \bibinfo{year}{2004}\natexlab{}.
\newblock \showarticletitle{{QUICKXPLAIN:} Preferred Explanations and
  Relaxations for Over-Constrained Problems}. In
  \bibinfo{booktitle}{\emph{{AAAI}}}. \bibinfo{pages}{167--172}.
\newblock


\bibitem[Karimi et~al\mbox{.}(2020)]%
        {Karimi-aistats20}
\bibfield{author}{\bibinfo{person}{Amir{-}Hossein Karimi},
  \bibinfo{person}{Gilles Barthe}, \bibinfo{person}{Borja Balle}, {and}
  \bibinfo{person}{Isabel Valera}.} \bibinfo{year}{2020}\natexlab{}.
\newblock \showarticletitle{Model-Agnostic Counterfactual Explanations for
  Consequential Decisions}. In \bibinfo{booktitle}{\emph{AISTATS}},
  Vol.~\bibinfo{volume}{108}. \bibinfo{publisher}{{PMLR}},
  \bibinfo{pages}{895--905}.
\newblock


\bibitem[Lakkaraju and Bastani(2020)]%
        {lakkaraju-aies20b}
\bibfield{author}{\bibinfo{person}{Himabindu Lakkaraju} {and}
  \bibinfo{person}{Osbert Bastani}.} \bibinfo{year}{2020}\natexlab{}.
\newblock \showarticletitle{"How do {I} fool you?": Manipulating User Trust via
  Misleading Black Box Explanations}. In \bibinfo{booktitle}{\emph{AIES}}.
  \bibinfo{pages}{79--85}.
\newblock


\bibitem[Liffiton et~al\mbox{.}(2016)]%
        {lpmms-cj16}
\bibfield{author}{\bibinfo{person}{Mark~H. Liffiton},
  \bibinfo{person}{Alessandro Previti}, \bibinfo{person}{Ammar Malik}, {and}
  \bibinfo{person}{Joao Marques{-}Silva}.} \bibinfo{year}{2016}\natexlab{}.
\newblock \showarticletitle{Fast, flexible {MUS} enumeration}.
\newblock \bibinfo{journal}{\emph{Constraints An Int. J.}}
  \bibinfo{volume}{21}, \bibinfo{number}{2} (\bibinfo{year}{2016}),
  \bibinfo{pages}{223--250}.
\newblock


\bibitem[Lipton(2018)]%
        {lipton-cacm18}
\bibfield{author}{\bibinfo{person}{Zachary~C. Lipton}.}
  \bibinfo{year}{2018}\natexlab{}.
\newblock \showarticletitle{The mythos of model interpretability}.
\newblock \bibinfo{journal}{\emph{Commun. {ACM}}} \bibinfo{volume}{61},
  \bibinfo{number}{10} (\bibinfo{year}{2018}), \bibinfo{pages}{36--43}.
\newblock


\bibitem[Lundberg and Lee(2017)]%
        {lundberg-nips17}
\bibfield{author}{\bibinfo{person}{Scott~M. Lundberg} {and}
  \bibinfo{person}{Su{-}In Lee}.} \bibinfo{year}{2017}\natexlab{}.
\newblock \showarticletitle{A Unified Approach to Interpreting Model
  Predictions}. In \bibinfo{booktitle}{\emph{{NIPS}}}.
  \bibinfo{pages}{4765--4774}.
\newblock


\bibitem[Marques{-}Silva(2024)]%
        {ms-isola24}
\bibfield{author}{\bibinfo{person}{Jo{\~{a}}o Marques{-}Silva}.}
  \bibinfo{year}{2024}\natexlab{}.
\newblock \showarticletitle{Logic-Based Explainability: Past, Present and
  Future}. In \bibinfo{booktitle}{\emph{ISoLA}}, Vol.~\bibinfo{volume}{15222}.
  \bibinfo{pages}{181--204}.
\newblock


\bibitem[Marques{-}Silva et~al\mbox{.}(2011)]%
        {msagl-amai11}
\bibfield{author}{\bibinfo{person}{Joao Marques{-}Silva},
  \bibinfo{person}{Josep Argelich}, \bibinfo{person}{Ana Gra{\c{c}}a}, {and}
  \bibinfo{person}{In{\^{e}}s Lynce}.} \bibinfo{year}{2011}\natexlab{}.
\newblock \showarticletitle{Boolean lexicographic optimization: algorithms {\&}
  applications}.
\newblock \bibinfo{journal}{\emph{Ann. Math. Artif. Intell.}}
  \bibinfo{volume}{62}, \bibinfo{number}{3-4} (\bibinfo{year}{2011}),
  \bibinfo{pages}{317--343}.
\newblock


\bibitem[Marques{-}Silva et~al\mbox{.}(2020)]%
        {msgcin-nips20}
\bibfield{author}{\bibinfo{person}{Joao Marques{-}Silva},
  \bibinfo{person}{Thomas Gerspacher}, \bibinfo{person}{Martin~C. Cooper},
  \bibinfo{person}{Alexey Ignatiev}, {and} \bibinfo{person}{Nina Narodytska}.}
  \bibinfo{year}{2020}\natexlab{}.
\newblock \showarticletitle{Explaining Naive Bayes and Other Linear Classifiers
  with Polynomial Time and Delay}. In \bibinfo{booktitle}{\emph{NeurIPS}}.
\newblock


\bibitem[Marques{-}Silva et~al\mbox{.}(2021)]%
        {msgcin-icml21}
\bibfield{author}{\bibinfo{person}{Joao Marques{-}Silva},
  \bibinfo{person}{Thomas Gerspacher}, \bibinfo{person}{Martin~C. Cooper},
  \bibinfo{person}{Alexey Ignatiev}, {and} \bibinfo{person}{Nina Narodytska}.}
  \bibinfo{year}{2021}\natexlab{}.
\newblock \showarticletitle{Explanations for Monotonic Classifiers}. In
  \bibinfo{booktitle}{\emph{ICML}}. \bibinfo{pages}{7469--7479}.
\newblock


\bibitem[Marques{-}Silva et~al\mbox{.}(2013a)]%
        {mshjpb-ijcai13}
\bibfield{author}{\bibinfo{person}{Joao Marques{-}Silva},
  \bibinfo{person}{Federico Heras}, \bibinfo{person}{Mikol{\'{a}}s Janota},
  \bibinfo{person}{Alessandro Previti}, {and} \bibinfo{person}{Anton Belov}.}
  \bibinfo{year}{2013}\natexlab{a}.
\newblock \showarticletitle{On Computing Minimal Correction Subsets}. In
  \bibinfo{booktitle}{\emph{IJCAI}}. \bibinfo{pages}{615--622}.
\newblock


\bibitem[Marques-Silva and Huang(2024)]%
        {msh-cacm24}
\bibfield{author}{\bibinfo{person}{Joao Marques-Silva} {and}
  \bibinfo{person}{Xuanxiang Huang}.} \bibinfo{year}{2024}\natexlab{}.
\newblock \showarticletitle{Explainability Is Not a Game}.
\newblock \bibinfo{journal}{\emph{Commun. ACM}} \bibinfo{volume}{67},
  \bibinfo{number}{7} (\bibinfo{date}{jul} \bibinfo{year}{2024}),
  \bibinfo{pages}{66–75}.
\newblock
\showISSN{0001-0782}
\urldef\tempurl%
\url{https://doi.org/10.1145/3635301}
\showDOI{\tempurl}


\bibitem[Marques-Silva and Ignatiev(2022)]%
        {msi-aaai22}
\bibfield{author}{\bibinfo{person}{Joao Marques-Silva} {and}
  \bibinfo{person}{Alexey Ignatiev}.} \bibinfo{year}{2022}\natexlab{}.
\newblock \showarticletitle{Delivering Trustworthy {AI} through formal {XAI}}.
  In \bibinfo{booktitle}{\emph{AAAI}}.
\newblock


\bibitem[Marques{-}Silva et~al\mbox{.}(2013b)]%
        {msjb-cav13}
\bibfield{author}{\bibinfo{person}{Joao Marques{-}Silva},
  \bibinfo{person}{Mikol{\'{a}}s Janota}, {and} \bibinfo{person}{Anton Belov}.}
  \bibinfo{year}{2013}\natexlab{b}.
\newblock \showarticletitle{Minimal Sets over Monotone Predicates in Boolean
  Formulae}. In \bibinfo{booktitle}{\emph{CAV}}. \bibinfo{pages}{592--607}.
\newblock


\bibitem[Marques{-}Silva et~al\mbox{.}(2017)]%
        {msjm-aij17}
\bibfield{author}{\bibinfo{person}{Jo{\~{a}}o Marques{-}Silva},
  \bibinfo{person}{Mikol{\'{a}}s Janota}, {and} \bibinfo{person}{Carlos
  Menc{\'{\i}}a}.} \bibinfo{year}{2017}\natexlab{}.
\newblock \showarticletitle{Minimal sets on propositional formulae. Problems
  and reductions}.
\newblock \bibinfo{journal}{\emph{Artif. Intell.}}  \bibinfo{volume}{252}
  (\bibinfo{year}{2017}), \bibinfo{pages}{22--50}.
\newblock


\bibitem[Menc{\'{\i}}a et~al\mbox{.}(2015)]%
        {mpms-ijcai15}
\bibfield{author}{\bibinfo{person}{Carlos Menc{\'{\i}}a},
  \bibinfo{person}{Alessandro Previti}, {and} \bibinfo{person}{Joao
  Marques{-}Silva}.} \bibinfo{year}{2015}\natexlab{}.
\newblock \showarticletitle{Literal-Based {MCS} Extraction}. In
  \bibinfo{booktitle}{\emph{IJCAI}}. \bibinfo{pages}{1973--1979}.
\newblock
\urldef\tempurl%
\url{http://ijcai.org/Abstract/15/280}
\showURL{%
\tempurl}


\bibitem[Miller(2019)]%
        {miller-aij19}
\bibfield{author}{\bibinfo{person}{Tim Miller}.}
  \bibinfo{year}{2019}\natexlab{}.
\newblock \showarticletitle{Explanation in artificial intelligence: Insights
  from the social sciences}.
\newblock \bibinfo{journal}{\emph{Artif. Intell.}}  \bibinfo{volume}{267}
  (\bibinfo{year}{2019}), \bibinfo{pages}{1--38}.
\newblock


\bibitem[minisat(2005)]%
        {minisat-gh}
minisat \bibinfo{year}{2005}\natexlab{}.
\newblock \bibinfo{title}{Minisat: Github version}.
\newblock \bibinfo{howpublished}{\url{https://github.com/niklasso/minisat}}.
\newblock


\bibitem[Morgado et~al\mbox{.}(2014)]%
        {mdms-cp14}
\bibfield{author}{\bibinfo{person}{Antonio Morgado}, \bibinfo{person}{Carmine
  Dodaro}, {and} \bibinfo{person}{Joao Marques{-}Silva}.}
  \bibinfo{year}{2014}\natexlab{}.
\newblock \showarticletitle{Core-Guided {MaxSAT} with Soft Cardinality
  Constraints}. In \bibinfo{booktitle}{\emph{{CP}}}. \bibinfo{pages}{564--573}.
\newblock


\bibitem[Morgado et~al\mbox{.}(2013)]%
        {morgado-cj13}
\bibfield{author}{\bibinfo{person}{Antonio Morgado}, \bibinfo{person}{Federico
  Heras}, \bibinfo{person}{Mark~H. Liffiton}, \bibinfo{person}{Jordi Planes},
  {and} \bibinfo{person}{Joao Marques{-}Silva}.}
  \bibinfo{year}{2013}\natexlab{}.
\newblock \showarticletitle{Iterative and core-guided {MaxSAT} solving: {A}
  survey and assessment}.
\newblock \bibinfo{journal}{\emph{Constraints An Int. J.}}
  \bibinfo{volume}{18}, \bibinfo{number}{4} (\bibinfo{year}{2013}),
  \bibinfo{pages}{478--534}.
\newblock


\bibitem[Moskewicz et~al\mbox{.}(2001)]%
        {Malik-dac01}
\bibfield{author}{\bibinfo{person}{Matthew~W. Moskewicz},
  \bibinfo{person}{Conor~F. Madigan}, \bibinfo{person}{Ying Zhao},
  \bibinfo{person}{Lintao Zhang}, {and} \bibinfo{person}{Sharad Malik}.}
  \bibinfo{year}{2001}\natexlab{}.
\newblock \showarticletitle{Chaff: Engineering an Efficient {SAT} Solver}.
  \bibinfo{publisher}{{ACM}}, \bibinfo{pages}{530--535}.
\newblock


\bibitem[Narodytska et~al\mbox{.}(2019)]%
        {nsmims-sat19}
\bibfield{author}{\bibinfo{person}{Nina Narodytska}, \bibinfo{person}{Aditya~A.
  Shrotri}, \bibinfo{person}{Kuldeep~S. Meel}, \bibinfo{person}{Alexey
  Ignatiev}, {and} \bibinfo{person}{Joao Marques{-}Silva}.}
  \bibinfo{year}{2019}\natexlab{}.
\newblock \showarticletitle{Assessing Heuristic Machine Learning Explanations
  with Model Counting}. In \bibinfo{booktitle}{\emph{SAT}}.
  \bibinfo{pages}{267--278}.
\newblock


\bibitem[Olson et~al\mbox{.}(2017)]%
        {Olson2017PMLB}
\bibfield{author}{\bibinfo{person}{Randal~S. Olson}, \bibinfo{person}{William
  La~Cava}, \bibinfo{person}{Patryk Orzechowski}, \bibinfo{person}{Ryan~J.
  Urbanowicz}, {and} \bibinfo{person}{Jason~H. Moore}.}
  \bibinfo{year}{2017}\natexlab{}.
\newblock \showarticletitle{PMLB: a large benchmark suite for machine learning
  evaluation and comparison}.
\newblock \bibinfo{journal}{\emph{BioData Mining}} \bibinfo{volume}{10},
  \bibinfo{number}{1} (\bibinfo{year}{2017}), \bibinfo{pages}{36}.
\newblock


\bibitem[Papadimitriou(1994)]%
        {Papadimitriou94}
\bibfield{author}{\bibinfo{person}{Christos~H. Papadimitriou}.}
  \bibinfo{year}{1994}\natexlab{}.
\newblock \bibinfo{booktitle}{\emph{Computational complexity}}.
\newblock \bibinfo{publisher}{Addison-Wesley}.
\newblock


\bibitem[Pedregosa and et~al.(2011)]%
        {scikitlearn}
\bibfield{author}{\bibinfo{person}{Fabian Pedregosa} {and} \bibinfo{person}{et
  al.}} \bibinfo{year}{2011}\natexlab{}.
\newblock \showarticletitle{Scikit-learn: Machine Learning in {P}ython}.
\newblock \bibinfo{journal}{\emph{Journal of Machine Learning Research}}
  \bibinfo{volume}{12} (\bibinfo{year}{2011}), \bibinfo{pages}{2825--2830}.
\newblock


\bibitem[Perron and Didier(2025)]%
        {or-tools}
\bibfield{author}{\bibinfo{person}{Laurent Perron} {and}
  \bibinfo{person}{Frédéric Didier}.} \bibinfo{year}{2025}\natexlab{}.
\newblock \bibinfo{title}{Google Optimization Tools (a.k.a., OR-Tools)}.
\newblock
  \bibinfo{howpublished}{\url{https://developers.google.com/optimization/cp/cp_solver/}}.
\newblock
\newblock
\shownote{v9.12}.


\bibitem[Previti et~al\mbox{.}(2017)]%
        {previti-sat17}
\bibfield{author}{\bibinfo{person}{Alessandro Previti}, \bibinfo{person}{Carlos
  Menc{\'{\i}}a}, \bibinfo{person}{Matti J{\"{a}}rvisalo}, {and}
  \bibinfo{person}{Joao Marques{-}Silva}.} \bibinfo{year}{2017}\natexlab{}.
\newblock \showarticletitle{Improving {MCS} Enumeration via Caching}. In
  \bibinfo{booktitle}{\emph{{SAT}}}. \bibinfo{pages}{184--194}.
\newblock


\bibitem[Quinlan(1986)]%
        {quinlan-ml86}
\bibfield{author}{\bibinfo{person}{J.~Ross Quinlan}.}
  \bibinfo{year}{1986}\natexlab{}.
\newblock \showarticletitle{Induction of Decision Trees}.
\newblock \bibinfo{journal}{\emph{Machine Learning}} \bibinfo{volume}{1},
  \bibinfo{number}{1} (\bibinfo{year}{1986}), \bibinfo{pages}{81--106}.
\newblock


\bibitem[Quinlan(1993)]%
        {quinlan-book93}
\bibfield{author}{\bibinfo{person}{J.~Ross Quinlan}.}
  \bibinfo{year}{1993}\natexlab{}.
\newblock \bibinfo{booktitle}{\emph{{C4.5:} Programs for Machine Learning}}.
\newblock \bibinfo{publisher}{Morgan Kaufmann}.
\newblock


\bibitem[Reiter(1987)]%
        {reiter-aij87}
\bibfield{author}{\bibinfo{person}{Raymond Reiter}.}
  \bibinfo{year}{1987}\natexlab{}.
\newblock \showarticletitle{A Theory of Diagnosis from First Principles}.
\newblock \bibinfo{journal}{\emph{Artif. Intell.}} \bibinfo{volume}{32},
  \bibinfo{number}{1} (\bibinfo{year}{1987}), \bibinfo{pages}{57--95}.
\newblock


\bibitem[Ribeiro et~al\mbox{.}(2016)]%
        {guestrin-kdd16b}
\bibfield{author}{\bibinfo{person}{Marco~T{\'{u}}lio Ribeiro},
  \bibinfo{person}{Sameer Singh}, {and} \bibinfo{person}{Carlos Guestrin}.}
  \bibinfo{year}{2016}\natexlab{}.
\newblock \showarticletitle{``{Why} Should {I} Trust You?'': Explaining the
  Predictions of Any Classifier}. In \bibinfo{booktitle}{\emph{KDD}}.
  \bibinfo{pages}{1135--1144}.
\newblock


\bibitem[Ribeiro et~al\mbox{.}(2018)]%
        {guestrin-aaai18}
\bibfield{author}{\bibinfo{person}{Marco~T{\'{u}}lio Ribeiro},
  \bibinfo{person}{Sameer Singh}, {and} \bibinfo{person}{Carlos Guestrin}.}
  \bibinfo{year}{2018}\natexlab{}.
\newblock \showarticletitle{Anchors: High-Precision Model-Agnostic
  Explanations}. In \bibinfo{booktitle}{\emph{{AAAI}}}.
  \bibinfo{pages}{1527--1535}.
\newblock


\bibitem[Selvaraju et~al\mbox{.}(2020)]%
        {Grad-CAM-ijcv20}
\bibfield{author}{\bibinfo{person}{Ramprasaath~R. Selvaraju},
  \bibinfo{person}{Michael Cogswell}, \bibinfo{person}{Abhishek Das},
  \bibinfo{person}{Ramakrishna Vedantam}, \bibinfo{person}{Devi Parikh}, {and}
  \bibinfo{person}{Dhruv Batra}.} \bibinfo{year}{2020}\natexlab{}.
\newblock \showarticletitle{Grad-CAM: Visual Explanations from Deep Networks
  via Gradient-Based Localization}.
\newblock \bibinfo{journal}{\emph{Int. J. Comput. Vis.}} \bibinfo{volume}{128},
  \bibinfo{number}{2} (\bibinfo{year}{2020}), \bibinfo{pages}{336--359}.
\newblock


\bibitem[Shih et~al\mbox{.}(2018)]%
        {darwiche-ijcai18}
\bibfield{author}{\bibinfo{person}{Andy Shih}, \bibinfo{person}{Arthur Choi},
  {and} \bibinfo{person}{Adnan Darwiche}.} \bibinfo{year}{2018}\natexlab{}.
\newblock \showarticletitle{A Symbolic Approach to Explaining Bayesian Network
  Classifiers}. In \bibinfo{booktitle}{\emph{IJCAI}}.
  \bibinfo{pages}{5103--5111}.
\newblock


\bibitem[Simonyan et~al\mbox{.}(2014)]%
        {Zisserman-iclr24}
\bibfield{author}{\bibinfo{person}{Karen Simonyan}, \bibinfo{person}{Andrea
  Vedaldi}, {and} \bibinfo{person}{Andrew Zisserman}.}
  \bibinfo{year}{2014}\natexlab{}.
\newblock \showarticletitle{Deep Inside Convolutional Networks: Visualising
  Image Classification Models and Saliency Maps}. In
  \bibinfo{booktitle}{\emph{ICLR}}.
\newblock


\bibitem[Slack et~al\mbox{.}(2020)]%
        {lakkaraju-aies20a}
\bibfield{author}{\bibinfo{person}{Dylan Slack}, \bibinfo{person}{Sophie
  Hilgard}, \bibinfo{person}{Emily Jia}, \bibinfo{person}{Sameer Singh}, {and}
  \bibinfo{person}{Himabindu Lakkaraju}.} \bibinfo{year}{2020}\natexlab{}.
\newblock \showarticletitle{Fooling {LIME} and {SHAP:} Adversarial Attacks on
  Post hoc Explanation Methods}. In \bibinfo{booktitle}{\emph{AIES}}.
  \bibinfo{pages}{180--186}.
\newblock


\end{thebibliography}
}{

}


\end{document}